\documentclass{article}
\usepackage[final,nonatbib]{neurips_2021} 
\usepackage{wrapfig}
\usepackage[T1]{fontenc}
\usepackage[utf8]{inputenc}  
\usepackage[american]{babel}
\usepackage{amsfonts}
\usepackage{pgothic}
\usepackage{yfonts}
\usepackage[colorlinks=true]{hyperref}
\usepackage{bigcenter}
\usepackage{amsthm}
\usepackage{amsmath}
\usepackage{thmtools, thm-restate}
\usepackage{amssymb,tabularx}
\usepackage{dsfont}
\usepackage{bm}
\usepackage{mathtools}
\usepackage{float}
\usepackage{stmaryrd}
\usepackage{times}
\usepackage{subfig}
\usepackage{tikz}
\usetikzlibrary{arrows,automata}
\tikzset{terminal state/.style={draw,rectangle,minimum size=.3in}}
\usepackage{enumitem}
\usepackage[boxed]{algorithm}
\usepackage{algpseudocode}
\usepackage[colorinlistoftodos,prependcaption,textsize=tiny]{todonotes}

\usepackage{xcolor}
\usepackage{xargs}       
\usepackage{ifthen}
\usepackage[multidot]{grffile}
\usepackage{physics}

\newcommandx{\rlm}[2][1=]{\todo[linecolor=violet,backgroundcolor=violet!25,bordercolor=violet,#1]{\textbf{Romain:} #2}}
\newcommandx{\ale}[2][1=]{\todo[linecolor=red,backgroundcolor=red!25,bordercolor=red,#1]{\textbf{ale:} #2}}
\newcommandx{\rtcm}[2][1=]{\todo[linecolor=red,backgroundcolor=red!25,bordercolor=red,#1]{\textbf{Remi:} #2}}
\newcommandx{\nfm}[2][1=]{\todo[linecolor=blue,backgroundcolor=blue!25,bordercolor=blue,#1]{\textbf{Nicole:} #2}}
\newcommandx{\wbm}[2][1=]{\todo[linecolor=green,backgroundcolor=green!25,bordercolor=green,#1]{\textbf{William:} #2}}

\DeclareRobustCommand{\ass}[1]{%
  A#1%
}
\DeclareRobustCommand{\con}[1]{%
  C#1%
}
\newcommand*{\jh}{J\&H}
\newcommand*{\app}{App.}
\newcommand*{\fig}{Fig.}
\newcommand*{\alg}{Alg.}
\newcommand*{\sect}{Section}
\newcommand*{\thm}{Thm.}
\newcommand*{\rem}{Rem.}
\newcommand*{\eq}{Eq.}
\newcommand*{\mytop}{\mathrel{\scalebox{0.5}{$\top$}}}
\newcommand*{\mybot}{\mathrel{\scalebox{0.5}{$\bot$}}}

\newcommand*{\myplus}{\mathrel{\scalebox{0.5}{$+$}}}
\newcommand*{\myminus}{\mathrel{\scalebox{0.5}{$-$}}}

\newcommand*{\myeq}{\mathrel{\scalebox{0.5}{$=$}}}
\newcommand*{\mysmaller}{\mathrel{\scalebox{0.5}{$<$}}}
\newcommand*{\mybigger}{\mathrel{\scalebox{0.5}{$>$}}}
\newcommand*{\myinfty}{\mathrel{\scalebox{0.5}{$-\infty$}}}
\newcommand*{\myeqsm}{\mathrel{\scalebox{0.5}{$\stackrel{\scalebox{0.7}{$<$}}{=}$}}}
\newcommand*{\myeqbg}{\mathrel{\scalebox{0.5}{$\stackrel{\scalebox{0.7}{$>$}}{=}$}}}

\DeclareMathOperator*{\R}{\mathbb{R}}

\DeclareMathOperator*{\argmax}{argmax}

\tikzstyle{startstop} = [rectangle, rounded corners, minimum width=3cm, minimum height=1cm,text centered, draw=black, fill=red!30]
\tikzstyle{io} = [trapezium, trapezium left angle=70, trapezium right angle=110, minimum width=3cm, minimum height=1cm, text centered, draw=black, fill=blue!30]
\tikzstyle{process} = [rectangle, minimum width=3cm, minimum height=1cm, text centered, draw=black, fill=orange!30]
\tikzstyle{decision} = [diamond, minimum width=3cm, minimum height=1cm, text centered, draw=black, fill=green!30]
\tikzstyle{arrow} = [thick,->,>=stealth]

\newtheorem{theorem}{Theorem}
\newenvironment{thma}[1]
  {\renewcommand{\thetheorem}{\ref*{#1}a}%
   \addtocounter{theorem}{-1}%
   \begin{theorem}}
  {\end{theorem}}
\newenvironment{thmb}[1]
  {\renewcommand{\thetheorem}{\ref*{#1}b}%
   \addtocounter{theorem}{-1}%
   \begin{theorem}}
  {\end{theorem}}
\newenvironment{thmc}[1]
  {\renewcommand{\thetheorem}{\ref*{#1}c}%
   \addtocounter{theorem}{-1}%
   \begin{theorem}}
  {\end{theorem}}

\newtheorem{remark}{Remark}
\newtheorem{corollary}{Corollary}
\newtheorem{lemma}{Lemma}

\newtheorem{assumption}{Assumption}

\floatname{algorithm}{Algorithm}

\usepackage{contour}
\usepackage{ulem}

\contourlength{1pt}

\contourlength{0.8pt}

\newcommand{\myuline}[1]{%
  \uline{\phantom{#1}}%
  \llap{\contour{white}{#1}}%
}

\algnewcommand{\InlineIfThen}[2]{
  \State \algorithmicif\ #1\ \algorithmicthen\ #2\ }
  
\newcommandx{\yaHelper}[2][1=\empty]{%
\ifthenelse{\equal{#1}{\empty}}%
  { \ensuremath{ \scriptstyle{ #2 } } } 
  { \raisebox{ #1 }[0pt][0pt]{ \ensuremath{ \scriptstyle{ #2 } } } }  
}

\newcommandx{\yrightarrow}[4][2=\empty, 3=\empty, 4=\empty, usedefault=@]{%
  \ifthenelse{\equal{#4}{\empty}}
  { \xrightarrow[ \protect{ \yaHelper[ #2 ]{ #1 } } ] {}} 
  { \xrightarrow[ \protect{ \yaHelper[ #2 ]{ #1 } } ]{ \protect{ \yaHelper[ #4 ]{ #3 } } } } 
}

\newcommandx{\yleftarrow}[4][2=\empty, 3=\empty, 4=\empty, usedefault=@]{%
  \ifthenelse{\equal{#4}{\empty}}
  { \xleftarrow[ \protect{ \yaHelper[ #2 ]{ #1 } } ] {} } 
  { \xleftarrow[ \protect{ \yaHelper[ #2 ]{ #1 } } ]{ \protect{ \yaHelper[ #4 ]{ #3 } } } } 
}

\newcommandx{\yRightarrow}[4][2=\empty, 3=\empty, 4=\empty, usedefault=@]{%
  \ifthenelse{\equal{#4}{\empty}}
  { \xRightarrow[ \protect{ \yaHelper[ #2 ]{ #1 } } ] {} } 
  { \xRightarrow[ \protect{ \yaHelper[ #2 ]{ #1 } } ]{ \protect{ \yaHelper[ #4 ]{ #3 } } } } 
}

\newcommandx{\yLeftarrow}[4][2=\empty, 3=\empty, 4=\empty, usedefault=@]{%
  \ifthenelse{\equal{#4}{\empty}}
  { \xLeftarrow[ \protect{ \yaHelper[ #2 ]{ #1 } } ] {} } 
  { \xLeftarrow[ \protect{ \yaHelper[ #2 ]{ #1 } } ]{ \protect{ \yaHelper[ #4 ]{ #3 } } } } 
}  

\makeatletter
\newcommand{\multiline}[1]{%
  \begin{tabularx}{\dimexpr\linewidth-\ALG@thistlm}[t]{@{}X@{}}
    #1
  \end{tabularx}
}
\makeatother

\title{Dr Jekyll and Mr Hyde: \\ The Strange Case of Off-Policy Policy Updates}

\author{%
  Romain Laroche\thanks{Equal contribution. Correspondence to: \href{mailto:rolaroch@microsoft.com}{rolaroch@microsoft.com}. All code available at \href{http://aka.ms/jnh}{http://aka.ms/jnh}.} \\
  Microsoft Research Montr\'eal, Canada\\
  \And
  R\'emi Tachet des Combes$^*$ \\
  Microsoft Research Montr\'eal, Canada\\
}

\begin{document} 
\maketitle
\begin{abstract}
    %
   The policy gradient theorem states that the policy should only be updated in states that are visited by the current policy, which leads to insufficient planning in the \textit{off-policy} states, and thus to convergence to suboptimal policies. We tackle this planning issue by extending the policy gradient theory to \textit{policy updates} with respect to any state density. Under these generalized policy updates, we show convergence to optimality under a necessary and sufficient condition on the updates' state densities, and thereby solve the aforementioned planning issue. We also prove asymptotic convergence rates that significantly improve those in the policy gradient literature.
   To implement the principles prescribed by our theory, we propose an agent, Dr Jekyll \& Mr Hyde (\jh{}), with a double personality: Dr Jekyll purely exploits while Mr Hyde purely explores. \jh{}'s independent policies allow to record two separate replay buffers: one on-policy (Dr Jekyll's) and one off-policy (Mr Hyde's), and therefore to update \jh{}'s models with a mixture of on-policy and off-policy updates. More than an algorithm, \jh{} defines principles for actor-critic algorithms to satisfy the requirements we identify in our analysis. We extensively test on finite MDPs where \jh{} demonstrates a superior ability to recover from converging to a suboptimal policy without impairing its speed of convergence. We also implement a deep version of the algorithm and test it on a simple problem where it shows promising results.
\end{abstract}

\section{Introduction}

Policy Gradient algorithms in Reinforcement Learning (RL) have enjoyed great success both theoretically \cite{williams1992simple,sutton2000policy,konda2002thesis,schulman2015trust} and empirically \cite{mnih2016asynchronous,silver2016mastering,schulman2017proximal,opeai2018}. Their principle consists in optimizing an objective function $\mathcal{J}$ (the expected discounted return) through gradient steps, both being formally specified below~\cite{sutton2000policy}:
    \begin{align}
        \mathcal{J}(\pi) &\doteq \mathbb{E}\left[\sum_{t=0}^\infty \gamma^t R_t\bigg| S_0\sim p_0, A_t\sim \pi(\cdot|S_t), S_{t+1}\sim p(\cdot|S_t,A_t), R_t\sim r(\cdot|S_t,A_t)\right] \\
        \nabla_\theta \mathcal{J}(\pi) &= \sum_{s\in\mathcal{S}} d_{\pi,\gamma}(s) \sum_{a\in\mathcal{A}} q_\pi(s, a) \nabla_\theta \pi(a|s) \quad\text{and}\quad \theta_{t+1} \yleftarrow{proj}[3pt] \theta_t + \eta_t \nabla_{\theta_t} \mathcal{J}(\pi_t), \label{eq:pg}
    \end{align}
using standard Markov Decision Process notations (recalled in \app{} \ref{app:notations}), where $\eta_t$ is the learning rate and $d_{\pi,\gamma}(s)\doteq \sum_{t=0}^\infty \gamma^t \mathbb{P}(S_t = s | \pi)$ the discounted state density, the policies $\pi\doteq \pi_\theta$ and $\pi_t\doteq \pi_{\theta_t}$ are implicitly parametrized by $\theta$ for conciseness, and $\yleftarrow{proj}[3pt]$ denotes the projection onto the parameter space. We observe that the update established by the policy gradient theorem is proportional to the state density of the current policy. This is problematic, as it implies that no value improvement can be induced in \textit{off-policy} states (\textit{i.e.} states that are rarely visited). As a consequence, planning is inefficient in those states to the point of potentially compromising convergence to the optimal policy.

To address the planning issue in off-policy states, we generalize the policy gradient theory by considering the more general policy update:
\begin{align}
    U(\theta,d) &\doteq \sum_{s\in\mathcal{S}} d(s) \sum_{a\in\mathcal{A}} q_{\pi}(s, a) \nabla_{\theta} \pi(a|s) \quad\text{and}\quad \theta_{t+1} \yleftarrow{proj}[3pt] \theta_t + \eta_t U(\theta_t,d_t),\label{eq:policyupdate}
\end{align}
where $d_t$ designates the state distribution on which the policy is updated. It does not have to match the distribution induced by $\pi$ and updates may therefore be off-policy. We note that $U(\theta,d)$ is not a gradient in the general case, hence the use of the term \textit{update}. We prove in the direct: $\pi_\theta\doteq\theta$ and softmax: $\pi_\theta\mathrel{\dot\sim}\exp(\theta)$ parametrizations that following the exact policy updates induces a sequence of value functions $q_t \doteq q_{\pi_t}$ that is monotonously increasing, upper bounded, and thereby converges to a $q_\infty$. We then show that the condition ``$\forall s\in\mathcal{S},\;\sum_t \eta_t d_t(s)=\infty$'' on the sequence $(d_t)$ is necessary and sufficient for $q_\infty$ to be optimal. Our result generalizes previous theorems to broader updates and milder assumptions \cite{agarwal2019optimality,mei2020global}.
Finally, we significantly improve the existing asymptotic convergence rates. We show that $(q_t)$ converges to the optimal value i) exactly in finite time for the direct parametrization (previously in $\mathcal{O}\left(t^{-1/2}|\mathcal{S}||\mathcal{A}|(1-\gamma)^{-6}\right)$\cite{agarwal2019optimality}), and ii) in $\mathcal{O}\left(t^{-1} |\mathcal{S}| |\mathcal{A}| (1-\gamma)^{-2}\right)$ for the softmax parametrization, to be compared with $\mathcal{O}\left(t^{-1} |\mathcal{S}|^2 |\mathcal{A}|^2 (1-\gamma)^{-6}\right)$ \cite{mei2020global}.

Building on our theoretical results, we design a novel algorithm: Dr Jekyll and Mr Hyde (\jh{}). Its principle consists in training two independent policies: a pure exploitation one (Dr Jekyll) and a pure exploration one (Mr Hyde), and give control to either one for full trajectories. 
\jh{}'s independent policies allow to record two separate replay buffers: one on-policy (Dr Jekyll's) and one off-policy (Mr Hyde's), and therefore to update \jh{}'s models with any desired mixture of on-policy and off-policy updates. 
Beyond an algorithm, \jh{} introduces conditions based on our analysis that actor-critic algorithms should follow to properly plan off-policy.
Furthermore, the separation of exploration and exploitation allows to stabilise the training of the exploitation policy while ensuring a full coverage of the state-action space through deep exploration~\cite{Osband2019}. 

We empirically validate our theoretical analysis and algorithmic innovation in both planning and RL settings. In the planning setting where we assume that $q_t$ is exactly known, we analyze the impact of the off-policiness of $d_t$ and conclude that, while it improves over classic policy gradient, the theoretical sufficient condition does not allow to converge in a reasonable amount of time in very hard planning tasks: we thereby recommend to enforce the stronger condition ``$\forall s\in\mathcal{S},\;\sum_t \eta_t d_t(s)\in\Omega(t)$''.
In the reinforcement learning domain, where $q_t$ must be estimated from the collected transitions, the off-policiness of \jh{}'s policy updates allows by design to satisfy the recommendation as long as Mr Hyde is able to cover the whole state space. This leads \jh{} to significantly outperform all competing actor-critic algorithms in the hard planning tasks and to be competitive with gradient updates in simple planning problems. As a proof of concept, we also test \jh{} in a deep reinforcement learning setting and show that it outperforms various baselines on a simple environment.

The paper is organized as follows. \sect{} \ref{sec:theory} develops the policy update theory. \sect{} \ref{sec:algo} describes \jh{} and positions it with respect to the literature. \sect{} \ref{sec:expes} presents the experiments and their results. Finally, \sect{} \ref{sec:conclusion} concludes the paper. The interested reader may refer to the appendix for the proofs (\app{} \ref{app:theory}), the domains (\app{} \ref{app:domains}), and the full experiment reports (\app{} \ref{app:expes-exact}, \ref{app:expes-sample}, and \ref{app:expes-deep}). Finally, the supplementary material contains the code for the experiments: algorithms and environments.

\section{Theoretical analysis}
\label{sec:theory}
First, we recall some background: policy gradient methods depend on a parametrization $\theta \in \R^{|\mathcal{S}| \times |\mathcal{A}|}$ of the policy. Like \cite{agarwal2019optimality,mei2020global}, we will focus on the classic direct and softmax parametrizations:
\begin{itemize}[leftmargin=.4in]
    \item \emph{direct}: $\pi(a|s) \doteq \theta_{s,a}$, for which $u_{s,a} = d_t(s) q_\pi(s, a)$ and the projection is on the simplex.
    \item \emph{softmax}: $\pi(a|s) \doteq \cfrac{\exp(\theta_{s, a})}{\sum_{a'} \exp(\theta_{s, a'})}$, for which $u_{s,a} = d_t(s) \pi(a|s)(q_\pi(s, a) - v_\pi(s))$.
\end{itemize} 

For both parametrizations, \cite{agarwal2019optimality} proves that following the policy gradient $\nabla_\theta \mathcal{J}(\pi)$ from \eq{} \eqref{eq:pg} eventually leads to the optimal policy in finite MDPs under the following assumptions and condition\footnote{An assumption \ass{\#} is a requirement on the environment or the application setting, while a condition \con{\#} is a requirement that may be enforced by a dedicated algorithm in any environment or setting.}:
\begin{enumerate}
    \item[] \ass{1}. The model $(p_0, p, r)$ is known.
    \item[] \ass{2}. The initial state-distribution covers the full state space: $\forall s\in\mathcal{S}, p_0(s)>0$.
    \item[] \con{3}. The learning rate is constant: $\eta_t = \frac{(1 - \gamma)^3}{2 \gamma |\mathcal{A}|}$ (direct) and $\eta_t = \frac{(1 - \gamma)^3}{8}$ (softmax).
\end{enumerate}
Those are stringent requirements. \ass{1} implies that the values $q_\pi$ and state density $d_{\pi,\gamma}$ of the policy are known exactly. \ass{2} is generally not satisfied in standard RL domains. Finally, verifying \con{3} makes learning slow to the point that it becomes impractical.

\subsection{Convergence properties}
In this section, we study, under \ass{1}, the convergence properties of the sequence of value functions $(q_t)$ induced by the update rule defined in \eq{} \eqref{eq:policyupdate}. We note that $U(\theta,d)$ is not a gradient in general. Our first theoretical result is the monotonicity of the value function sequence $(q_t)$.
\begin{restatable}[\textbf{Monotonicity under the direct and softmax parametrization}]{theorem}{monotonicity}
    Under \ass{1}, the sequence of value functions $q_t \doteq q_{\pi_t}$ and $v_t \doteq v_{\pi_t}$ are monotonously increasing.
    \label{thm:monotonicity}
\end{restatable}

By the monotonous convergence theorem, \thm{} \ref{thm:monotonicity} directly implies the convergence of $(q_t)$:
\begin{restatable}[\textbf{Convergence under the direct and softmax parametrization}]{corollary}{convergence}
    Under \ass{1}, the sequence of value functions $q_t$ uniformly converges: $q_\infty \doteq \lim_{t\to\infty}q_t$.
    \label{cor:convergence}
\end{restatable}

In order to go further and prove convergence to a global optimal value, we need to enforce an additional condition. Indeed, the update $U$ relies multiplicatively on $d_t$ which could be equal (or tend very fast) to 0 in some states, compromising the policy's ability to reach the optimal value. This argument is not only theoretical, it has been observed by many that policy gradient can get stuck in suboptimal policies, even with entropy regularization. We designed our chain domain to exhibit such behaviour (see \sect{} \ref{sec:domain}).

Next, we show that $(q_t)$ converges to the optimal value function under some necessary and sufficient condition.
\begin{restatable}[\textbf{Optimality under the direct and softmax parametrization}]{theorem}{optimality}
    Under \ass{1}, the following condition:
    \begin{enumerate}
        \item[] \con{4}. Each state $s$ is updated with weights that sum to infinity over time: $\sum_{t=0}^\infty \eta_t d_t(s)=\infty$,
    \end{enumerate}
    is necessary and sufficient to guarantee that the sequence of value functions $(q_t)$ converges to optimality: $q_\infty = q_\star\doteq\max_{\pi\in\Pi}q_\pi$.
    \label{thm:optimality}
\end{restatable}
\begin{proof}[Proof sketch]
    In both direct and softmax parametrizations, we assume there exists a state-action pair $(s,a)$ advantageous with respect to the state value limits $q_\infty$ and $v_\infty \doteq \lim_{t\to\infty}v_t$, that is: $\text{adv}_\infty(s,a) \doteq q_\infty(s,a) - v_\infty(s) > 0$. We then prove that in this state $s$, the policy improvement yielded by the update is lower bounded by a linear function of the update weight $\eta_t d_t(s)$. Both parametrizations require different proof techniques and are dealt with in different theorems. Summing over $t$, we notice that this lower bounding sum diverges to infinity, which contradicts Corollary \ref{cor:convergence}. We therefore infer that there cannot exist such a state-action pair, which allows us to conclude that no policy improvement is possible, and thus that the values are optimal: $q_\infty = \max_{\pi\in\Pi}q_\pi$.
    
    For the necessity of \con{4}, we show that the parameter update is upper bounded in both parametrizations by a term linear in the action gap. By choosing the reward function adversarially, we may set it sufficiently small so that the sum of all the gradient steps is insufficient to reach optimality.
\end{proof}

\thm{} \ref{thm:optimality} is impactful along five dimensions.

\myuline{Practice of on-policy undiscounted updates:} As \thm{} \ref{thm:optimality} is applicable to any distribution sequence, it allows us in particular to consider the practice of using on-policy undiscounted updates: $d_t\doteq d_{\pi_t,1}$. Thus, it resolves a longstanding gap between the policy gradient theory and the actor-critic algorithm implementations \cite{baselines,caspi_itai_2017_1134899,deeprl,pytorchrl,spinningup2018,liang2018rllib,stooke2019rlpyt}. This mismatch and its lack of theoretical grounding were identified in \cite{thomas2014bias}. Later, \cite{DBLP:journals/corr/abs-1906-07073} proved that the practitioners' undiscounted updates are not the gradient of any function and may be strongly biased under state aliasing. Recently, \cite{Zhang2021} studied the practical advantage of the undiscounted updates from a representation learning perspective. \con{4} encompasses both standard policy gradient and the undiscounted update rule: convergence to optimality of both is guaranteed as long as \con{4} is verified. Our analysis thus shows that the convergence properties of policy gradient and undiscounted updates require the same set of assumptions and conditions. Conversely, it is possible to prove convergence to sub-optimal policies of either when \con{4} is violated. In other words, \thm{} \ref{thm:optimality} allows to reduce the study of specific algorithms to whether they verify \con{4}.

\myuline{Experience replay and off-policy updates:} It also justifies the use of an experience replay for the actor, a trick also widely used in the literature to distributed the training over several agents \cite{mnih2015human,Schaul2015,Wang2016,Horgan2018,Schmitt2020}. Furthermore, while widely overlooked in the literature, we prove that off-policy updates have an even higher impact, since they are necessary to guarantee the convergence to optimality\footnote{Other papers proving convergence to optimality of policy gradient rely on \ass{2}, which implies \con{4}.}. We leverage this discovery to introduce new design principles and a novel algorithm in \sect{} \ref{sec:algo}.

\myuline{Policy gradient theory:} Next, \thm{} \ref{thm:optimality} generalizes previous optimality convergence results along two axes. The initial state distribution $p_0$ is not required to cover the state space anymore (aka \ass{2}). \ass{2} is unrealistic in most applications and cannot be enforced by an algorithm. Relaxing it is an important open problem~\cite{agarwal2019optimality}, we disprove it in the full class of MDP. However, it might be possible to find a class of MDPs larger than the one verifying \ass{2} where policy gradient also converges. Additionally, the condition on our learning rates is also much more flexible than \con{3}.

\myuline{Density vs. policy regularization:} In contrast, \con{4} can be controlled by a dedicated algorithm. This theoretical result promotes the principle of density-based regularization \cite{Liu2020,Qin2021}, at the expense of policy-based regularization that cannot guarantee that \con{4} is satisfied and sometimes fails to plan over the whole state space (our chain domain experiments in \sect{} \ref{sec:expes} illustrate it well).

\myuline{Towards a generalized policy iteration theorem:} Finally, by combining \thm{} \ref{thm:optimality}, which analyzes policy improvement, with a thorough theoretical analysis of the policy evaluation step (discussed in Section \ref{sec:conditions}), we see a path towards results describing conditions on both components of generalized policy iteration that guarantee convergence to optimality. Furthermore, \thm{} \ref{thm:optimality} gives sufficient conditions on the policy improvement step. The generalized policy iteration theorem was conjectured in \cite{sutton2018reinforcement} but never proved. 

\subsection{Convergence rates}
Next, we give convergence rates for both softmax and direct parametrizations:
\begin{restatable}[\textbf{Asymptotic convergence rates under the direct parametrization}]{theorem}{directrates}
    With direct parametrization, under \ass{1} and \con{4}, the sequence of value functions $q_t$ converges to optimality in finite time:
    \begin{align}
        \exists t_0, \text{ such that } \forall t\geq t_0, \quad q_t = q_\star.
    \end{align}
    \label{thm:direct_rates}
\end{restatable}
\con{4} is required for optimality and convergence rate contrarily to the softmax parametrization. We significantly improve over previous bounds in $\mathcal{O}\left(t^{-1/2}|\mathcal{S}||\mathcal{A}|(1-\gamma)^{-6}\right)$\cite{agarwal2019optimality}. We wish to emphasize that Theorem~\ref{thm:direct_rates} does not say anything about how large $t_0$ is, the rate is purely asymptotic. 

\begin{restatable}[\textbf{Asymptotic convergence rates under the softmax parametrization}]{theorem}{softmaxrates}
    With softmax parametrization, under \ass{1}, \con{4}, and \ass{8}:
    \begin{enumerate}
        \item[] \ass{8}. The optimal policy is unique: $\forall s, \; q_\star(s,a_1)=q_\star(s,a_2)=v_\star(s)$ implies $a_1 = a_2$,
    \end{enumerate}
    the sequence of value functions $q_t$ converges asymptotically as follows:
    \begin{align}
        \exists t_0, \text{ such that } \forall t\geq t_0, \quad v_\star(s) - v_t(s) \leq \frac{8 |\mathcal{A}|(v_{\mytop}-v_{\mybot})}{(1-\gamma)\min_{s\in\text{supp}(d_{\pi_\star,\gamma})}\delta(s)\sum_{t'=t_0}^{t-1} \eta_{t'} d_{t'}(s)},
    \end{align}
    where $\delta(s)=v_\star(s)-\max_{a\in\mathcal{A}/\{\pi_\star(s)\}}q_\star(s,a)$ is the gap with the best suboptimal action in state $s$, $v_{\mytop}$ (resp. $v_{\mybot}$) is the maximal (resp. minimal) value, and $\textnormal{supp}(d_{\pi_\star,\gamma})$ denotes the support of the distribution of the optimal policy.
    \label{thm:softmax_rates}
\end{restatable}
If $(d_t)$ satisfies the additional mild condition that the support of $d_\star$ is covered \textit{on average}:
\begin{enumerate}
    \item[] \con{5}. $\forall s\in\textnormal{supp}(d_\star), \lim_{t\to\infty}\frac{1}{t}\sum_{t'=0}^t \eta_{t'}d_{t'}(s) \doteq e_{\mybot}(s) > 0$,
\end{enumerate}
then we obtain the following convergence rate in $\mathcal{O}(t^{-1})$:
\begin{align}
    \exists t_0, \text{ such that } \forall t\geq t_0, \quad v_\star(s) - v_t(s) \leq \frac{8|\mathcal{A}| (v_{\mytop}-v_{\mybot})}{(t-t_0)(1-\gamma) \min_{s\in\textnormal{supp}(d_{\pi_\star,\gamma})} \delta(s) e_{\mybot}(s)}.
\end{align}

\con{5} is implied by \ass{2} and \con{3}. Alternately, assuming $\eta_t\geq\eta_{\mybot}>0 \;\forall t$, it is verified by a uniform distribution in which case $e_{\mybot}(s)=\frac{\eta_{\mybot}}{|\mathcal{S}|}$, or by an on-policy distribution $d_{\pi_t}$ that converges to the distribution of some optimal policy $d_{\pi_\star}$ in which case $e_{\mybot}(s)=\eta_{\mybot}d_{\pi_\star}(s)$. Also, note that \con{4} is required for optimality, but not for convergence rates with respect to $q_\infty$. 
\thm{} \ref{thm:softmax_rates} establishes asymptotic rates in $\mathcal{O}\left(t^{-1} |\mathcal{S}| |\mathcal{A}| (1-\gamma)^{-2}\right)$, to be compared with $\mathcal{O}\left(t^{-1} |\mathcal{S}|^2 |\mathcal{A}|^2 (1-\gamma)^{-6}\right)$ \cite{mei2020global}. The gap is partially explained by dropping \con{3} and allowing any learning rate (see \app{} \ref{app:softmax-bounds} for a thorough comparison). It is possible, see \thm{} \ref{thm:ass3_softmax} of \app{} \ref{app:theory}, to show that under \ass{2}, a condition on the learning rate, and a bounding condition on $d_t$ ($\exists d_{\mytop} \geq d_{\mybot} > 0$ such that $\forall s,t,\;d_{\mytop} \geq d_t(s) \geq d_{\mybot}$), we have: $t_0 \leq \mathcal{O}\left(\frac{|\mathcal{S}|}{(1-\gamma)^7} \frac{d_{\mytop}}{d_{\mybot}} \frac{1}{\min_s \delta(s)}\right)$.

\subsection{Related work}

\cite{Degris2012,Imani2018,Zhang2019} study the off-policy, continuing setting paired with the average value objective. That objective is either expressed with a state visitation coming from the behavioral policy (the excursion objective) or from the trained policy (the alternative life objective). Its gradient can be computed exactly or approximated and used to optimize the objective. The weights assigned to each state during an update stem from the gradient computation (exactly as in the discounted return gradient we study). It is an interesting question, left for future work, to understand whether those weights have desirable properties from a convergence standpoint. In particular, it is possible that a condition akin to \con{4} exists in the continuing objective case.

From a proof technique perspective, \cite{agarwal2019optimality} studies gradient ascent and, as a byproduct, can rely on standard optimization results: strong convexity and convergence of gradients to 0. We cannot do so as our updates are not gradients anymore. More precisely, Lemma C2 in~\cite{agarwal2019optimality} uses the strong convexity of the objective function to prove that following the gradient results in a value improvement (assuming a learning rate sufficiently small). We provide a more general (we do not need the condition on the learning rate) and, in our admittedly biased opinion, more elegant proof of that lemma in our Theorem 1. Second, their Lemma C5 uses (i) the convergence of the gradient to 0 (a standard result with gradient as-/des-cent), and (ii) the assumption that all states have a non-vanishing density to infer that, in the limit, the learnt policy does not assign any mass to states that have a non-zero advantage. Neither (i) nor (ii) hold in our setting, we had to leverage \con{4} instead.

\section{Dr Jekyll and Mr Hyde: an actor-critic with convergence guarantees}
\label{sec:algo}
\subsection{Conditions for an actor-critic algorithm to converge to optimality}
\label{sec:conditions}
\con{4} states the necessary and sufficient condition for optimal planning with the exact model. In a reinforcement learning setting however, the model is not known; the policy must be learnt from samples collected in the environment. As a consequence, \ass{1} cannot be made anymore, we need to rely on an exploration condition instead:
\begin{enumerate}
    \item[] \con{6}. Each state-action pair is explored infinitely many times: $\forall s,a,\,\lim_{t\to\infty} n_t(s,a) = \infty,$
\end{enumerate}
where $n_t(s,a)$ is the count of samples collected for the state-action pair $(s,a)$ at time $t$.

In finite MDPs, \con{6} provides a necessary and sufficient condition for having an unbiased estimator of the value. As is customary, we call this estimator the critic and denote it $\mathring{q}_t$. \con{6} is necessary because the required statistical precision can only be achieved when the number of samples tends to infinity; it is sufficient as guaranteed by many off-policy policy evaluation algorithms from the literature \cite{Precup2000,Thomas2016}. Classic concentration bounds, such as Hoeffding's inequality, tell us that with $\mathcal{O}(1/\xi^2)$ samples in every state, a $\xi$-accurate critic can be achieved. Under \con{4} and \con{6}, there thus exists a timestep after which updates based on $\mathring{q}_t$ are sufficiently close to the true ones for a continuity argument to show that the policy improves eventually\footnote{A formal proof would require dealing with several technical challenges due to the stochasticity of the value updates that break the monotonicity of the estimator accuracy. We leave it for future work.} (as guaranteed by theory under \ass{1}).

Our results till now have assumed a finite MDP and models with sufficient capacity. In larger domains, those assumptions do not hold anymore, and therefore our theory does not apply. For instance, \cite{DBLP:journals/corr/abs-1906-07073} proves that using the on-policy undiscounted update $d_{\pi_t,1}$ instead of the policy gradient can lead to highly suboptimal policies under state aliasing. Their counter-example applies to our policy updates as well. Nevertheless, we believe that state aliasing is a worst-case scenario and conjecture that it does not happen to neural networks thanks to their high expressive capacity. Also note that this concern with respect to distribution shift is general and could be formulated for any neural model, including supervised models, or in RL the purely value-based ones such as DQN that are frequently trained off-policy. The consensus in the literature is that neural models do not suffer too much from distribution shift as long as the testing set distribution is well covered by the training set. Since appropriate coverage of the state-action space is actually the final objective of our off-policy policy updates, we expect minimal impact on this dimension, though it does remain to be formally demonstrated.
Our empirical results from \sect{} \ref{sec:expe-deep} support this conjecture. The theoretical study of the function approximation setting is left for future work.

\subsection{Our solution to enforce \con{4} and \con{6}}
\paragraph{Enforcing \con{4}} \con{4} defines the necessary and sufficient condition for asymptotic convergence to optimality. However, we will see in \sect{} \ref{sec:expes-exact} that in difficult planning tasks, \con{4} can be insufficient for convergence to happen in a reasonable amount of time. Below, we therefore introduce \con{4-s}, a condition stronger\footnote{\con{4-s} is identical to \con{5}, but applied to the full state space $\mathcal{S}$ (it thus implies \con{5}).} than \con{4}, as well as two techniques for \con{4-s} to be verified:
\begin{enumerate}
    \item[] \con{4-s}. $\forall s\in\mathcal{S}, \quad \lim_{T\to\infty}\frac{1}{T}\sum_{t=0}^T \eta_t d_t(s)\geq d_{\mybot}(s)>0.$
\end{enumerate}

The first technique to enforce \con{4-s} is to apply the approximate expected policy update \cite{Silver2014,lillicrap2015continuous,Ciosek2018}:
    \begin{align}
         \widehat{U}(\theta,s) &\doteq \sum_{a\in\mathcal{A}} \mathring{q}(s, a) \nabla_{\theta} \pi(a|s). \label{eq:expected_update}
    \end{align}
The expected policy update is deterministic in the sense that it does not require sampling actions. As there exists a deterministic optimal policy, this in turn implies that the learning rate $\eta_t$ does not need to satisfy the second Robbins-Monro condition: $\sum_t \eta_t^2 < \infty$, and thus may be constant $\eta_t\doteq\eta$.

The second technique to enforce \con{4-s} is to ensure that the density of updates $d_t$ does not decay to 0 for any state $s$. This can be obtained by maintaining a constant proportion $o_t$ of off-policy actor updates in order to cover the whole state space. We propose to do so by recording two replay buffers: one with on-policy samples, \textit{i.e.} samples collected with an exploitation policy, and another with off-policy samples, \textit{i.e.} samples collected with an exploration policy. 

To the best of our knowledge, this type of prioritized experience replay has never been used in this fashion, nor to this effect. We note that off-policy policy gradient is a concept that exists in the literature, but it refers to the application of policy gradients from batch samples~\cite{Liu2020}. Following the success of DQN~\cite{mnih2015human}, experience replays have also been used in actor-critic methods~\cite{Wang2016} which bears some resemblance to our suggestion. Finally, off-policy actor-critic with shared prioritized experience replay~\cite{Schaul2015} has been applied to large-scale experiments with distributed agents~\cite{Horgan2018,Schmitt2020}.

\paragraph{Enforcing \con{6}} Many RL algorithms (actor-critic or purely value-based) ensure some form of exploration. They broadly form two groups: dithering exploration (\textit{e.g.}, epsilon greedy, softmax, entropy regularization \cite{Tokic2010,sutton2018reinforcement,haarnoja2018soft}), and deep exploration (\textit{e.g.}, UCB, thompson sampling, density constraints \cite{Ortner2007,Bellemare2016,Osband2016,Qin2021}). Strong cases have been made against dithering exploration, arguing in particular its inability to ensure visits to all states and therefore convergence to optimality \cite{Osband2019}. In spite of these arguments (that we recall in \app{} \ref{app:policy_regularization}), there are still only very few actor-critic algorithms that realize deep exploration \cite{Ciosek2019,Roy2020}. We conjecture that this relates to the following observations: i) exploration involves a moving objective, hence a non-stationarity of the desired policy, ii) actor-critic algorithms have a structural inertia in their policy (in contrast to value-based methods that can completely switch policy when an action's value overtakes another's). As a consequence, a dual exploration-exploitation actor-critic algorithm takes a lot of time to switch from deep exploration to exploitation, and back, making it inefficient. To avoid this issue, we propose to train two policies, a pure exploration one and a pure exploitation one. As an added benefit, they will be used to constitute the on-policy and off-policy replay buffers introduced in the second technique above.

\begin{algorithm}[t]
\caption{Dr Jekyll \& Mr Hyde algorithm. After initialization of parameters and buffers, we enter the main loop. At every time step, an action, chosen by the behavioral policy, is executed in the environment to produce a transition $\tau_t$ (line 5). $\tau_t$ is stored in the replay buffer of the personality in control (either Dr Jekyll or Mr Hyde, line 6). If the trajectory is done, the algorithm samples a new personality to be in control during the next one (line 7). Then, the updates of the models start (line 8). The updates for Mr Hyde's policy $\tilde{\pi}$ and Dr Jekyll's critic $\mathring{q}$ are underspecified, and may be any algorithm in the literature satisfying the exploration (for Mr Hyde) or unbiased (for Dr Jekyll) conditions. When $(\epsilon_t) = 0$, \jh{} amounts to on-policy expected updates from a single replay buffer.}
\multiline{\textbf{Input:} Scheduling of exploration $(\epsilon_t)$, of off-policiness $(o_t)$ and of actor learning rate $(\mathring{\eta}_t)$.}
\begin{algorithmic}[1]
    \State Initialize Dr Jekyll's replay buffer $\mathring{D}=\emptyset$, actor $\mathring{\pi}$, and critic $\mathring{q} $. \Comment{exploitation agent}
    \State Initialize Mr Hyde's replay buffer $\tilde{D}=\emptyset$, and policy $\tilde{\pi}$. \Comment{exploration agent}
    \State Set the behavioural policy and working replay buffer to Dr Jekyll's: $\pi_{b} \leftarrow \mathring{\pi}$ and $D_{b} \leftarrow \mathring{D}$.
    \For{$t=0$ to $\infty$}
        \State Sample a transition $\tau_t = \langle s_t,a_t\sim\pi_{b}(\cdot|s_t),s_{t+1}\sim p(\cdot|s_t,a_t),r_t\sim r(\cdot|s_t,a_t)\rangle$.
        \State Add it to the working replay buffer $D_b\leftarrow D_b\cup\{\tau_t\}$.
        \InlineIfThen{$\tau$ was terminal,}{$(\pi_{b}, D_b) \leftarrow (\tilde{\pi}, \tilde{D})$ w.p. $\epsilon_t$, $(\pi_{b}, D_b) \leftarrow (\mathring{\pi}, \mathring{D})$ otherwise.}
        \If{Update step,}
            \State $\tau \doteq \langle s,a,s',r \rangle \sim \tilde{D}$ w.p. $o_t$, $\tau \doteq \langle s,a,s',r \rangle \sim \mathring{D}$ otherwise.
            \State Update Mr Hyde's policy $\tilde{\pi}$ on $\tau$. \Comment{\textit{e.g.} with $Q$-learning trained on UCB rewards}
            \State Update Dr Jekyll's critic $\mathring{q}$ on $\tau$. \Comment{\textit{e.g.} with SARSA update}
            \State Expected update of Dr Jekyll's actor $\mathring{\pi}$ in state $s$. \Comment{\eq{} \eqref{eq:expected_update}}
        \EndIf
    \EndFor
\end{algorithmic}
\label{alg:J&H}
\end{algorithm}

\subsection{Dr Jekyll and Mr Hyde algorithm (\jh{})}
The objective of this section is to introduce a novel algorithm that satisfies conditions \con{4} (or \con{4-s}) and \con{6} by design. To do so, we maintain a mixture of two policies:
\begin{itemize}
    \item Dr Jekyll $\mathring{\pi}_t$ is a pure exploitation policy,
    \item Mr Hyde $\tilde{\pi}_{t}$ is a pure exploration policy.
\end{itemize}

At the beginning of a new trajectory, Mr Hyde $\tilde{\pi}_{t}$ (resp. Dr Jekyll $\mathring{\pi}_{t}$) is chosen with probability $\epsilon_t$ (resp. $1-\epsilon_t$) and used to generate the full trajectory. Dr Jekyll $\mathring{\pi}_t$ is trained with the update rule of \eq{} \eqref{eq:policyupdate}, where $d_t$ is defined by prioritized sampling over the experience replays, $q_t$ is replaced with an unbiased estimator $\mathring{q}$ of the value of $\mathring{\pi}_t$, and $\mathring{\theta}_t$ are the parameters of $\mathring{\pi}_t$. $\tilde{\pi}_{t}$ is a pure exploratory policy, designed to verify almost surely: $\forall (s,a)\in \mathcal{S}\times\mathcal{A},\;\lim_{t\to\infty} n_{|\tilde{D}|}(s,a) \geq \tilde{d}_{\mybot} > 0,$, where $n_{|\tilde{D}_t|}(s,a)$ is the count of samples $(s,a)$ in its experience replay.
We note that a uniform policy satisfies this condition, but with a very small constant $\tilde{d}_{\mybot}$ which compromises convergence in a reasonable time. Any deep exploration algorithm should be guaranteeing it\footnote{Efficient exploratory algorithms are a challenge in large environments. The design of such algorithms is beyond the scope of this paper.}. In our finite MDP experiments, we implement $q$-learning with UCB rewards \cite{Bellemare2016}: $\tilde{\pi}_{t}\doteq \argmax_{a\in\mathcal{A}} \tilde{q}_{t}$ where $\tilde{q}_{t}$ is trained to predict the expectation of the discounted sum of rewards: $\tilde{r}(s,a)\doteq \frac{1}{\sqrt{n_t(s,a)}}$ (see full \alg{} \ref{alg:finiteMDP-JH} in \app{} \ref{app:expes-sample}). In our deep learning experiments, Hyde is a Double-DQN~\cite{vanhasselt2015reinforcement} trained to maximize an exploration bonus based on Random Network Distillation~\cite{Burda2019} (\sect{} \ref{sec:expe-deep}).

The advantages of separating the exploration from the exploitation policy are the following. First, it is easier to specify the exploration/exploitation trade-off and get full control on the exploration
requirements for condition \con{6}: $\forall s,a,\,\sum_t d_{\pi_t}(s,a) = \infty$. Second, Jekyll's actor $\mathring{\pi}$ is not optimized under a moving objective which would otherwise induce a high level of instability on the policy. Third, one can define the on-policiness/off-policiness $o_t$ by recording two separate replay buffers $\mathring{D}$ and $\tilde{D}$, for trajectories respectively controlled by Dr Jekyll and Mr Hyde. This offers full control on condition \con{4}/\con{4-s} and on the asymptotic behaviour of $\sum_t \eta_t d_t(s)$.

These various observations lead to the design of the Dr Jekyll and Mr Hyde algorithm (\jh{}), formally detailed in \alg{} \ref{alg:J&H}. 

\section{Experiments}
\label{sec:expes}
\subsection{Domains}
\label{sec:domain}
    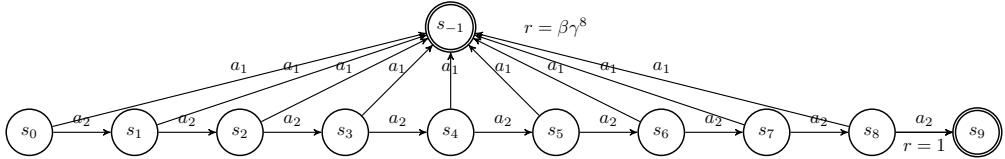
\begin{figure}[h]
        \begin{center}
            \scalebox{0.7}{
                \begin{tikzpicture}[->, >=stealth', scale=0.6 , semithick, node distance=2cm]
                    \tikzstyle{every state}=[fill=white,draw=black,thick,text=black,scale=1]
                    \node[state]    (x0)                {$s_0$};
                    \node[state]    (x1)[right of=x0]   {$s_1$};
                    \node[state]    (x2)[right of=x1]   {$s_2$};
                    \node[state]    (x3)[right of=x2]   {$s_3$};
                    \node[state]    (x4)[right of=x3]   {$s_4$};
                    \node[state]    (x5)[right of=x4]   {$s_5$};
                    \node[state]    (x6)[right of=x5]   {$s_6$};
                    \node[state]    (x7)[right of=x6]   {$s_7$};
                    \node[state]    (x8)[right of=x7]   {$s_8$};
                    \node[state,accepting]    (x9)[right of=x8]   {$s_9$};
                    \node[state,accepting]    (x-1)[above of=x4]   {$s_{-1}$};
                    \node[] (y)[above of=x5] {$r=\beta\gamma^8$};
                    \path
                    (x0) edge[above]    node{$a_1$}     (x-1)
                    (x1) edge[above]    node{$a_1$}     (x-1)
                    (x2) edge[above]    node{$a_1$}     (x-1)
                    (x3) edge[above]    node{$a_1$}     (x-1)
                    (x4) edge[above]    node{$a_1$}     (x-1)
                    (x5) edge[above]    node{$a_1$}     (x-1)
                    (x6) edge[above]    node{$a_1$}     (x-1)
                    (x7) edge[above]    node{$a_1$}     (x-1)
                    (x8) edge[above]    node{$a_1$}     (x-1)
                    (x0) edge[above]    node{$a_2$}     (x1)
                    (x1) edge[above]    node{$a_2$}     (x2)
                    (x2) edge[above]    node{$a_2$}     (x3)
                    (x3) edge[above]    node{$a_2$}     (x4)
                    (x4) edge[above]    node{$a_2$}     (x5)
                    (x5) edge[above]    node{$a_2$}     (x6)
                    (x6) edge[above]    node{$a_2$}     (x7)
                    (x7) edge[above]    node{$a_2$}     (x8)
                    (x8) edge[above]    node{$a_2$}     (x9)
                    (x8) edge[below]    node{$r=1$}   (x9);
                \end{tikzpicture}
            }
            \caption{Deterministic chain MDP. Initial state is $s_0$. Reward is 0 everywhere except when accessing final states $s_{-1}$ and $s_9$. Reward in $s_{-1}$ is set such that $q(s_0,a_1)=\beta q_\star(s_0,a_2)$, with $\beta\in[0,1)$.}
            \label{fig:chain}
        \end{center}
    \end{figure}
    \textbf{Chain Domain} The chain domain is designed to measure the ability of algorithms to overcome immediate rewards pushing the policy gradient towards suboptimal policies. In every state $s_k$, the agent has the opportunity to play action $a_1$ and receive an immediate reward of $\beta\gamma^{|\mathcal{S}|-2}$, or to play $a_2$ and progress to next state $s_{k+1}$ without any immediate reward. A reward of 1 is eventually obtained when reaching state $s_{|\mathcal{S}|-1}$. \fig{}~\ref{fig:chain} represents a chain of size 10. We report the normalized performance: $\overline{\mathcal{J}}_\pi = \frac{\mathcal{J}_\pi-\mathcal{J}_{\mybot}}{\mathcal{J}_\star - \mathcal{J}_{\mybot}}$, where $\mathcal{J}_\star$ is the optimal performance and $\mathcal{J}_{\mybot}\doteq q(s_0,a_1) = \beta \gamma^{|\mathcal{S}|-2}$.
    
    \textbf{Random MDPs} The random MDPs domain is designed to test the algorithms in situations where exploration is not an issue. Indeed, by its design of stochastic transition functions, random MDPs will have a non-null chance to visit every state whatever the behavioural policy. It is therefore a domain where we expect policy gradient updates to perform well, perhaps optimally, and hope that our modified updates still perform comparably. We reproduce the random MDPs environment published in App. B.1.3 of \cite{Laroche2019}. We report the normalized performance: $\overline{\mathcal{J}}_\pi = \frac{\mathcal{J}_\pi-\mathcal{J}_{u}}{\mathcal{J}_\star - \mathcal{J}_{u}}$, where $\mathcal{J}_\star$ is the optimal performance and $\mathcal{J}_{u}$ is the performance of the uniform policy. The full description of both domains is available in \app{} \ref{app:domains}.
    
\subsection{Finite MDP policy planning (\ass{1})}
\label{sec:expes-exact}
We test performance against time, learning rate $\eta$ of the actor, MDP parameters $|\mathcal{S}|$ and $\beta$, off-policiness $o_t$, and policy entropy regularization weight $\lambda$, with both direct and softmax parametrizations, on the chain and random MDPs. The full report is available in \app{} \ref{app:expes-exact}. 

On the chain domain, we confirm that enforcing updates with $d_t\doteq o_t d_u + (1-o_t)\frac{d_{\pi_t,\gamma}}{\lVert d_{\pi_t,\gamma} \rVert_1}$ including an off-policy component $o_t>0$ on a uniform state distribution $d_u(s)\doteq\frac{1}{|\mathcal{S}|}$ helps path discovery and policy planning, while on-policy updates fail at converging to the optimal policy, even with policy entropy regularization. We also observe that while $o_t\in\Omega(t)$ enforces \con{4}, and therefore guarantees convergence to optimality, it may not happen in a reasonable amount of time, and a constant off-policiness is preferable. By sweeping over the chain parameters $\beta$ and $|\mathcal{S}|$, we observe that on-policy updates are sensitive to both: even with a small $\beta=0.1$, a reasonably sized chain $|\mathcal{S}|=15$ cannot be solved. In contrast, $o_t=0.5$ converges fast to optimality even with $\beta=0.95$ and $|\mathcal{S}|=25$.

On the random MDPs domain, we observe empirically that discounted and undiscounted updates perform well but oftentimes stagnate at 99\% of the optimal performance; $d_t$ with an off-policy component performs better and gets even closer to optimality. Finally, we note that purely uniform updates: $o_t=1$ slow down training in the random MDPs experiment. We conclude that it is best to include both on-policy and off-policy components in $d_t$. These experiments also allow us to observe the biased convergence implied by policy entropy regularization.

\begin{figure*}[t]
	\centering
	\subfloat[Chain experiment]{
		\includegraphics[trim = 5pt 5pt 5pt 5pt, clip, width=0.32\columnwidth]{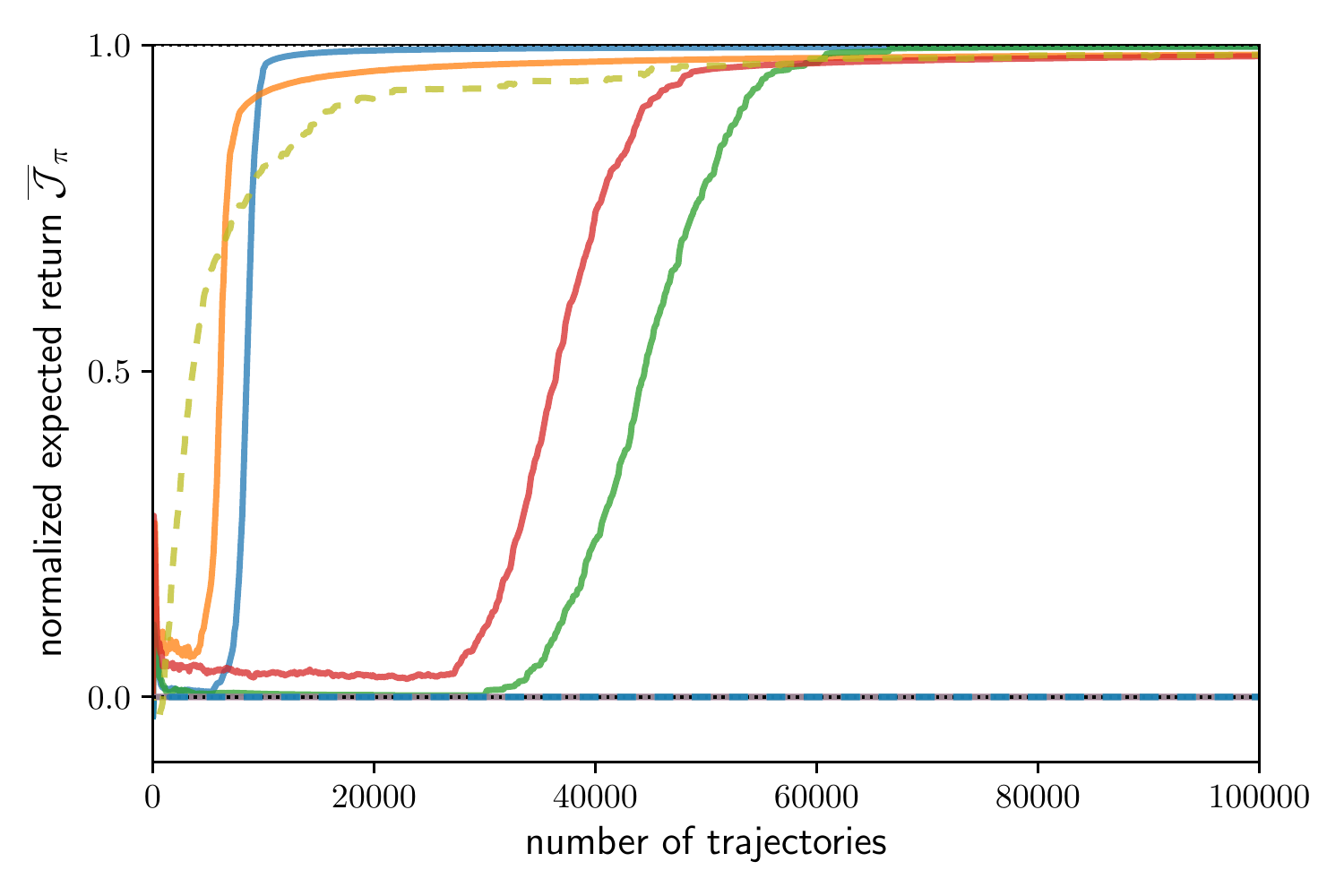}
		\label{fig:sample-chain-ucb-mean}
	}
	\subfloat[Random MDPs experiment]{
		\includegraphics[trim = 5pt 5pt 5pt 5pt, clip, width=0.32\columnwidth]{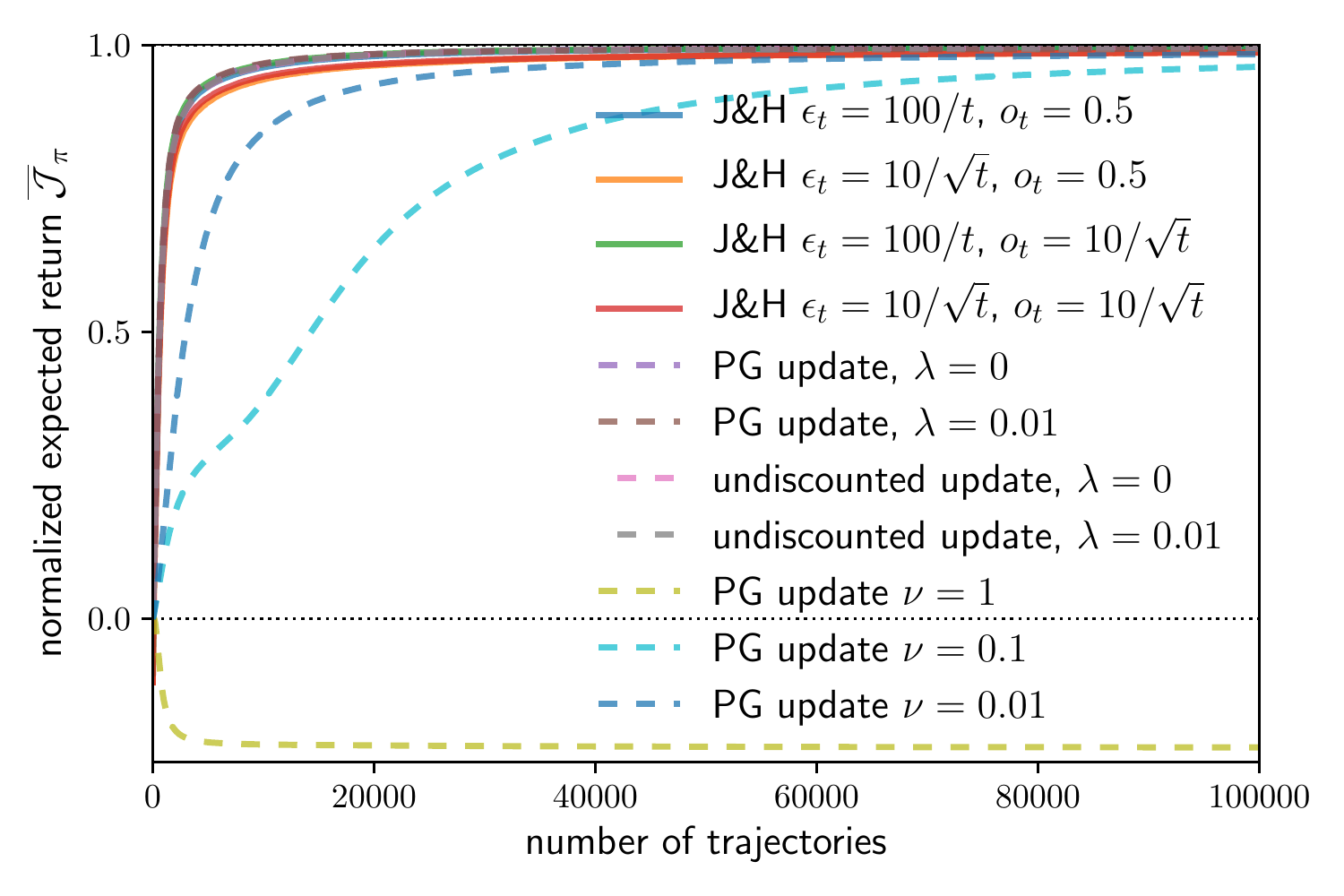}
		\label{fig:sample-garnets-ucb-mean}
	}
	\subfloat[Random MDPs zoomed in]{
		\includegraphics[trim = 5pt 5pt 5pt 5pt, clip, width=0.32\columnwidth]{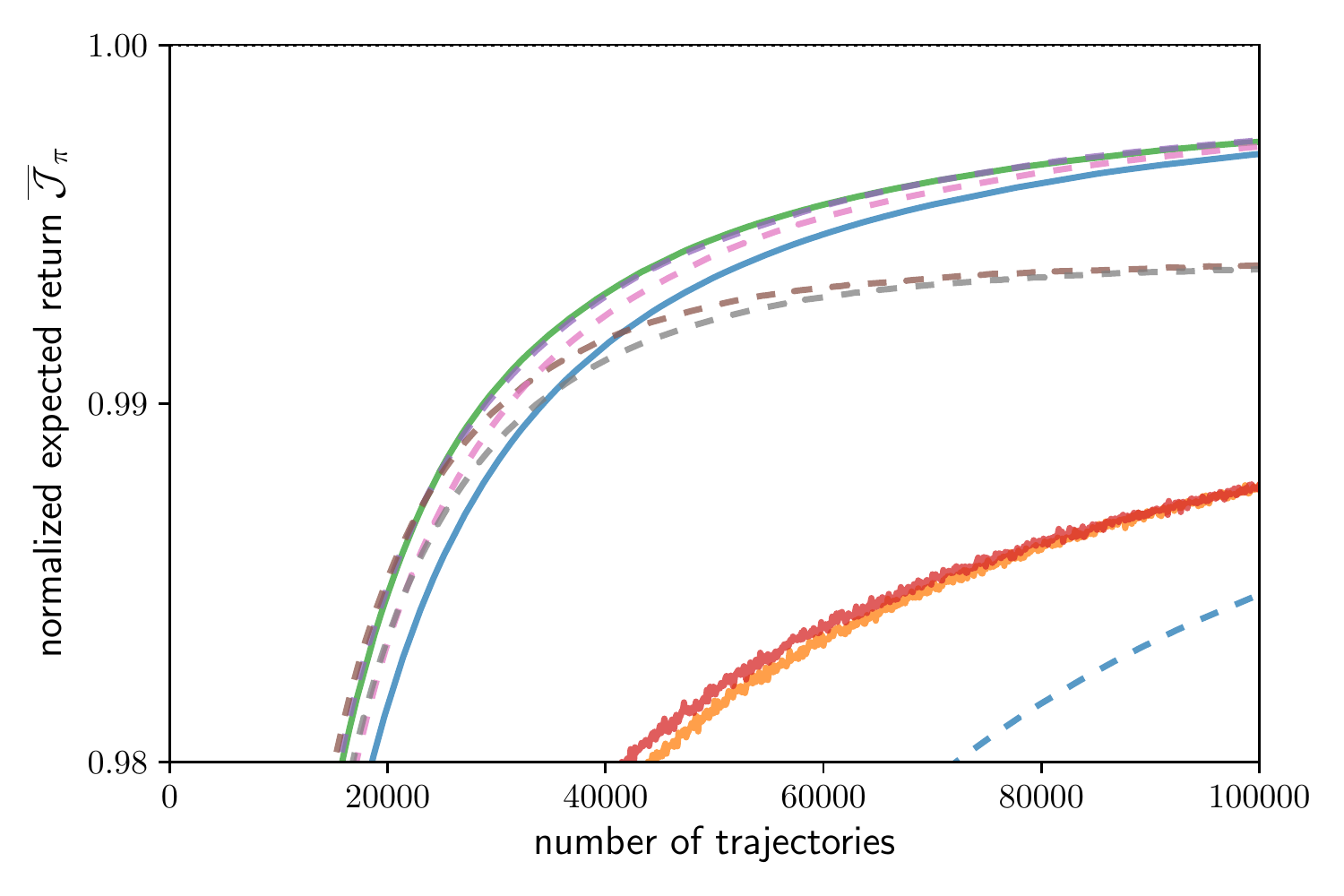}
		\label{fig:sample-garnets-ucb-mean-zoom}
	}
	\caption{RL experiments: normalized expected return vs. number of trajectories (200+ runs).}
		\label{fig:sample-garnets-ucb}
	\vspace{-10pt}
\end{figure*}

\subsection{Finite MDP reinforcement learning experiments}
\label{sec:expes-sample}
With the softmax parametrization, we tested \jh{} with various scheduling for $\epsilon_t$ and $o_t$ against on-policy (PG/undiscounted) updates, with/without policy entropy regularization (hyperparameter $\lambda$), and with/without UCB critic (UCB hyperparameter $\nu$). The full description of the algorithms implementations is available in \app{} \ref{app:algos}. Our policy planning experiments suggested that $o_t$ was best constant at $0.5$, the same holds in our RL experiments. \fig{} \ref{fig:sample-chain-ucb-mean} reveals that on-policy updates with/without policy entropy regularization cannot solve the chain experiment of size 10 with $\beta=0.8$. In contrast, the task is solved by on-policy updates with a strong UCB bonus $\nu=1$ for the critic, and by all \jh{} implementations. Concerning \jh{} scheduling hyperparameters, we observe that $o_t=0.5$ allows to identify the optimal policy significantly faster, and $\epsilon_t=\frac{100}{t}$ is sufficient and allows to converge faster asymptotically than $\epsilon_t=\frac{10}{\sqrt{t}}$. \fig{} \ref{fig:sample-garnets-ucb-mean} shows that all algorithms are able to solve the random MDP task. However, UCB critics with high $\nu\in\{0.1,1\}$ (including the only on-policy setting that succeeded on the chain task) fail to do so properly: they eventually will, but explore too much within the experimental time. If we look more closely at the tight convergence of every algorithm on \fig{} \ref{fig:sample-garnets-ucb-mean-zoom}, we observe that the policy entropy regularization incurs a bias and that off-policiness of updates of \jh{} does not slow down convergence compared to vanilla on-policy updates. Note that \jh{} with $\epsilon_t=\frac{10}{\sqrt{t}}$ performs worse only because we report the expected performance of the mixture of Dr Jekyll and Mr Hyde. The performance of Dr Jekyll alone is comparable, and even slightly better than vanilla on-policy updates (see \fig{} \ref{fig:sample-garnets-time-mean-jek} in \app{} \ref{app:expes-sample}).

For completeness, we also test performance against time, learning rates of the actor $\eta$ and critic $\eta_c$, MDP parameters $|\mathcal{S}|$ and $\beta$, off-policiness $o_t$, exploration $\epsilon_t$, and critic initialization $q_0$ with the softmax parametrization on the chain and random MDPs domains. See \app{} \ref{app:expes-sample} for the full report.

\subsection{Deep reinforcement learning experiments}\label{sec:expe-deep}

To conclude our empirical evaluation, we implemented a deep version of \jh{}, as described in Algorithm~\ref{alg:J&H}, by parametrizing the agents using deep neural networks (see \app{} \ref{app:expes-deep} for full details).

Dr Jekyll is a policy network, with a standard architecture, trained using the updates described in Eq.~\eqref{eq:expected_update}. As previously mentioned, any exploration algorithm can be used for Mr Hyde. In our experiments, we chose Random Network Distillation~\cite[RND]{Burda2019} to generate exploration bonuses. Two networks, one random, the target, and a second one, the predictor, are used to assess the novelty of an observed state via the distance between the output of the target and the output of the predictor on that state. Each time a given state is evaluated, the predictor is trained to predict the output of the target. The more a state is seen, the smaller the prediction error will be. The error is used as a reward signal for Mr Hyde, a standard Double-DQN~\cite[DDQN]{vanhasselt2015reinforcement} trained to maximize it. By doing so, Mr Hyde is incentivized to explore parts of the state space that have not been visited much yet. 

We train \jh{} on a version of the Four Rooms environment~\cite{SUTTON1999181}, a 15x15 grid split into four rooms (see \app{} \ref{app:expes-deep} for the exact layout). The agent, placed at random initially, needs to navigate to a fixed goal location, where it is granted a positive reward. For each step taken in the environment, the agent incurs a small negative reward. We consider two levels for the task. In level 1 (the original version of the game), the initial state distribution covers the entire state space, which corresponds to \ass{2} being verified. To make exploration harder, we also consider an initial state distribution that does not contain any state from the room where the goal is located, and call that task level 2. We compare \jh{} to DDQN and Soft-Actor Critic~\cite[SAC]{haarnoja2018soft}. For fair comparison, we also train these agents on the environment rewards augmented with the RND rewards used to train Mr Hyde, represented by the curves labelled DDQN+RND and SAC+RND. 

Results can be found in \fig{} \ref{fig:four-rooms}. On level 1, we see that \jh{} reaches quite fast the maximal score of $90$, The baselines do not perform as well and fail to converge to the optimal policy. Interestingly, even with \ass{2} verified, adding an RND bonus led to increased performance. We also notice that overall SAC outperforms DDQN. The same observations can be made on level 2. We note that on some seeds, \jh{} takes more time to learn the optimal policy, leading to visible plateaus and temporary high variance. However, a powerful property underlined by our experiments and offered by the decoupling of exploration from exploitation is the ability of \jh{}, unhindered by dithering and/or conflicting reward signals, to converge to optimality. Comparatively, the baselines get stuck in a mixed exploration/exploitation behavior and fail to reach the maximum score.

\begin{figure*}[t]
	\centering
	\subfloat[Level 1]{
		\includegraphics[trim = 5pt 5pt 5pt 5pt, clip, width=0.45\columnwidth]{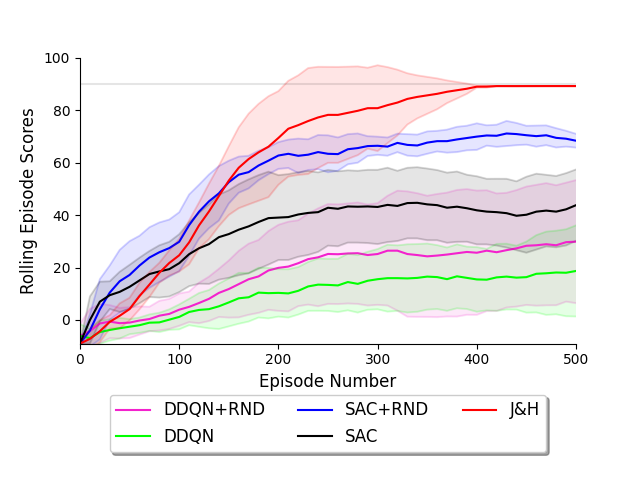}
		\label{fig:four-rooms-base}
	}
	\subfloat[Level 2]{
		\includegraphics[trim = 5pt 5pt 5pt 5pt, clip, width=0.45\columnwidth]{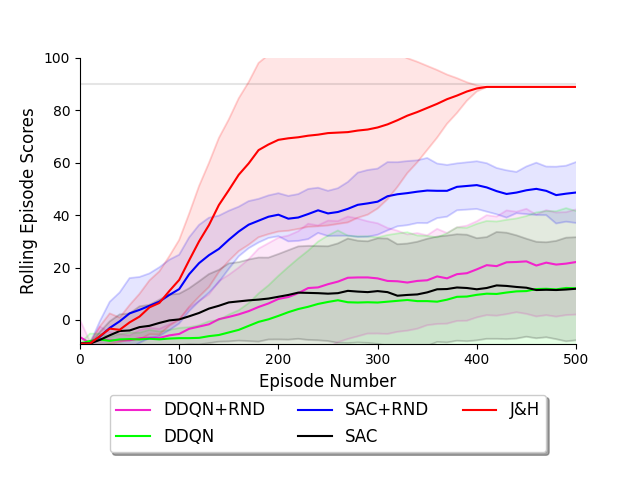}
		\label{fig:four-rooms-sto}
	}
	\caption{Deep RL experiments: score vs. number of training episodes, averaged over 10 seeds.}
		\label{fig:four-rooms}
	\vspace{-10pt}
\end{figure*}

\section{Contributions and limitations}
\label{sec:conclusion}
\myuline{Contributions:} We study a planning issue in actor-critic algorithms and tackle it by extending the policy gradient theory to \textit{policy updates} with respect to any state density. Under these generalized policy updates, we show convergence to optimality under a necessary and sufficient condition on the updates' state densities. We also significantly improve previous asymptotic convergence rates. We implement the principles prescribed by our theory in a novel algorithm, Dr Jekyll \& Mr Hyde (\jh{}), with a double personality: Dr Jekyll purely exploits while Mr Hyde purely explores. \jh{}'s independent policies allow to record two separate replay buffers: one on-policy (Dr Jekyll's) and one off-policy (Mr Hyde's), and therefore to update \jh{}'s models with a mixture of on-policy and off-policy updates. Beyond \jh{}, we define conditions for actor-critic algorithms to satisfy the requirements from our analysis. We extensively test \jh{} on finite MDPs and deep RL where it demonstrates superior planning abilities, and at least comparable asymptotic rates, to its competitors.

\myuline{Limitations and future work:} On the theory side, our non-asymptotic convergence rates are limited to the softmax and require \ass{2} and \con{3}. Although we are convinced our arguments hold, we did not develop the formal convergence proof in the RL setting (see second paragraph of \sect{} \ref{sec:conditions}). Finally, our theory does not tackle function approximation. In that context, we believe techniques from~\cite{agarwal2019optimality} could be useful. On the empirical side, we wish to emphasize that our deep RL experiments are to be considered as a proof of concept of the relevance of \jh{}'s principles in that setting. A thorough study, encompassing more environments and baselines, potentially both in discrete and continuous actions spaces, is left for future work.

\ack{}
We would like to thank Shangtong Zhang for our fruitful discussions and Alessandro Sordoni for helping us presenting and organizing the paper.

\bibliographystyle{plain}
\bibliography{biblioReport}


\newpage
\appendix
\section{MDP notations and basic knowledge on policy gradients}
\label{app:notations}
\subsection{General notations}
\begin{itemize}
    \item $\Delta_\mathcal{X}$ is the simplex over set $\mathcal{X}$.
    \item $|\mathcal{X}|$ is the cardinality of set $\mathcal{X}$.
    \item $X\sim d$ means that random variable X is sampled from distribution $d$.
    \item $\mathbb{E}[f(X)|X\sim d]$ denotes the expectation of $f(X)$ when random variable $X$ is sampled from distribution $d$.
    \item $\mathbb{P}[X>x|Y=y]$ denotes the probability that random variable $X$ is greater than $x$ conditioned to the fact that random variable $Y$ equals $y$.
    \item Subscript $\top$ (resp. $\bot$) denotes the maximal (resp. minimal) value in a range. For instance, $r_{\mytop}$ denotes the maximal reward.
    \item $f(\mathcal{X})$ denotes the sum of $f(x)$ over set $\mathcal{X}$: $f(\mathcal{X})\doteq \sum_{x\in \mathcal{X}} f(x)$.
\end{itemize}

\subsection{Markov Decision Processes}
A Markov Decision Process (MDP) is a tuple $m=\langle\mathcal{S},\mathcal{A}, p, p_0, r, \gamma\rangle$, where:
\begin{itemize}
    \item $\mathcal{S}$ is the set of states,
    \item $\mathcal{A}$ is the set of actions,
    \item $p(s'|s,a):\mathcal{S}\times\mathcal{A}\to \Delta_{\mathcal{S}}$ is the probability of accessing state $s'$ after performing action $a$ in state $s$,
    \item $p_0\in\Delta_{\mathcal{S}}$ is the initial state probability,
    \item $r(s,a):\mathcal{S}\times\mathcal{A}\to \mathbb{R}$ is the (possibly stochastic) reward function,
    \item and $\gamma$ is the discount factor.
\end{itemize}

A stochastic policy $\pi(a|s): \mathcal{S}\to \Delta_{\mathcal{A}}$ determines the probability of performing action $a$ in state $s$ and therefore determines the behaviour of an agent. 
The objective of a planning algorithm is to maximize the expected return:
\begin{align}
    \mathcal{J}(\pi) &\doteq \mathbb{E}\left[\sum_{t=0}^\infty \gamma^t R_t\bigg| S_0\sim p_0, A_t\sim \pi(\cdot|S_t), S_{t+1}\sim p(\cdot|S_t,A_t), R_t\sim r(\cdot|S_t,A_t)\right].
\end{align}

\subsection{Values and advantages}
More generally, the state (resp. state-action) value function $v_\pi:\mathcal{S}\to\mathbb{R}$ (resp. $q_\pi:\mathcal{S}\times\mathcal{A}\to\mathbb{R}$) is the expected return starting from state $s$ (resp. after performing action $a$ in state $s$), when following policy $\pi$ afterwards:
\begin{align}
    v_\pi(s) &\doteq \mathbb{E}\left[\sum_{t=0}^\infty \gamma^t R_t\bigg| S_0=s, A_t\sim \pi(\cdot|S_t), S_{t+1}\sim p(\cdot|S_t,A_t), R_t\sim r(\cdot|S_t,A_t)\right],\\
    q_\pi(s,a) &\doteq \mathbb{E}\left[\sum_{t=0}^\infty \gamma^t R_t\bigg| S_0=s,A_0=a, A_t\sim \pi(\cdot|S_t), S_{t+1}\sim p(\cdot|S_t,A_t), R_t\sim r(\cdot|S_t,A_t)\right].
\end{align}

We have several notorious identities, for instance: 
\begin{align}
    \mathcal{J}(\pi) &= \mathbb{E}\left[v_\pi(S)| S_0\sim p_0\right],\\
    v_\pi(s) &= \mathbb{E}\left[q_\pi(s,A)| A\sim \pi(\cdot|s)\right],\\
    q_\pi(s,a) &= \mathbb{E}\left[R + \gamma v_\pi(S')| S'\sim p(\cdot|s,a), R\sim r(\cdot|s,a)\right].
\end{align}

We also define the advantage of an action $a$ (resp. a policy $\pi$) over a value $v$ in a given state $s$: 
\begin{align}
    \textnormal{adv}_{v}(s,a) &\doteq \mathbb{E}\left[R + \gamma v(S')| S'\sim p(\cdot|s,a), R\sim r(\cdot|s,a)\right] - v(s) \\
    \text{resp.}\quad \textnormal{adv}_{v}(s,\pi) &\doteq \mathbb{E}\left[\textnormal{adv}_{v}(s,A)|A\sim \pi(\cdot|s)\right].
\end{align}

By extension, we talk about advantage over a stochastic policy and write $\textnormal{adv}_{\pi}$ to denote the advantage over the value induced by this stochastic policy $\textnormal{adv}_{\pi}\doteq \textnormal{adv}_{v_\pi}$. For notational simplicity, we write $v_t$, $q_t$, $\textnormal{adv}_{t}$ instead of $v_{\pi_t}$, $q_{\pi_t}$ and $\textnormal{adv}_{\pi_t}$. A policy $\pi'$ is said to be advantageous over a value or another policy if its advantage is positive in all states $s\in\mathcal{S}$. During our analysis, we will constantly use the policy improvement theorem:
\begin{theorem}
    \label{th:improvement}
    (Policy Improvement Theorem, see, e.g., Eq 4.7 \& 4.8 in \cite{sutton2018reinforcement}) \\
    If $\forall s \in \mathcal{S}$, $\pi'$ is advantageous over $\pi$:
    \begin{align}
        \textnormal{adv}_{\pi}(s, \pi') \geq 0, \text{then } \forall s \in \mathcal{S}, v_{\pi'}(s) \geq v_\pi(s). \label{eq:improvment_value}
    \end{align}
    If the first inequality is strict in a given state $s$, so is the second inequality on $s$.
\end{theorem}

\subsection{Densities and policy gradients}
For a given stochastic policy $\pi$, we define its state visit density as:
\begin{align}
    d_{\pi,\gamma}(s) \doteq \sum_{t=0}^\infty \gamma^t \mathbb{P}\left[S_t = s | S_0\sim p_0, A_t\sim \pi(\cdot|S_t), S_{t+1}\sim p(\cdot|S_t,A_t)\right]
\end{align}

The canonical policy gradient is:
\begin{align}
\textstyle{\nabla_\theta \mathcal{J}(\pi) \doteq \sum_{s } d_{\pi,\gamma}(s) \sum_{a } q_\pi(s, a) \nabla_\theta \pi(a|s)}.
\end{align}
The update vector often used in practice is:
\begin{align}\label{eq:undiscounted_actor}
\textstyle{\sum_{s } d_{\pi,1}(s) \sum_{a } q_\pi(s, a) \nabla_\theta \pi(a|s)}.
\end{align}
We consider a more general update vector
\begin{align*}
\textstyle{U(\theta,d) \doteq \sum_{s} d(s) \sum_{a } q_\pi(s, a) \nabla_\theta \pi(a|s)},
\end{align*}
where $d:\mathcal{S} \to \mathbb{R}^+$ may be any positive function over $\mathcal{S}$.

The policy gradient methods depend on a parametrization $\theta \in \R^{|\mathcal{S}| \times |\mathcal{A}|}$ of the policy. In our finite MDPs theory, we focus on the two most common ones:
\begin{itemize}
    \item \emph{direct}: $\pi(a|s) \doteq \theta_{s,a}$ with update: $u_{s,a} = d(s) q_\pi(s, a)$ and projection on the simplex.
    \item \emph{softmax}: $\pi(a|s) \doteq \cfrac{\exp(\theta_{s, a})}{\sum_{a'} \exp(\theta_{s, a'})}$ with update: $u_{s,a} = d(s) \pi(a|s)\textnormal{adv}_\pi(s, a)$.
\end{itemize}  

\subsection{Limits of on-policy updates with and without policy entropy regularization}
\label{app:policy_regularization}
The problem with on-policy updates (updates that are issued from a density over the state visitation of the current policy) is the conflicting objective of exploring and converging. If we want to converge, then mechanically the updates on off-density states will dry out. Let us assume that, in order to discover a better policy, the agent needs to go off-policy from state $s$, and that the state $s'$ where there is a positive advantage that would lead to a policy flip is actually encountered $h$ timesteps after going off-policy. If we assume that exploration is performed through dithering (typically policy entropy regularization in actor-critic algorithms), ie with a small exploration probability $\epsilon_t$ to take the action leading to $s'$, then state $s'$  will be visited only with probability $\epsilon_t^h$. As a consequence, to ensure that $\sum_t d_{\pi_t}(s)=\infty$, as we show to be necessary in \thm{} \ref{thm:optimality}, $\epsilon_t$ would need to decrease in $\Omega(t^{-1/h})$ which is very slow (and $h$ would be a hyperparameter strongly dependent on the environment). The overall issue comes from the myopia of exploration, a problem already observed in other works that promote deep exploration such as \cite{Osband2019}. 

Only a handful of works have investigated deep exploration in actor-critic methods \cite{Burda2019,Ciosek2019}. We believe that they face difficulties due to the inertia of the policy network. Indeed, the policy gradient requires on-policy updates which implies that the policy network must constantly adapt when switching from an exploration strategy to an exploitation strategy and back, which can be time-consuming and inefficient. Our fix consists in splitting the exploration and exploitation behaviour into two policies: Mr Hyde and Dr Jekyll. Mr Hyde performs pure deep exploration and will eventually collect, for each state-action, transitions in the replay buffer in a sufficient amount to obtain statistical significance. Dr Jekyll purely exploits what it learnt from a mixture of on-policy and off-policy policy updates, the latter granting a faster adaptation to the off-policy discoveries made during exploration.

\subsection{Comparison of bounds established in \thm{} \ref{thm:softmax_rates} with previous bounds}
\label{app:softmax-bounds}
\cite{mei2020global} established (to the best of their and our knowledge) the first convergence-rate result for softmax policy gradient for MDPs. As far as we know, there has not been others since. We start by recalling the differences between our setting and theirs:
\begin{itemize}
    \item They have an all-time upper bound of the sub-optimality $\mathcal{J}(\pi_\star) - \mathcal{J}(\pi_t)\leq \epsilon$, while we only provide an asymptotic result of the form: $\exists t_0,$ such that $\forall t>t_0$, $\mathcal{J}(\pi_\star) - \mathcal{J}(\pi_t)\leq \epsilon$.
    \item Our result is more general in the sense that it is applicable to any policy update scheme satisfying \con{4}, while theirs only apply to policy gradient in settings (\ass{2} and \con{3}) that guarantee \con{4} and \con{5}.
    \item Our theorem assumes \ass{8} (see Rem. \ref{rem:unique}), requiring that there exists a unique optimal policy (\textit{i.e.} $\nexists \;(s,a_1\neq a_2)\in\mathcal{S}\times\mathcal{A}^2,$ such that $v_\star(s)=q_\star(s,a_1)=q_\star(s,a_2)$). Although \cite{mei2020global} states that their approach works with multiple optimal policies, they did not prove it formally and the extension of their proofs from a unique optimal policy to multiple ones is not clear to us. It seems that they would face the same kind of difficulties we have (basically a not well-defined action gap with moving $q_t$). We discuss the limitations with respect to \ass{8} in more details in \rem{} \ref{rem:unique} of \app{} \ref{app:rates}.
\end{itemize}

Now we recall their theorem under our notations:
\begin{theorem}[\thm{} 4 of \cite{mei2020global}]
    Assuming \ass{1}, \ass{2}, \con{3}, and \ass{8} hold, then for all $t\geq 1$:
    \begin{align}
        v_\star(s)-v_t(s) \leq \frac{16|\mathcal{S}|(r_{\mytop} - r_{\mybot})}{t(1-\gamma)^6\inf_{s\in\mathcal{S},t\geq 1} \pi_t(\pi_\star(s)|s)^2}\cdot \Big\lVert \frac{(1-\gamma)d_{\pi_\star,\gamma}}{p_0}\Big\rVert^2_\infty \cdot \Big\lVert \frac{1}{p_0}\Big\rVert_\infty ,
    \end{align}
    where $\pi_\star: \mathcal{S}\to\mathcal{A}$ is an arbitrary deterministic optimal policy.
\end{theorem}

And we recall our result under \ass{1}, \con{5}, and \ass{8}:
\begin{align}
    \exists t_0, \text{ such that } \forall t\geq t_0, \quad v_\star(s) - v_t(s) \leq \frac{8|\mathcal{A}| (v_{\mytop}-v_{\mybot})}{(t-t_0)(1-\gamma) \min_{s\in\textnormal{supp}(d_\star)} \delta(s) e_{\mybot}(s)}.
\end{align}
\begin{enumerate} 
    \item[] \con{5}. $\forall s\in\textnormal{supp}(d_\star), \lim_{t\to\infty}\frac{1}{t}\sum_{t'=0}^t \eta_{t'}d_{t'}(s) \doteq e_{\mybot}(s) > 0$,
\end{enumerate}

Now we look at each classic dependency:
\begin{itemize}
    \item Time $t$: both are in $\mathcal{O}(t^{-1})$ but ours uses an offset $t_0$, hence its qualification of asymptotic convergence rate.
    \item State set size $|\mathcal{S}|$: we have a hidden dependency in $\min_{s\in\textnormal{supp}(d_\star)} \delta(s) e_{\mybot}(s)$ which is in $\Omega(1/|\mathcal{S}|)$, so our bound is in $\mathcal{O}(|\mathcal{S}|)$. Their theorem has an additional state set size dependency in $\lVert \frac{1}{p_0}\rVert_\infty$, so their bound is in $\mathcal{O}(|\mathcal{S}|^2)$.
    \item Action set size $|\mathcal{A}|$: we are in $\mathcal{O}(|\mathcal{A}|)$. Their theorem has an implicit action set size dependency in $\inf_{s\in\mathcal{S},t\geq 1} \pi_t(\pi_\star(s)|s)$, so their bound is in $\mathcal{O}(|\mathcal{A}|^2)$.
    \item Horizon/discount factor $\frac{1}{1-\gamma}$: we have a possible hidden dependency in $v_{\mytop}-v_{\mybot}$ (depending on e.g. the sparsity of rewards), which yields a bound in $\mathcal{O}((1-\gamma)^{-2})$, to be compared with their dependency in $\mathcal{O}((1-\gamma)^{-6})$. They pay the cost for setting through \con{3} the learning rate $\eta$ in $\mathcal{O}((1-\gamma)^{-3})$. We interpret the last $1-\gamma$ difference by the fact that we only deal with asymptotic rates and do not have to account for propagation through the MDP. Note that while the horizon appears in $\lVert \frac{(1-\gamma)d_{\pi_\star,\gamma}}{p_0}\rVert^2_\infty$, this term is larger than 1 (it is the $\sup$ of a ratio of distributions), hence the $1-\gamma$ cannot be taken out.
    \item State initialization distribution $p_0$: there lies the beauty of policy updates, we do not have any dependency on it. However, we do show that policy gradient will not converge to the optimal solution in general if \ass{2} is not satisfied, and in that case off-policy updates must be used. The dependency in $p_0$ is particularly strong: possibly cubic with $\lVert \frac{1}{p_0}\rVert_\infty$. The worst case happens in $\lVert \frac{(1-\gamma)d_{\pi_\star,\gamma}}{p_0}\rVert^2_\infty$ when a state is rarely seen at initialization but visited a lot by optimal policies .
\end{itemize}

Finally, we discuss the constants that are difficult to compare:
\begin{itemize}
    \item Divergence from the optimal policy during training $\inf_{s\in\mathcal{S},t\geq 1} \pi_t(\pi_\star(s)|s)$: this is a uncontrolled term that shows that policy gradient is very sensitive to preliminary convergence to sub-optimal solution. This term can easily take infinitesimal values (see \textit{e.g.} \fig{} 1d of~\cite{mei2020global}), and the dependency is squared.
    \item Optimal action gap $\delta(s)$: we depend on this constant because we first prove the convergence speed in policy and only then in value. If the action gap is very small in some states, resulting in a loose bound, it may be useful to use different proof techniques to avoid the convergence in policy step.
    \item State density of policy updates $e_{\mybot}(s)$: with a constant factor $\kappa$ on uniform policy updates, $e_{\mybot}(s)\doteq \frac{\kappa}{|\mathcal{S}|}$, and we accounted for this hidden dependency in the state set size comparison above.
\end{itemize}

\newpage
\section{Theoretical results}
\label{app:theory}
\subsection{Assumptions and conditions}
\begin{enumerate} 
    \item[] \ass{1} The model $p_0,p,r$ is known,
    \item[] \ass{2} The initial state-distribution covers the full state space: $\forall s\in\mathcal{S}, p_0(s)>0$,
    \item[] \con{3} The learning rate is constant: $\eta_t = \frac{(1 - \gamma)^3}{2 \gamma |\mathcal{A}|}$ (direct) and $\eta_t = \frac{(1 - \gamma)^3}{8}$ (softmax),
    \item[] \con{4}. Each state $s$ is updated with weights that sum to infinity over time: $\sum_{t=0}^\infty \eta_t d_t(s)=\infty$,
    \item[] \con{4-s}. $\forall s\in\mathcal{S}, \quad \lim_{T\to\infty}\frac{1}{T}\sum_{t=0}^T \eta_t d_t(s)\geq d_{\mybot}(s)>0.$
    \item[] \con{5}. $\forall s\in\textnormal{supp}(d_\star), \lim_{t\to\infty}\frac{1}{t}\sum_{t'=0}^t \eta_{t'}d_{t'}(s) \doteq e_{\mybot}(s) > 0$,
    \item[] \con{6}. Each state-action pair is explored infinitely many times: $\forall s,a,\,\lim_{t\to\infty} n_t(s,a) = \infty.$
    \item[] \con{7}. There exists $d_{\mytop} \geq d_{\mybot} > 0$ such that $\forall s,t, d_{\mytop} \geq d_t(s) \geq d_{\mybot}$.
    \item[] \ass{8}. The optimal policy is unique: $\forall s, \; q_\star(s,a_1)=q_\star(s,a_2)=v_\star(s)$ implies $a_1 = a_2$.
\end{enumerate}

\subsection{Proofs}
\subsubsection{Monotonicity and convergence}
\monotonicity*
\begin{proof}[Proof sketch]
    First, we prove that, with both direct and softmax parametrizations, the update defined in \eq{} \eqref{eq:policyupdate} leads to a policy that is advantageous over the original one. By the policy improvement theorem, it means that the value function $v_t$ must be increasing, which also implies that $q_t$ is increasing.
\end{proof}
\begin{proof}
    The direct and softmax parametrizations are treated independently in Theorems \ref{thm:monotonicity-direct} and \ref{thm:monotonicity-softmax}.
\end{proof}

\begin{thma}{thm:monotonicity}[\textbf{Monotonicity under the direct parametrization}]
    Under \ass{1}, the sequence of value functions $q_t \doteq q_{\pi_t}$ and $v_t \doteq v_{\pi_t}$ are monotonously increasing.
\label{thm:monotonicity-direct}
\end{thma}
\begin{proof}
    Let us fix any $s \in \mathcal{S}$. We recall that
    \begin{align}
        \pi_{t+1}(\cdot |s) = \mathcal{P}_{\Delta(\mathcal{A})}(\pi_t(\cdot|s) + \eta d(s) q_t(s, \cdot)).
    \end{align}
    By the definition of projections, this gives us:
    \begin{align}
        &&\rVert \pi_t(\cdot|s) + \eta d(s) q_t(s, \cdot) - \pi_{t+1}(\cdot |s) \lVert^2 \quad&\leq\quad \rVert \pi_t(\cdot|s) + \eta d(s) q_t(s, \cdot) - \pi_t(\cdot|s) \lVert^2 \\
        \iff &&2 \langle \pi_t(\cdot|s) - \pi_{t+1}(\cdot |s), \eta d(s) q_t(s, \cdot) \rangle \;\;&+\;\; \rVert \pi_t(\cdot|s) - \pi_{t+1}(\cdot |s) \lVert^2 \quad\leq\quad 0 \\
        \label{eq:product}
        \iff &&\frac{1}{2} \rVert \pi_t(\cdot|s) - \pi_{t+1}(\cdot |s) \lVert^2 \quad&\leq\quad \langle \pi_{t+1}(\cdot |s) - \pi_t(\cdot|s), \eta d(s) q_t(s, \cdot) \rangle \\
        \implies &&0 \quad&\leq\quad \langle \pi_{t+1}(\cdot |s) - \pi_t(\cdot|s), \eta d(s) q_t(s, \cdot) \rangle.
    \end{align}
    Using the fact that $d(s) > 0$ and $\eta > 0$ and replacing the vectors by their values allows to infer:
    \begin{align}
        \textnormal{adv}_t(s, \pi_{t+1}) \doteq \sum_{a\in\mathcal{A}}(\pi_{t+1}(a|s)-\pi_{t}(a|s))q_t(s, a) &\geq 0,
    \end{align}
    which is true for all $s$ and therefore allows to apply the policy improvement theorem, which concludes the proof of the monotonicity of the value through the direct parametrization policy update.
\end{proof}

\begin{thmb}{thm:monotonicity}[\textbf{Monotonicity under the softmax parametrization}]
    Under \ass{1}, the sequence of value functions $q_t \doteq q_{\pi_t}$ and $v_t \doteq v_{\pi_t}$ are monotonously increasing.
    \label{thm:monotonicity-softmax}
\end{thmb}
\begin{proof}
    This lemma is a generalization of Lemma C.2 in~\cite{agarwal2019optimality}, relaxing the assumption on the learning rate, as well as extending the result to our update rule. By the performance difference lemma (see e.g. Lemma 3.2 in~\cite{agarwal2019optimality}), it is sufficient to show that:
    \begin{equation}
        \forall s\in\mathcal{S}, \quad\textnormal{adv}_t(s, \pi_{t+1}) \doteq \sum_{a}\pi_{t+1}(a|s) \text{adv}_t(s, a) \geq 0.
    \end{equation}
    Let us define:
    \begin{align}
        \mathcal{A}^{\mytop}_s &= \left\{a\in\mathcal{A}|\text{adv}_t(s, a) \geq 0 \right\}, \\
        \mathcal{A}^{\mybot}_s &= \left\{a\in\mathcal{A}|\text{adv}_t(s, a) < 0\right\}.
    \end{align}
    For any $s$, we have (the weights in the exponential have all increased, and the advantages are non-negative):
    \begin{align}
        \sum_{a \in \mathcal{A}^{\mytop}_s } e^{(\theta_{t})_{s,a} + \eta_t d_t(s) \pi_t(a|s) \text{adv}_t(s, a)} \text{adv}_t(s, a) &\geq \sum_{a \in \mathcal{A}^{\mytop}_s} e^{(\theta_{t})_{s,a}} \text{adv}_t(s, a).
    \end{align}
    Similarly  (the weights in the exponential have all decreased, and the advantages are negative):
    \begin{align}
        \sum_{a \in \mathcal{A}^{\mybot}_s } e^{(\theta_{t})_{s,a} + \eta_t d_t(s) \pi_t(a|s) \text{adv}_t(s, a)} \text{adv}_t(s, a) &\geq \sum_{a \in \mathcal{A}^{\mybot}_s} e^{(\theta_{t})_{s,a}} \text{adv}_t(s, a).
    \end{align}
    Summing gives:
    \begin{align}
        \sum_{a} e^{(\theta_{t+1})_{s,a}} \text{adv}_t(s, a) \geq \sum_{a} e^{(\theta_{t})_{s,a}} \text{adv}_t(s, a) = 0.
    \end{align}
    Normalizing the left-hand side by $\sum_{a} e^{(\theta_{t+1})_{s,a}}$ gives the policy and concludes the proof.
\end{proof}

\convergence*
\begin{proof}
    Since $\forall s, a, t, q_t(s,a) \leq \frac{1}{1 - \gamma} r_{\mytop}$, the monotonous convergence theorem guarantees the existence of $q_\infty \in \mathbb{R}^{|\mathcal{S}|\times|\mathcal{A}|}$ that is the limit of the sequence of $q_t$.
\end{proof}

\subsubsection{Optimality}
\optimality*
\begin{proof}[Proof sketch]
    In both direct and softmax parametrizations, we assume that there exists a state-action pair $(s,a)$ that is advantageous over the state value limits $q_\infty$ and $v_\infty \doteq \lim_{t\to\infty}v_t$: $\text{adv}_\infty(s,a) \doteq q_\infty(s,a) - v_\infty(s) > 0$. Then, we prove that in this state, the policy improvement yielded by the update is lower bounded by a linear function over the update weight $\eta_t d_t(s)$ applied in state $s$. Next, we sum over $t$ and notice that this lower bounding sum diverges to infinity, which contradicts Corollary \ref{cor:convergence}.
    
    We may therefore infer that there cannot exist a state-action pair that is advantageous over $q_\infty$ and $v_\infty$, which allows us to conclude that no policy improvement is possible, and thus that the values are optimal: $q_\infty = \max_{\pi\in\Pi}q_\pi$.
    
    For the necessity of \con{4}, we show that the parameter update is upper bounded in both parametrizations by a term linear in the action gap. By choosing the reward function adversarially, we may set it sufficiently small so that the sum of all the gradient steps are insufficient to reach optimality.
\end{proof}
\begin{proof}
    The sufficient conditions for direct and softmax parametrizations are treated independently in \thm{} \ref{thm:optimality-direct} and \ref{thm:optimality-softmax}. The necessary condition is proven in \thm{} \ref{thm:necessary}.
\end{proof}

\begin{thma}{thm:optimality}[\textbf{Optimality under the direct parametrization}]
    Under \ass{1} and \con{4}, the sequence of value functions $q_t$ converges to optimality:
    \begin{align}
        q_\infty = q_\star\doteq\max_{\pi\in\Pi}q_\pi.
    \end{align}
    \label{thm:optimality-direct}
\end{thma}
\begin{proof}
    Let us assume that $v_\infty<v_\star$, then, by the policy improvement theorem, there must be some state $s$ for which an advantage $q_\infty(s,a_{\mytop})-v_\infty(s)=\epsilon(s)>0$ over $\pi_t$ exists, with $a_{\mytop}\in\mathcal{A}_{\mytop}(s)=\argmax_{a\in\mathcal{A}}q_\infty(s,a)$. Let us define the state value-gap $\delta(s)\coloneqq q_\infty(s,a_{\mytop})-\max_{a_{\mybot}\in\mathcal{A}_{\mybot}(s)}q_\infty(s,a_{\mybot})>0$, with $\mathcal{A}_{\mybot}(s)\coloneqq\mathcal{A}/\mathcal{A}_{\mytop}(s)$.  
    
    Since we proved that $q_t \to_{t\to\infty} q_\infty$, there exists $t_0$ such that for all $t \geq t_0$ and $a\in\mathcal{A}$, $q_\infty(s,a)-q_t(s,a)\leq\frac{\delta(s)}{2}$. This guarantees two things for any $t \geq t_0$:
    \begin{align}
        &\forall a \in \mathcal{A}_{\mytop}(s), q_t(s,a) \geq q_\infty(s,a_{\mytop}) - \frac{\delta(s)}{2}, \\
        &\forall a \in \mathcal{A}_{\mybot}(s), q_t(s,a) \leq q_\infty(s,a_{\mytop}) - \delta(s).
    \end{align}
    Remembering that
    \begin{align}
        \pi_{t+1}(\cdot|s) &\doteq \mathcal{P}_{\Delta(\mathcal{A})} \left(\pi_{t}(\cdot|s) + \eta_t d_t(s)q_t(s,\cdot)\right),
    \end{align}
    the above inequalities allow us to apply Lemma~\ref{lem:proj_lower} to $\pi_{t}(\cdot|s)$ with $\mathcal{A}_{\mytop}(s)$ the set of coordinates $\{1,\dots,k\}$, $\alpha = \eta_t d_t(s) (q_\infty(s,a_{\mytop}) - \frac{\delta(s)}{2})$ and $\beta = \eta_t d_t(s) (q_\infty(s,a_{\mytop}) - \delta(s))$, giving:
    \begin{align}
        \sum_{a\in\mathcal{A}_{\mytop}(s)}\pi_{t+1}(a|s) &\geq \min\left(1, \sum_{a\in\mathcal{A}_{\mytop}(s)}\pi_{t}(a|s) + \frac{\alpha - \beta}{2}\right) \\
        &= \min\left(1, \sum_{a\in\mathcal{A}_{\mytop}(s)}\pi_{t}(a|s) + \eta_t d_t(s)\frac{\delta(s)}{4}\right). \label{eq:direct-opt}
    \end{align}

    By assumption, we know that $\sum_{a\in\mathcal{A}_{\mytop}(s)}\pi_{t}(a|s)\leq \lim_{t' \to \infty} \sum_{a\in\mathcal{A}_{\mytop}(s)}\pi_{t'}(a|s)<1$, we may conclude that $\forall t \geq t_0$:
    \begin{align}
        &\sum_{a\in\mathcal{A}_{\mytop}(s)}\pi_{t+1}(a|s)-\sum_{a\in\mathcal{A}_{\mytop}(s)}\pi_{t}(a|s) \geq \eta_t d_t(s)\frac{\delta(s)}{4}.
    \end{align}
    Hence,
    \begin{align}
        \lim_{t' \to \infty} \sum_{a\in\mathcal{A}_{\mytop}(s)}\pi_{t'}(a|s)-\sum_{a\in\mathcal{A}_{\mytop}(s)}\pi_{t_0}(a|s) \geq \sum_{t'=t_0}^\infty \eta_{t'} d_{t'}(s)\frac{\delta(s)}{4}.
    \end{align}
    
    Since $\sum_{t'=t}^\infty \eta_{t'} d_{t'}(s) =\infty$ and the left-hand side is bounded, this contradicts the assumption that there exists an advantageous action in state $s$. This property must be true in every state $s$, we may therefore conclude that $v_\infty$ is optimal.
\end{proof}
\begin{thmb}{thm:optimality}[\textbf{Optimality under the softmax parametrization}]
    Under \ass{1} and \con{4}, the sequence of value functions $q_t$ converges to optimality:
    \begin{align}
        q_\infty = q_\star\doteq\max_{\pi\in\Pi}q_\pi.
    \end{align}
    \label{thm:optimality-softmax}
\end{thmb}
\begin{proof}[Proof sketch]
    This theorem generalizes \thm{} 5.1 of \cite{agarwal2019optimality} and our proof borrows many of their ideas. The generalization could have been made by adapting some of the lemmas their proof relies on, but we believe that the full rewriting of the proof has its merits: it is self contained, and ends up being significantly shorter (~3 pages versus ~10 pages). In this proof sketch, we stack without proving them the arguments used in the full proof that follows. 
    
    We start by fixing state $s$ and by partitioning the action set according to their value $q_\infty(s,a)$ in $s$ as compared to $v_\infty(s)$: smaller $\mathcal{A}^{\mysmaller}_s$, equal $\mathcal{A}^{\myeq}_s$, or larger $\mathcal{A}^{\mybigger}_s$. We assume towards a contradiction that there must be some $s$ for which $\mathcal{A}^{\mybigger}_s$ is non-empty, as otherwise, the policy improvement theorem tells us that $v_\infty$ is optimal. In such a state, by the uniform convergence theorem, there is a timestep $t_0$ from which for all $t \geq t_0$, the values $q_t$ are $\xi$-close to $q_\infty$. We choose $\xi$ to be smaller than any action gap in state $s$. We then notice that if an action in $\mathcal{A}^{\myeq}_s$ gets a policy smaller than that of an action in $\mathcal{A}^{\mybigger}_s$, then it remains this way in the future (it undergoes smaller updates). We further partition $\mathcal{A}^{\myeq}_s$: $\mathcal{A}^{\myeqsm}_s$ contains the actions for which the policy gets smaller than the policy in all actions of $\mathcal{A}^{\mybigger}_s$, $\mathcal{A}^{\myeqbg}_s$ contains the actions for which this event never happens. We call $t_1\geq t_0$, the timestep when $\forall a^{\myeqsm}\in\mathcal{A}^{\myeqsm}_s, a^{\mybigger}\in\mathcal{A}^{\mybigger}_s, \pi_{t_1}(a^{\myeqsm}|s) \leq \pi_{t_1}(a^{\mybigger}|s)$.
    
    Then, we observe that $\mathcal{A}^{\myeqbg}_s$ cannot be empty because this set is the only one that contains actions that may prevent full convergence on $\mathcal{A}^{\mybigger}_s$. We then look at the evolution of the ratio of policy mass on $\mathcal{A}^{\mybigger}_s$ versus $\mathcal{A}^{\mysmaller}_s\cup\mathcal{A}^{\myeqsm}_s$ and show that this ratio is monotonously increasing. We then establish an upper bound on it by arguing that otherwise it would imply that the sum of the updates on $\mathcal{A}^{\myeqbg}_s$ would become negative. This cannot happen because actions in $\mathcal{A}^{\myeqbg}_s$ need to diverge to infinity to get some policy mass, and none of them can compensate by diverging to $-\infty$, as it would otherwise have been partitioned in $\mathcal{A}^{\myeqsm}_s$. We conclude the proof by showing that if the parameter of an action in $\mathcal{A}^{\myeqbg}_s$ diverges to $+\infty$, then the parameters of actions in $\mathcal{A}^{\mybigger}_s$ will also diverge to $+\infty$. This implies the aforementioned ratio diverges to $\infty$ contradicting it being upper bounded. Hence $\mathcal{A}^{\mybigger}_s$ must be empty in every state and consequently $v_\infty$ optimal.
\end{proof}
\begin{proof}[Full proof]
    We recall the update rule with softmax parametrization:
    \begin{align}
        (\theta_{t+1})_{s,a} &= (\theta_{t})_{s,a} + \eta_t d_t(s) \pi_t(a|s) \text{adv}_t(s, a).
    \end{align}
    
    We start from Corollary \ref{cor:convergence} that proves the existence of $v_\infty(s)$ and $q_\infty(s,a)$ to which $v_t(s)$ and $q_t(s,a)$ converge from below for any state and action. We then consider the following partition of the actions in state $s$:
    \begin{align}
        \mathcal{A}^{\mybigger}_s &= \left\{a\in\mathcal{A}|q_\infty(s,a)>v_\infty(s)\right\} \\
        \mathcal{A}^{\myeq}_s &= \left\{a\in\mathcal{A}|q_\infty(s,a) = v_\infty(s)\right\} \\
        \mathcal{A}^{\mysmaller}_s &= \left\{a\in\mathcal{A}|q_\infty(s,a) < v_\infty(s)\right\}.
    \end{align}
    If $\forall s, \;\mathcal{A}^{\mybigger}_s = \emptyset$, the policy improvement theorem tells us that $q_{\infty}$ is optimal and our work is done.
    
    We thus assume towards a contradiction that $\mathcal{A}^{\mybigger}_s$ is non-empty for a certain state $s$. We define $\delta_{\infty}(s)$ as the smallest positive improvement of $q_\infty$ over $v_\infty$ in state $s$:
    \begin{align}
        \delta_\infty(s) \coloneqq \min_{a\in\mathcal{A}^{\mybigger}_s} q_\infty(s,a)-v_\infty(s).
    \end{align}
    From the convergence of $v_t(s)$ and $q_t(s,a)$, we know that for any $\xi > 0$, there exists $t_0$ such that for any $s,a$ and $t \geq t_0$:
    \begin{align}
        v_\infty(s) - \xi &\leq v_t(s) \leq v_\infty(s), \label{eq:xi1} \\
        q_\infty(s,a) - \xi &\leq q_t(s,a) \leq q_\infty(s,a). \label{eq:xi2}
    \end{align}
   
    Moving forward, we choose $\xi \leq \frac{1}{2}\min\{\delta_\infty(s), v_\infty(s)-\max_{a\in\mathcal{A}^{\mysmaller}_s} q_\infty(s,a)\}$ and let $t_0$ be the corresponding timestep. 
    
    This allows us to further partition the set $\mathcal{A}^{\myeq}_s$ as:
    \begin{align}
        \mathcal{A}^{\myeqsm}_s &\doteq \left\{a\in\mathcal{A}^{\myeq}_s|\exists t_a \geq t_0, \forall t' \geq t_a, \forall a^{\mybigger} \in \mathcal{A}^{\mybigger}_s, \theta_{t'}(s,a^{\mybigger}) \geq \theta_{t'}(s,a) \right\} \\
        \mathcal{A}^{\myeqbg}_s &\doteq \mathcal{A}^{\myeq}_s \setminus \mathcal{A}^{\myeqsm}_s.
    \end{align}
    We note that trivially: $\mathcal{A} = \mathcal{A}^{\myeqsm}_s \cup \mathcal{A}^{\myeqbg}_s \cup \mathcal{A}^{\mybigger}_s \cup \mathcal{A}^{\mysmaller}_s$. We let $t_1 = \max_{a \in \mathcal{A}^{\myeqsm}_s} t_a$ and in the following consider $t \geq t_1$. We recall the notation $f(\mathcal{X})\doteq \sum_{x\in\mathcal{X}} f(x)$ for any function $f$ and set $\mathcal{X}$.
    
    The updates $u_t(s,a) \doteq \pi_t(a|s)\text{adv}_t(s, a)$ on those various sets can be bounded as follows.
    
    On $\mathcal{A}^{\mybigger}_s$:
    \begin{align}
        u_t(s, \mathcal{A}^{\mybigger}_s) &\doteq \sum_{a^{\mybigger}\in\mathcal{A}^{\mybigger}_s} u_t(s, a^{\mybigger}) \\
        &= \sum_{a^{\mybigger}\in\mathcal{A}^{\mybigger}_s} \pi_t(a^{\mybigger}|s)(q_t(s, a^{\mybigger}) - v_t(s)) \\
        &\geq ((v_{\infty}(s)+\delta_\infty(s) - \xi) - v_{\infty}(s))\sum_{a^{\mybigger}\in\mathcal{A}^{\mybigger}_s}\pi_t(a^{\mybigger}|s) \\
        &= (\delta_\infty(s) - \xi)\pi_t(\mathcal{A}^{\mybigger}_s|s). \label{eq:adv_bigger}
    \end{align}
    
    On $\mathcal{A}^{\myeqbg}_s$:
    \begin{align}
        u_t(s, \mathcal{A}^{\myeqbg}_s) &\doteq \sum_{a^{\myeqbg}\in\mathcal{A}^{\myeqbg}_s} u_t(s, a^{\myeqbg}) \\
        &= \sum_{a^{\myeqbg}\in\mathcal{A}^{\myeqbg}_s} \pi_t(a^{\myeqbg}|s)\left(\sum_{b^{\mysmaller}\in\mathcal{A}^{\mysmaller}_s\cup\mathcal{A}^{\myeqsm}_s} \pi_{t}(b^{\mysmaller}|s) (q_t(s,a^{\myeqbg}) - q_t(s,b^{\mysmaller}))\right.\nonumber \\
        &\quad\quad\quad\quad\quad\quad\quad\quad\quad + \sum_{b^{\mybigger}\in\mathcal{A}^{\mybigger}_s} \pi_{t}(b^{\mybigger}|s) (q_t(s,a^{\myeqbg}) - q_t(s,b^{\mybigger})) \\
        &\quad\quad\quad\quad\quad\quad\quad\quad\quad \left.+ \sum_{b^{\myeqbg}\in\mathcal{A}^{\myeqbg}_s} \pi_{t}(b^{\myeqbg}|s) (q_t(s,a^{\myeqbg}) - q_t(s,b^{\myeqbg}))\right) \nonumber\\
        &\leq \sum_{a^{\myeqbg}\in\mathcal{A}^{\myeqbg}_s} \pi_t(a^{\myeqbg}|s)\left(\pi_t(\mathcal{A}^{\mysmaller}_s|s)+\pi_t(\mathcal{A}^{\myeqsm}_s|s)\right)(v_\infty(s)-v_{\mybot}) \nonumber\\
        &\quad - \sum_{a^{\myeqbg}\in\mathcal{A}^{\myeqbg}_s} \pi_t(a^{\myeqbg}|s)\pi_t(\mathcal{A}^{\mybigger}_s|s)(\delta_\infty(s) - \xi) \\
        &\quad + \sum_{(a^{\myeqbg},b^{\myeqbg})\in\mathcal{A}^{\myeqbg}_s\times\mathcal{A}^{\myeqbg}_s} \pi_{t}(a^{\myeqbg}|s) \pi_{t}(b^{\myeqbg}|s) (q_t(s,a^{\myeqbg}) - q_t(s,b^{\myeqbg})) \\
        \vspace{0.2cm}&= \pi_t(\mathcal{A}^{\myeqbg}_s|s)\left(\left(\pi_t(\mathcal{A}^{\mysmaller}_s|s)+\pi_t(\mathcal{A}^{\myeqsm}_s|s)\right)(v_\infty(s)-v_{\mybot})-\pi_t(\mathcal{A}^{\mybigger}_s|s)(\delta_\infty(s) - \xi)\right) \label{eq:56} \\
        \vspace{0.2cm}&\leq \left(1-\pi_{t}(\mathcal{A}^{\myeqbg}_s|s)\right)(v_\infty(s)-v_{\mybot}).
    \end{align}
    Finally, on $\mathcal{A}^{\mysmaller}_s$:
    \begin{align}
        u_t(s, \mathcal{A}^{\mysmaller}_s) &\doteq \sum_{a^{\mysmaller}\in\mathcal{A}^{\mysmaller}_s} u_t(s, a^{\mysmaller}) \leq 0.
    \end{align}
    
    This gives:
    \begin{align}
        (\theta_{t'})_{s,\mathcal{A}^{\mybigger}_s} &\geq (\theta_{t})_{s,\mathcal{A}^{\mybigger}_s} + \left(\delta_\infty(s) - \xi\right)\sum_{t''=t}^{t'-1}\eta_{t''} d_{t''}(s) \pi_{t''}(\mathcal{A}^{\mybigger}_s|s),  \label{eq:softmax_plus}\\
        (\theta_{t'})_{s,\mathcal{A}^{\myeqbg}_s} &\leq (\theta_{t})_{s,\mathcal{A}^{\myeqbg}_s} + (v_\infty(s)-v_{\mybot})\sum_{t''=t}^{t'-1}\eta_{t''} d_{t''}(s) \left(1-\pi_{t''}(\mathcal{A}^{\myeqbg}_s|s)\right), \label{eq:sum_eq}\\
        (\theta_{t'})_{s,\mathcal{A}^{\mysmaller}_s} &\leq (\theta_{t})_{s,\mathcal{A}^{\mysmaller}_s}. \label{eq:softmax_minus}
    \end{align}
    
    \myuline{$\mathcal{A}^{\myeqbg}_s \neq \emptyset$.} 
    If $\mathcal{A}^{\myeqbg}_s$ is empty, \eq{} \eqref{eq:softmax_plus}, and \eq{} \eqref{eq:softmax_minus} imply that $\pi_{t}(\mathcal{A}^{\mybigger}_s|s)$ increases monotonously with $t$. \con{4} then guarantees that $(\theta_{t'})_{s,\mathcal{A}^{\mybigger}_s}$ diverges to $\infty$. Consequently, $\pi_{t}(\mathcal{A}^{\mybigger}_s|s)$ will converge to 1 and $v_t(s)$ will tend to a value larger than $v_\infty(s) + \delta_\infty(s)$ which is impossible. Moving forward, we consider $\mathcal{A}^{\myeqbg}_s \neq \emptyset$ non-empty.
    
    \myuline{$\forall a^{\mybigger}\in\mathcal{A}^{\mybigger}_s, \;a^{\myeqsm}\in\mathcal{A}^{\myeqsm}_s,\, \theta_t(s,a^{\mybigger})\geq \theta_t(s,a^{\myeqsm})$, $\pi_t(s,a^{\mybigger})\geq \pi_t(s,a^{\myeqsm})$, and $u_t(s,a^{\mybigger})\geq u_t(s,a^{\myeqsm})$.} 
    
    The first two inequalities stem directly by construction of $\mathcal{A}^{\myeqsm}_s$. For the last one, we have:
    \begin{align}
        u_t(s,a^{\mybigger}) = \pi_t(a^{\mybigger}|s)\text{adv}_t(s, a^{\mybigger}) &\geq \pi_t(s,a^{\myeqsm}) \text{adv}_t(s, a^{\mybigger}) \nonumber \\
        &= \pi_t(s,a^{\myeqsm}) (q_t(s, a^{\mybigger}) - v_t(s)) \nonumber \\
        &\geq \pi_t(s,a^{\myeqsm}) (v_\infty(s) + \delta_\infty - \xi - v_t(s)) \nonumber \\
        &\geq \pi_t(s,a^{\myeqsm}) (\delta_\infty - \xi) \nonumber \\
        &\geq \xi \pi_t(s,a^{\myeqsm})  \nonumber \\
        &\geq \pi_t(s,a^{\myeqsm}) \text{adv}_t(s,a^{\myeqsm}) = u_t(s,a^{\myeqsm}). \label{eq:once_then_over}
    \end{align}
    In the last line, we used: $\text{adv}_t(s,a^{\myeqsm}) = q_t(s, a^{\myeqsm}) - v_t(s) \leq q_\infty(s, a^{\myeqsm}) - v_t(s) = v_\infty(s) - v_t(s) \leq \xi$.
    
    \myuline{Mass on $\mathcal{A}^{\mybigger}_s$ and $\mathcal{A}^{\mysmaller}_s\cup\mathcal{A}^{\myeqsm}_s$.}
    We now look at the evolution of $\text{ratio}_t(\mathcal{A}^{\mybigger}_s,\mathcal{A}^{\mysmaller}_s\cup\mathcal{A}^{\myeqsm}_s) \doteq \frac{\pi_t(\mathcal{A}^{\mybigger}_s|s)}{\pi_t(\mathcal{A}^{\mysmaller}_s\cup\mathcal{A}^{\myeqsm}_s|s)}$:
    \begin{align}
        \text{ratio}_{t+1}(\mathcal{A}^{\mybigger}_s,\mathcal{A}^{\mysmaller}_s\cup\mathcal{A}^{\myeqsm}_s) &= \frac{\sum_{a\in\mathcal{A}^{\mybigger}_s}\exp\left((\theta_{t+1})_{s,a}\right)}{\sum_{a\in\mathcal{A}^{\mysmaller}_s\cup\mathcal{A}^{\myeqsm}_s}\exp\left((\theta_{t+1})_{s,a}\right)} \\
        &= \frac{\sum_{a\in\mathcal{A}^{\mybigger}_s}\exp\left((\theta_{t})_{s,a} +u_t(s,a)\right)}{\sum_{a\in\mathcal{A}^{\mysmaller}_s\cup\mathcal{A}^{\myeqsm}_s}\exp\left((\theta_{t})_{s,a}+u_t(s,a)\right)} \\
        &\geq \frac{\sum_{a\in\mathcal{A}^{\mybigger}_s}\exp\left((\theta_{t})_{s,a}\right)}{\sum_{a\in\mathcal{A}^{\mysmaller}_s\cup\mathcal{A}^{\myeqsm}_s}\exp\left((\theta_{t})_{s,a}\right)} \\
        &\geq \text{ratio}_{t}(\mathcal{A}^{\mybigger}_s,\mathcal{A}^{\mysmaller}_s\cup\mathcal{A}^{\myeqsm}_s) \label{eq:ratio_monotonicity} 
    \end{align}
    
    
    \eq{} \eqref{eq:ratio_monotonicity} establishes the monotonicity of the ratio and shows that $\forall t' \geq t$:
    \begin{align}
       \pi_{t'}(\mathcal{A}^{\mybigger}_s|s) &\geq \pi_{t'}(\mathcal{A}^{\mysmaller}_s\cup\mathcal{A}^{\myeqsm}_s|s)\text{ratio}_{t}(\mathcal{A}^{\mybigger}_s,\mathcal{A}^{\mysmaller}_s\cup\mathcal{A}^{\myeqsm}_s) \label{eq:ratio}.
    \end{align}
    
    \myuline{\eq{} \eqref{eq:56} remains positive.} \eq{} \eqref{eq:ratio_monotonicity} allows us to show that \eq{} \eqref{eq:56} cannot become negative. Let us assume towards a contradiction that it does become negative at a given time $t$, ie
    \[
        \left(\pi_t(\mathcal{A}^{\mysmaller}_s|s)+\pi_t(\mathcal{A}^{\myeqsm}_s|s)\right)(v_\infty(s)-v_{\mybot})-\pi_t(\mathcal{A}^{\mybigger}_s|s)(\delta_\infty(s) - \xi) \leq 0.
    \]
    The monotonicity of the above ratio then guarantees that it will remain negative for all $t'>t$. In other words, for all $t' > t$, $u_{t'}(s, \mathcal{A}^{\myeqbg}_s) \leq 0$. 
    
    If the sum of updates on actions in $\mathcal{A}^{\myeqbg}_s$ is negative, two things can happen:
    \begin{itemize}
        \item Either the actions in $\mathcal{A}^{\myeqbg}_s$ all have parameters $\theta$ that are upper bounded. This means that $\pi_t(\mathcal{A}^{\mybigger}_s|s)$ will eventually converge to 1. In this case, we can reuse the arguments used to disprove $\mathcal{A}^{\myeqbg}_s = \emptyset$, and reach the same contradiction.
        \item Or at least one action in $\mathcal{A}^{\myeqbg}_s$ has a parameter $\theta$ whose $\lim \sup$ is $+\infty$. In that case, since the sum of updates on the set $\mathcal{A}^{\myeqbg}_s$ is negative, at least one action in $\mathcal{A}^{\myeqbg}_s$ must have a $\theta$ whose $\lim \inf$ is $-\infty$. Letting $a$ denote that action, there thus exists a time $t$ at which $\forall a^{\mybigger} \in \mathcal{A}^{\mybigger}_s, \, \theta_t(s,a^{\mybigger}) > \theta_t(s,a)$. Following the reasoning from \eq{} \eqref{eq:once_then_over}, the updates to those various actions will guarantee that the property holds for all $t' \geq t$, which contradicts the construction of $\mathcal{A}^{\myeqbg}_s$. 
    \end{itemize}
    Thus, we assume from now on that \eq{} \eqref{eq:56} is positive.

    \myuline{Lower bound on $\pi_{t'}(\mathcal{A}^{\mybigger}_s|s)$}
    The positivity of \eq{}~\eqref{eq:56} gives us an upper bound on the ratio: $\text{ratio}_{t}(\mathcal{A}^{\mybigger}_s,\mathcal{A}^{\mysmaller}_s\cup\mathcal{A}^{\myeqsm}_s)\leq \frac{v_\infty(s)-v_{\mybot}}{\delta_\infty(s) - \xi}$, and equivalently the interesting inequality:
    \begin{align}
       && \pi_t(\mathcal{A}^{\mysmaller}_s\cup\mathcal{A}^{\myeqsm}_s|s)(v_\infty(s)-v_{\mybot}) &\geq \pi_t(\mathcal{A}^{\mybigger}_s|s)(\delta_\infty(s) - \xi) \\
       \iff&& \,\,  \pi_t(\mathcal{A}^{\mysmaller}_s\cup\mathcal{A}^{\myeqsm}_s|s)(v_\infty(s)-v_{\mybot}+\delta_\infty(s) - \xi) &\geq (\pi_t(\mathcal{A}^{\mybigger}_s|s)+ \pi_t(\mathcal{A}^{\mysmaller}_s|s))(\delta_\infty(s) - \xi) \\
       \iff&& \,\, \pi_t(\mathcal{A}^{\mysmaller}_s\cup\mathcal{A}^{\myeqsm}_s|s) &\geq \frac{(1-\pi_t(\mathcal{A}^{\myeqbg}_s|s))(\delta_\infty(s) - \xi)}{v_\infty(s)-v_{\mybot}+\delta_\infty(s) - \xi}.
    \end{align}

    Combining it with \eq{} \eqref{eq:ratio}, we get for all $t'>t$:
    \begin{align}
        \pi_{t'}(\mathcal{A}^{\mybigger}_s|s) &\geq \pi_{t'}(\mathcal{A}^{\mysmaller}_s\cup\mathcal{A}^{\myeqsm}_s|s)\text{ratio}_{t}(\mathcal{A}^{\mybigger}_s,\mathcal{A}^{\mysmaller}_s\cup\mathcal{A}^{\myeqsm}_s) \\
        &\geq \underbrace{\frac{\text{ratio}_{t}(\mathcal{A}^{\mybigger}_s,\mathcal{A}^{\mysmaller}_s\cup\mathcal{A}^{\myeqsm}_s)(\delta_\infty(s) - \xi)}{v_\infty(s)-v_{\mybot}+\delta_\infty(s) - \xi}}_{\kappa>0}\left(1-\pi_{t'}(\mathcal{A}^{\myeqbg}_s|s)\right).
    \end{align}
    
    
    \myuline{The policy cannot converge to $\mathcal{A}^{\myeqbg}_s$.} Let us assume that the policy converges on $\mathcal{A}^{\myeqbg}_s$:
    \begin{align}
        \lim_{t\to\infty} \pi_t(\mathcal{A}^{\myeqbg}_s|s) = 1.
    \end{align}
    From \eq{} \eqref{eq:adv_bigger}, we know that $(\theta_t)_{s,\mathcal{A}^{\mybigger}}$ increases with time (and we assumed that $\mathcal{A}^{\mybigger}_s$ is non-empty). This implies that $(\theta_t)_{s,\mathcal{A}^{\myeqbg}}\to \infty$. We may therefore infer from \eq{} \eqref{eq:sum_eq} that:
    \begin{align}
        \sum_{t=0}^{\infty}\eta_{t} d_{t}(s) \left(1-\pi_{t'}(\mathcal{A}^{\myeqbg}_s|s)\right) = \infty,
    \end{align}
    then:
    \begin{align}
        (\theta_{t'})_{s,\mathcal{A}^{\mybigger}_s} &\geq (\theta_{t})_{s,\mathcal{A}^{\mybigger}_s} + \sum_{t''=t}^{t'-1}\eta_{t''} d_{t''}(s) \pi_{t''}(\mathcal{A}^{\mybigger}_s|s) \left(\delta_\infty(s) - \xi\right) \\ 
        &\geq (\theta_{t})_{s,\mathcal{A}^{\mybigger}_s} + \kappa\left(\delta_\infty(s) - \xi\right) \sum_{t''=t}^{t'-1}\eta_{t''} d_{t''}(s)\left(1-\pi_{t''}(\mathcal{A}^{\myeqbg}_s|s)\right).
    \end{align}
    Thus: $(\theta_{t'})_{s,\mathcal{A}^{\mybigger}_s} \to \infty$,
    which implies the divergence of $\text{ratio}_t(\mathcal{A}^{\mybigger}_s,\mathcal{A}^{\mysmaller}_s\cup\mathcal{A}^{\myeqsm}_s)$, and a contradiction with the upper bound established earlier: we cannot have $\pi_t(\mathcal{A}^{\myeqbg}_s|s) \to 1$.
    
    \myuline{Conclusion.} We thus know that some policy mass must remain outside of $\mathcal{A}^{\myeqbg}_s$. Because the ratio is bounded and increasing, some mass must be assigned to both $\mathcal{A}^{\mybigger}_s$ and $\mathcal{A}^{\mysmaller}_s\cup\mathcal{A}^{\myinfty}_s$. This implies that $(\theta_{t})_{s,\mathcal{A}^{\mybigger}_s}$ must diverge to $\infty$, and the ratio with it, leading to a final contradiction. We may therefore conclude that there does not exist a state where policy improvement is possible, in other words, the policy is optimal.
\end{proof}

\begin{thmc}{thm:optimality}[\textbf{Necessity of \con{4} for optimality}]
    Assumption \con{4} is necessary to guarantee that the sequence of value functions $v_t$ converges to optimality. \label{thm:necessary}
\end{thmc}
\begin{proof}
    We consider a minimal MDP with one state: $\mathcal{S}\doteq \{s\}$ and two terminal actions: $\mathcal{A}\doteq\{a_1,a_2\}$. We assume that $\sum_t \eta_t d_t(s) = m < \infty$ and $r_1=1$. Now, we prove that we may choose $r_2<1$ such that the sequence $(v_t)$ does not converge to optimality: $\lim_{t\to\infty} v_t<v_\star=1$.
    
    \paragraph{Direct parametrization:} the update is:
    \begin{align}
        u_{s,a_1} &= d_t(s) q_\pi(s, a_1) \\
        &= d_t(s) \\
        u_{s,a_2} &= d_t(s) q_\pi(s, a_2) \\
        &= d_t(s) r_2 \\
        (\theta_{t+1})_{s,a_1} &= \min\left(1, (\theta_{t})_{s,a_1} + \eta_t d_t(s) \frac{1-r_2}{2}\right) \\
        (\theta_{t+1})_{s,a_2} &= \max\left(0, (\theta_{t})_{s,a_2} + \eta_t d_t(s) \frac{r_2-1}{2}\right).
    \end{align}
    Starting from $(\theta_0)_{s,a_1}$ (resp. $\theta_{0}(a_2|s)$), $\theta_{t}(a_1|s)$ (resp. $\theta_{t}(a_2|s)$) hits 1 (resp. 0), if and only if:
    \begin{align}
        && \sum_{t=0}^\infty \eta_t d_t(s) \frac{1-r_2}{2} &\geq 1-(\theta_{0})_{s,a_1} \\
        \iff &&  m(1-r_2) &\geq 2-2(\theta_{0})_{s,a_1} \\
        \iff &&  r_2 &\leq 1 - \frac{2}{m}\left(1-(\theta_{0})_{s,a_1}\right).
    \end{align}
    If we adversarially choose $r_2=1 - \frac{1}{m}\left(1-(\theta_{0})_{s,a_1}\right)$, then the optimality is never reached.
    
    \paragraph{Softmax parametrization:} the update is:
    \begin{align}
        u_{s,a_1,t} &= d_t(s) \pi_t(a_1|s)\textnormal{adv}_t(s, a_1) \\
        &= d_t(s) \pi_t(a_1|s)(1-\pi_t(a_1|s) - \pi_t(a_2|s)r_2) \\
        &= d_t(s) \pi_t(a_1|s)(1-\pi_t(a_1|s))(1-r_2) \\
        &\leq d_t(s) \frac{1-r_2}{4}
        u_{s,a_2,t} &= d_t(s) \pi_t(a_2|s)\textnormal{adv}_t(s, a_2) \\
        &= d_t(s) \pi_t(a_2|s)(r_2-\pi_t(a_1|s) - \pi_t(a_2|s)r_2) \\
        &= d_t(s) \pi_t(a_1|s)(1-\pi_t(a_1|s))(r_2-1) \\
        &\geq d_t(s) \frac{r_2-1}{4}.
    \end{align}  
    Therefore, we know that, starting from $(\theta_0)_{s,a_1}$, $\theta_t$ is upper bounded:
    \begin{align}
        (\theta_t)_{s,a_1} &\doteq (\theta_0)_{s,a_1} + \sum_{t'=0}^t \eta_{t'} u_{s,a_1,t} \\
        &\leq (\theta_0)_{s,a_1} + \sum_{t'=0}^t \eta_{t'} d_t(s) \frac{1-r_2}{4} \\
        &\leq (\theta_0)_{s,a_1} + m \frac{1-r_2}{4}. \\
    \end{align}
    Similarly, we know that, starting from $(\theta_0)_{s,a_2}$, $\theta_t$ is upper bounded:
    \begin{align}
        (\theta_t)_{s,a_2} &\doteq (\theta_0)_{s,a_2} + \sum_{t'=0}^t \eta_{t'} u_{s,a_2,t} \\
        &\geq (\theta_0)_{s,a_2} + \sum_{t'=0}^t \eta_{t'} d_t(s) \frac{r_2-1}{4} \\
        &\geq (\theta_0)_{s,a_2} - m \frac{1-r_2}{4}. \\
    \end{align}
    As a consequence, for any $0\leq r_2<1$ and any $t$, we have:
    \begin{align}
        v_\star(s) - v_t(s) &= 1 - (1 - \pi_t(a_2|s)(1-r_2)) \\
        &\geq \cfrac{\exp\left((\theta_0)_{s,a_2} - m \frac{1-r_2}{4}\right)}{\exp\left((\theta_0)_{s,a_1} + m \frac{1-r_2}{4}\right)}(1-r_2) \\
        &\geq \exp\left((\theta_0)_{s,a_2} - (\theta_0)_{s,a_1} - \frac{m}{2}\right)(1-r_2),
    \end{align}
    which is strictly positive and therefore concludes the proof.
\end{proof}

\subsubsection{Convergence rates}
\label{app:rates}
\softmaxrates*
\begin{remark}

    With \ass{8}, \thm{} \ref{thm:softmax_rates} provides bounds that improve the asymptotic rates that were stated before from $\mathcal{O}(|\mathcal{A}|^3)$ to $\mathcal{O}(|\mathcal{A}|)$).
    
    We note that we were not able to convince ourselves that~\cite{mei2020global} (the only existing convergence rate for the softmax policy) properly handles the case of multiple optimal policies either. We do believe that their proof works in the bandit setting, corresponding to their Section 3.2.1. However the general MDP setting studied in Section 3.2.2 is problematic. More specifically, their proof of Lemma 8 for multiple actions (only found in the Appendix) does not apply to the sum over all optimal actions, but to the sum over all greedy actions with respect to the current policy. This prevents a direct application of the rest of their results and raises the question of whether the rate does hold in that case.
    
    We do not know whether the issue is structural, but we do believe that our proof technique, which ignores RL optimization properties, will not be able to prove a bound in $\mathcal{O}(1/t)$ for problems with multiple optimal policies. Indeed, the problem arises when the optimization commits to one optimal action $a^{\mytop}_1$ at the expense of $a^{\mytop}_2$ in some state $s$ (because e.g. $q_t(s,a^{\mytop}_2)<v_t(s)<q_t(s,a^{\mytop}_1)$), up to a point where $\pi_t(s,a^{\mytop}_1) \gg \pi_t(s,a^{\mytop}_2)$, and $a^{\mytop}_2$ later experiences a significant jump in value due to efficient optimization in its subsequent states, inducing $q_t(s,a^{\mytop}_1)<v_t(s)<q_t(s,a^{\mytop}_2)$. Then, the parameter update will be extremely slow: proportional to the product of $\pi_t(s,a^{\mytop}_2)$ and $(1-\pi_t(s,a^{\mytop}_1)+\pi_t(s,a^{\mytop}_2))$, which are two very small quantities. The policy convergence to $\mathcal{A}^{\mytop}_s$ will be much slower, even go backwards for some time, since $(\theta_t)_{s,a^{\mytop}_1}$ will decrease and $(\theta_t)_{s,a^{\mytop}_2}$ will be too small to have a significant impact on $\sum_{a^{\mytop}\in\mathcal{A}^{\mytop}_s} \pi_t(s,a^{\mytop})$.
     \label{rem:unique}
\end{remark}
\begin{proof}[Proof sketch]
    We consider the partition between the optimal and the suboptimal actions in every state $s$:
    \begin{align}
        \mathcal{A}^{\mytop}_s &= \left\{a\in\mathcal{A}|q_\star(s,a)=v_\star(s)\right\} \quad\text{and}\quad
        \mathcal{A}^{\mybot}_s = \left\{a\in\mathcal{A}|q_\star(s,a)<v_\star(s)\right\}.
    \end{align}
    We define $\delta(s)$ as the gap with the best suboptimal action: $\delta(s) \doteq v_\star(s) - \max_{a\in\mathcal{A}^{\mybot}_s} q_\star(s,a)$. Then, thanks to the uniform convergence and optimality proved in Corollary \ref{cor:convergence} and \thm{} \ref{thm:optimality}, we know that there exists $t_0$, such that $\forall t\geq t_0, \lVert q_\star - q_t\rVert_\infty \leq \frac{\delta(s)}{2}$. Then, we consider the sequence:
    \begin{align}
        (X_{t})_{t\geq t_0} &\doteq \left\{\begin{array}{l}
             X_{t_0} = \max_{a\in\mathcal{A}^{\mytop}_s}(\theta_{t_0})_{s,a} - \max_{a\in\mathcal{A}^{\mybot}_s}(\theta_{t_0})_{s,a}  \\
             X_{t+1} = X_{t} + \frac{\delta(s)}{8}\eta_{t} d_{t}(s) e^{-X_{t}}
             \end{array}\right.
    \end{align}
    and prove that $\forall t\geq t_0$, $X_{t}\leq \max_{a\in\mathcal{A}^{\mytop}_s} (\theta_{t})_{s,a} - \max_{a\in\mathcal{A}^{\mybot}_s} (\theta_{t})_{s,a}$. Further analysis allows us to prove the following rate in policy convergence to suboptimal actions:
    \begin{align}
        \sum_{a\in\mathcal{A}^{\mybot}_s}\pi_{t}(a|s) &\leq \frac{8|\mathcal{A}|}{\delta(s)\sum_{t'=t_0}^{t-1} \eta_{t'} d_{t'}(s)},
    \end{align}
    which allows us to prove the following upper bound on the value convergence rate:
    \begin{equation*}
        v_{\star}(s)-v_t(s) \leq \frac{8|\mathcal{A}|(v_{\mytop}-v_{\mybot})}{(1-\gamma)\min_{s\in\text{supp}(d_{\pi_\star,\gamma})}\delta(s)\sum_{t'=t_0}^{t-1} \eta_{t'} d_{t'}(s)}. \qedhere
    \end{equation*}
\end{proof}
\begin{proof}[Full proof]
    We consider the following partition of the actions in state $s$:
    \begin{align}
        \mathcal{A}^{\mytop}_s &= \left\{a\in\mathcal{A}|q_\star(s,a)=v_\star(s)\right\} \\
        \mathcal{A}^{\mybot}_s &= \left\{a\in\mathcal{A}|q_\star(s,a)<v_\star(s)\right\}.
    \end{align}
    
    We define $\delta(s)$ as the gap with the best suboptimal action\footnote{If there is no suboptimal action, then $\sum_{a\in\mathcal{A}^{\mytop}_s}\pi(a|s)=1$ and we proved what we wanted.}:
    \begin{align}
        \delta(s) \coloneqq v_\star(s) - \max_{a\in\mathcal{A}^{\mybot}_s} q_\star(s,a).
    \end{align}
    
    From the convergence of $v_t(s)$ and $q_t(s,a)$, we also know that for any $\xi > 0$, there exists $t_0$ such that for any $s,a$ and $t \geq t_0$:
    \begin{align}
        v_\star(s) - \xi &\leq v_t(s) \leq v_\star(s) \label{eq:xi1b} \\
        q_\star(s,a) - \xi &\leq q_t(s,a) \leq q_\star(s,a). \label{eq:xi2b}
    \end{align}
    
    We fix $\xi\doteq \frac{\delta(s)}{2}$. We then write the advantage as:
    \begin{align}
        \text{adv}_t(s, a) &= \sum_{a'\neq a}\pi(a'|s)\left(q_t(s, a)-q_t(s,a')\right), \label{eq:advantage}
    \end{align}
    which gives the following bounds:        
    \begin{align}
        \forall a\in\mathcal{A}^{\mytop}_s,\quad\text{adv}_t(s, a) &\geq \left(1-\pi_{t}(\mathcal{A}^{\mytop}_s|s)\right)\left(v_{\star}(s) - \xi - (v_{\star}(s) - \delta(s))\right) \nonumber \\
        &\quad + \left(\pi_{t}(\mathcal{A}^{\mytop}_s|s)-\pi_{t}(a|s)\right)\left(v_{\star}(s) - \xi - v_{\star}(s)\right)\\
        &= \frac{\delta(s)}{2}\left(1-\pi_{t}(\mathcal{A}^{\mytop}_s|s)\right) -\frac{\delta(s)}{2}\left(\pi_{t}(\mathcal{A}^{\mytop}_s|s)-\pi_{t}(a|s)\right), \\
        \forall a\in\mathcal{A}^{\mybot}_s,\quad\text{adv}_t(s, a) &= q_t(s,a) - v_t(s) \\
        &\leq q_{\star}(s,a) - (v_{\star}(s)- \xi) \\
        &\leq -\frac{\delta(s)}{2}.
    \end{align}
    
    The terms in red may imply a negative advantage for actions in $\mathcal{A}^{\mytop}_s$. This is the term we overlooked at main document submission time. This term is null if $\mathcal{A}^{\mytop}_s$ contains only one element $a^{\mytop}_s$ and this is the assumption (\ass{8}) we need to make to progress further.
    
    Therefore, we have for the parameters:
    \begin{align}
        (\theta_{t+1})_{s,a^{\mytop}_s} &\geq (\theta_{t})_{s,a^{\mytop}_s} + \eta_{t} d_{t}(s) \pi_{t}(a^{\mytop}_s|s)\left(1-\pi_{t}(a^{\mytop}_s|s)\right)\frac{\delta(s)}{2}, \label{eq:softmax_mytop}\\
        \forall a\in\mathcal{A}/\{a^{\mytop}_s\},\quad(\theta_{t+1})_{s,a} &\leq (\theta_{t})_{s,a} - \eta_{t} d_{t}(s) \pi_{t}(a|s)\frac{\delta(s)}{2} \leq (\theta_{t})_{s,a}.\label{eq:softmax_mybot}
    \end{align}
    
    Let us consider the following sequence:
    \begin{align}
        (X_{t})_{t\geq t_0} &\coloneqq \left\{\begin{array}{l}
             X_{t_0} = (\theta_{t_0})_{s,a^{\mytop}_s} - \max_{a\in\mathcal{A}/\{a^{\mytop}_s\}}(\theta_{t_0})_{s,a}  
            \vspace{2pt}\\
             X_{t+1} = X_{t} + \frac{\delta(s)}{8}\eta_{t} d_{t}(s) e^{-X_{t}}
        \end{array} \right.\\
    \end{align}
    
    We prove by induction that for all $t\geq t_0$, $X_{t}\leq (\theta_{t})_{s,a^{\mytop}_s} - \max_{a\in\mathcal{A}/\{a^{\mytop}_s\}} (\theta_{t})_{s,a}$. 
    
    We have equality at initialization. Let us now assume the property true for $t$ and prove it still holds for $t+1$. Let $a^{\mybot}_t$ denote the action in $\mathcal{A}/\{a^{\mytop}_s\}$ for which the parameter is maximal at time $t$:
    \begin{align}
        a^{\mybot}_t \in \argmax_{a\in\mathcal{A}/\{a^{\mytop}_s\}} (\theta_{t})_{s,a} = \argmax_{a\in\mathcal{A}/\{a^{\mytop}_s\}} \pi_{t}(a|s).
    \end{align}
    
    We recall that $t\geq t_0$ guarantees that $v_\star - v_t \leq \xi$, which implies that $\pi_{t}(a^{\mytop}_s|s) \geq \frac{1}{2}$. Moreover,
    \begin{align}
        1-\pi_{t}(a^{\mytop}_s|s) &= \cfrac{\sum_{a\in\mathcal{A}/\{a^{\mytop}_s\}}e^{(\theta_{t})_{s,a}}}{\sum_{a\in\mathcal{A}}e^{(\theta_{t})_{s,a}}} \\
        &\geq \cfrac{e^{(\theta_{t})_{s,a^{\mybot}_t}}}{2 e^{(\theta_{t})_{s,a^{\mytop}_s}}}.
    \end{align}
    We now compute:
    \begin{align}
        (\theta_{t+1})_{s,a^{\mytop}_s} - \max_{a\in\mathcal{A}/\{a^{\mytop}_s\}} (\theta_{t+1})_{s,a} &\geq (\theta_{t})_{s,a^{\mytop}_s} - (\theta_{t})_{s,a^{\mybot}_t} + \eta_{t} d_{t}(s) \pi_{t}(a^{\mytop}_s|s)\left(1-\pi_{t}(a^{\mytop}_s|s)\right)\frac{\delta(s)}{2} \label{eq:theta_evol} \\
        &\geq X_t + \eta_{t} d_{t}(s) \frac{1}{2} \cfrac{e^{(\theta_{t})_{s,a^{\mybot}_t}}}{2 e^{(\theta_{t})_{s,a^{\mytop}_s}}} \frac{\delta(s)}{2}\\
        &\geq X_t + \frac{\delta(s)}{8} \eta_{t} d_{t}(s) e^{-(\theta_{t})_{s,a^{\mytop}_s} + (\theta_{t})_{s,a^{\mybot}_t}}\\
        &\geq X_t + \frac{\delta(s)}{8} \eta_{t} d_{t}(s) e^{-X_{t}} \label{eq:monotonicity} \\
        &= X_{t+1}
    \end{align}
    which concludes the induction. 
    
    Above, \eqref{eq:theta_evol} is a direct application of \eqref{eq:softmax_mytop} and \eqref{eq:softmax_mybot}, and \eqref{eq:monotonicity} comes from the induction hypothesis and the fact that $x \mapsto x + \frac{\delta(s)}{8} \eta_{t} d_{t}(s) e^{-x}$ is increasing on $[\frac{\delta(s)}{8} \eta_{t} d_{t}(s), +\infty)$.

    $X_{t}$ belonging to that interval is easily guaranteed for $t$ large enough as $X_{t} \to \infty$ and we assume $\eta_{t} d_{t}(s)$ to be bounded.
    
    Let us now study the sequence $(X_{t})_{t\geq t_0}$. To that end, we define the function $f(t)$ solution on $[t_0, +\infty)$ of the ordinary differential equation (note that $t$ is now a continuous variable):
    \begin{align}
        \left\{\begin{array}{l} f(t_0) = X_{t_0} \\
        \frac{d f(t)}{dt} = \frac{\delta(s)}{8} \eta_{t} d_{t}(s) e^{-f(t)},
        \end{array} \right.
    \end{align}
    where $\eta_{t} d_{t}(s)$ is the piece-wise constant function defined as $\eta_{t} d_{t}(s) = \eta_{\lfloor t \rfloor} d_{\lfloor t \rfloor}(s)$.
    
    From the evolution equations of $X_{t}$ and $f(t)$, we see that $\forall t \in \mathbb{N}, X_{t} \geq f(t)$. Additionally, we have:
    \begin{equation}
        f(t) = \log (\frac{\delta(s)}{8} \int_{t_0}^{t} \eta_{t'} d_{t'}(s) dt' + e^{X_{t_0}}).
    \end{equation}
    
    In particular, going back to $t \in \mathbb{N}$, we obtain:
    \begin{equation}
        X_{t} \geq \log (\frac{\delta(s)}{8} \int_{t_0}^{t} \eta_{t'} d_{t'}(s) dt' + e^{X_{t_0}}) = \log (\frac{\delta(s)}{8} \sum_{t'=t_0}^{t-1} \eta_{t'} d_{t'}(s) + e^{X_{t_0}}).
    \end{equation}
    We can now write the following rate in policy convergence:
    \begin{align}
        1-\pi_{t}(a^{\mytop}_s|s) &= \cfrac{\sum_{a\in\mathcal{A}/\{a^{\mytop}_s\}}e^{(\theta_{t})_{s,a}}}{\sum_{a\in\mathcal{A}}e^{(\theta_{t})_{s,a}}} \\
        &\leq \cfrac{(|\mathcal{A}|-1) e^{(\theta_{t})_{s,a^{\mybot}_t}}}{e^{(\theta_{t})_{s,a^{\mytop}_s}}} \\ 
        &\leq (|\mathcal{A}|-1) e^{-X_{t}} \\
        &\leq |\mathcal{A}| \cfrac{1}{\frac{\delta(s)}{8} \sum_{t'=t_0}^{t-1} \eta_{t'} d_{t'}(s) + e^{X_t}} \\
        &\leq \cfrac{8|\mathcal{A}|}{\delta(s) \sum_{t'=t_0}^{t-1} \eta_{t'} d_{t'}(s)} .
    \end{align}
    
    On the value side, we further get:
    \begin{align}
        v_{\star}(s)-v_t(s) &= \mathbb{P}\left[A = a^{\mytop}_s|A\sim \pi_t(\cdot|s)\right] \times \gamma \mathbb{E}\left[v_{\star}(S')-v_t(S')|S'\sim p(\cdot|s,a^{\mytop}_s)\right] \\
        &\quad + \mathbb{P}\left[A\in \mathcal{A}/\{a^{\mytop}_s\}|A\sim \pi_t(\cdot|s)\right]\left(v_{\star}(s)-\mathbb{E}\left[q_t(s,A)|A\sim \pi_t(\cdot|s)\cap\mathcal{A}/\{a^{\mytop}_s\}\right]\right) \nonumber\\
        &\leq \gamma \mathbb{E}\left[v_{\star}(S')-v_t(S')| S'\sim p(\cdot|s,a^{\mytop}_s)\right]  +  \frac{8 |\mathcal{A}|(v_{\star}(s)-\min_{a\in\mathcal{A}}q_{t}(s,a))}{\delta(s)\sum_{t'=t_0}^{t-1} \eta_{t'} d_{t'}(s)} \\
        &\leq \frac{8 |\mathcal{A}|(v_{\mytop}-v_{\mybot})}{(1-\gamma)\min_{s\in\text{supp}(d_{\pi_\star,\gamma})}\delta(s)\sum_{t'=t_0}^{t-1} \eta_{t'} d_{t'}(s)},
    \end{align}
    where $v_{\mytop}$ (resp. $v_{\mybot}$) stand for the maximal (resp. minimal) value, which is upper bounded by $\frac{r_{\mytop}}{1-\gamma}$ (resp. $\frac{r_{\mybot}}{1-\gamma}$), often times much smaller (resp. larger), and where $\text{supp}(d_{\pi_\star,\gamma})$ denotes the support of the distribution of the optimal policy. This concludes the proof.
\end{proof}

\directrates*
\begin{proof}
    This is a direct consequence of Lemma \ref{lem:proj_lower} and \eq{} \eqref{eq:direct-opt} in \thm{} \ref{thm:optimality-direct}, which shows that the probability of choosing an optimal action grows by updates that sum to infinity until it hits its ceiling probability of 1. Therefore, this ceiling will be hit in a finite time $t_1$ in all states.
\end{proof}

\begin{lemma}\label{lem:proj_lower}
Let $x \in \Delta_n$, the simplex of $\mathbb{R}^n$, and $y \in \mathbb{R}^n$. Assume that there exists $0 < k < n$ and $\alpha > \beta$ such that $\forall i \leq k, y_i \geq \alpha$ and $\forall i > k, y_i \leq \beta$, then
\begin{equation}
    \sum_{i \leq k} \mathcal{P}_{\Delta_n}(x+y)_i \geq \min\left(1, \sum_{i \leq k} x_i + \frac{\alpha - \beta}{2}\right). \label{eq:proj_lower}
\end{equation}
\end{lemma}

\begin{proof}
Let us assume without loss of generality that $\forall i, y_i \geq 0$ (by e.g. subtracting from each of them $\min y_i$). Let us also define the following sets:
\begin{align}
    \mathcal{A}^{\mytop} &= \{ i \leq k \}, \\ 
    \mathcal{A}^{\mybot,\myplus} &= \{ i > k \, | \, \mathcal{P}_{\Delta_n}(x+y)_i > 0 \}, \\
    \mathcal{A}^{\mybot,\myminus} &= \{ i > k \, | \, \mathcal{P}_{\Delta_n}(x+y)_i = 0 \}.
\end{align}
We assume that $\mathcal{A}^{\mybot,\myplus} \neq \emptyset$, as otherwise we would have $\sum_{i \leq k} \mathcal{P}_{\Delta_n}(x+y)_i = 1$, and \eqref{eq:proj_lower} would hold.

We are interested in the effect of $\mathcal{P}_{\Delta_n}$ on the mass of the coordinates of $x+y$ that belong to $\mathcal{A}^{\mytop}$.
Given the assumption that $y$ is coordinate-wise non-negative, the Euclidian projection onto $\Delta_n$ will remove from the coordinates of $x+y$ the mass $\sum_{i} y_i$ that $x$ gained by adding $y$ to it. That mass will be removed as uniformly as possible (corresponding to the equality case in Cauchy-Schwarz inequality), while ensuring that each coordinate remains positive. In our case, this means that $\mathcal{P}_{\Delta_n}$ will remove all the mass of the coordinates in $\mathcal{A}^{\mybot,\myminus}$, and then remove the remaining added mass uniformly from the other coordinates\footnote{We note that this can potentially result in some coordinates in $\mathcal{A}^{\mytop}$ becoming negative. As we are looking for an upper bound on the mass removed from $\mathcal{A}^{\mytop}$, this is not problematic (the mass removed in excess from $\mathcal{A}^{\mytop}$ will then be added back to it and removed from coordinates both in and out of $\mathcal{A}^{\mytop}$, resulting in a smaller decrease).}:
\begin{itemize}
    \item The mass added to the first $k$ coordinates is $\sum_{i \in \mathcal{A}^{\mytop}} y_i$.
    \item The total mass to remove is $K = \sum_{i} y_i$.
    \item The mass removed from $\mathcal{A}^{\mybot,\myminus}$ is $K_{\mathcal{A}^{\mybot,\myminus}}\doteq\sum_{i \in \mathcal{A}^{\mybot,\myminus}} x_i + y_i$, which is larger than $\sum_{i \in \mathcal{A}^{\mybot,\myminus}} y_i$.
    \item The remaining mass $K_{\mathcal{A}^{\mytop}}+K_{\mathcal{A}^{\mybot,\myplus}}= K-K_{\mathcal{A}^{\mybot,\myminus}}$ is smaller than $\sum_{i \in \mathcal{A}^{\mytop} \cup \mathcal{A}^{\mybot,\myplus}} y_i$ and removed uniformly from the coordinates in $\mathcal{A}^{\mytop} \cup \mathcal{A}^{\mybot,\myplus}$.
\end{itemize}
Formally, this gives us:
\begin{align}
    \sum_{i \in \mathcal{A}^{\mytop}} \mathcal{P}_{\Delta_n}(x+y)_i - \sum_{i \in \mathcal{A}^{\mytop}} x_i &= \sum_{i \in \mathcal{A}^{\mytop}} y_i - \frac{|\mathcal{A}^{\mytop}|}{|\mathcal{A}^{\mytop}| + |\mathcal{A}^{\mybot,\myplus}|}\left(K_{\mathcal{A}^{\mytop}}+K_{\mathcal{A}^{\mybot,\myplus}}\right)\\ 
    &\geq \sum_{i \in \mathcal{A}^{\mytop}} y_i - \frac{|\mathcal{A}^{\mytop}|}{|\mathcal{A}^{\mytop}| + |\mathcal{A}^{\mybot,\myplus}|} \sum_{i \in \mathcal{A}^{\mytop} \cup \mathcal{A}^{\mybot,\myplus}} y_i \\
    &= \frac{|\mathcal{A}^{\mybot,\myplus}|}{|\mathcal{A}^{\mytop}| + |\mathcal{A}^{\mybot,\myplus}|} \sum_{i \in \mathcal{A}^{\mytop}} y_i - \frac{|\mathcal{A}^{\mytop}|}{|\mathcal{A}^{\mytop}| + |\mathcal{A}^{\mybot,\myplus}|} \sum_{i \in \mathcal{A}^{\mybot,\myplus}} y_i \\
    &\geq \frac{|\mathcal{A}^{\mybot,\myplus}| |\mathcal{A}^{\mytop}|}{|\mathcal{A}^{\mytop}| + |\mathcal{A}^{\mybot,\myplus}|} \alpha - \frac{|\mathcal{A}^{\mytop}||\mathcal{A}^{\mybot,\myplus}|}{|\mathcal{A}^{\mytop}| + |\mathcal{A}^{\mybot,\myplus}|} \beta \\
    &= \frac{|\mathcal{A}^{\mybot,\myplus}| |\mathcal{A}^{\mytop}|}{|\mathcal{A}^{\mytop}| + |\mathcal{A}^{\mybot,\myplus}|} (\alpha - \beta).
\end{align}
$\frac{|\mathcal{A}^{\mybot,\myplus}| |\mathcal{A}^{\mytop}|}{|\mathcal{A}^{\mytop}| + |\mathcal{A}^{\mybot,\myplus}|}$ is an increasing function of $|\mathcal{A}^{\mybot,\myplus}|$ and $|\mathcal{A}^{\mytop}|$ and both those quantities are larger than $1$. Hence:
\begin{align}
    \sum_{i \in \mathcal{A}^{\mytop}} \mathcal{P}_{\Delta_n}(x+y)_i - x_i &\geq \frac{\alpha - \beta}{2},
\end{align}
which concludes the proof.
\end{proof}

\subsection{Non-asymptotic convergence rates under \ass{2}}

We now define an extra assumption that allows us to get non-asymptotic convergence rates under \ass{2}.

\begin{enumerate} 
    \item[] \con{7}. There exists $d_{\mytop} \geq d_{\mybot} > 0$ such that $\forall s,t, d_{\mytop} \geq d_t(s) \geq d_{\mybot}$.
\end{enumerate}

We will be needing $\Big\lVert \frac{1}{p_0} \Big\rVert_\infty$ and note that under \ass{2}, it is a well-defined quantity.

Finally, we let $M_t \in \mathbb{R}^{|\mathcal{S}||\mathcal{A}| \times |\mathcal{S}||\mathcal{A}|}$ be the diagonal matrix defined by $M_t(s,a;s,a) = \frac{d_t(s)}{d_{\pi_t,\gamma}(s)}$ and $0$ elsewhere. We note that:
\begin{align}
    U(\theta_t,d_t) = M_t \nabla_{\theta|\theta_t} \mathcal{J}(\pi_t).
\end{align}
We also denote by $\lambda_t = \min_s \frac{d_t(s)}{d_{\pi_t,\gamma}(s)}$ and $\Lambda_t = \max_s \frac{d_t(s)}{d_{\pi_t,\gamma}(s)}$ the smallest and largest eigenvalues of $M_t$, and by $\kappa_t = \frac{\Lambda_t}{\lambda_t}$ its condition number. We have the following inequalities:
\begin{align}
    \lambda_t \geq (1 - \gamma) d_{\mybot}, \hspace{2cm} \Lambda_t \leq d_{\mytop}\Big\lVert \frac{1}{p_0} \Big\rVert_\infty, \hspace{2cm} \kappa_t \leq \Big\lVert \frac{1}{p_0} \Big\rVert_\infty \frac{d_{\mytop}}{(1 - \gamma) d_{\mybot}}.\label{eq:kappa}
\end{align}

\begin{theorem}[\textbf{Convergence rate for the softmax parametrization under \ass{2}}]\label{thm:ass3_softmax}
With a softmax parametrization, under \ass{2}, \con{7} and \ass{8},
and with $\eta = \frac{(1-\gamma)^3}{8}\frac{1}{d_{\mytop}} \Big\lVert \frac{1}{p_0} \Big\rVert_\infty^{-1}$, we have the following convergence rate for any $s$:
\begin{align}
    v_\infty(s) - v_t(s) \leq \frac{16 |\mathcal{S}| C}{(1 - \gamma)^7 t} \cdot \Big\lVert \frac{d_{\pi_\star,\gamma}}{p_0} \Big\rVert_\infty^2 \, \Big\lVert \frac{1}{p_0} \Big\rVert_\infty^2 \frac{d_{\mytop}}{d_{\mybot}},
\end{align}
where $C = \frac{1}{\inf_{s\in\mathcal{S},t\geq 1} \pi_t(\pi_\star(s)|s)^2}$ and $\pi_\star$ is the unique (per \ass{8}) optimal policy .
\end{theorem}

\begin{proof}
We note that this result is a generalization of the bound from~\cite{mei2020global} to our more general update rule. The extra $\Big\lVert \frac{1}{p_0} \Big\rVert_\infty \frac{d_{\mytop}}{(1 - \gamma) d_{\mybot}}$ term, coming from the upper bound on $\kappa_t$, is the price we pay for the generalization.

Our proof follows the logic from~\cite{mei2020global} and relies on the fact that $M_t$ is diagonal and positive-definite. Hence it performs some form of conditioning on the gradient, which will affect the convergence rate, but does not change the overall evolution of the policies. More formally, we start by using the strong convexity of $\mathcal{J}$ to get:

\begin{align}
    |\mathcal{J}(\pi_{t+1}) - \mathcal{J}(\pi_t) - \langle\nabla_\theta \mathcal{J}(\pi_t), \theta_{t+1} - \theta_{t}\rangle| \leq \frac{4}{(1 - \gamma)^3} \|\theta_{t+1} - \theta_{t}\|^2.
\end{align}

Now, reusing notations from~\cite{mei2020global}, we let $\delta_t = \mathcal{J}(\pi_*) - \mathcal{J}(\pi_t)$. Let us also remember that $U(\theta_t,d_t) = M_t \nabla_\theta \mathcal{J}(\pi_t)$, and that $\lambda_t$ and $\Lambda_t$ are respectively the smallest and largest eigenvalues of $M_t$. Then
\begin{align}
    \delta_{t+1} - \delta_t &= \mathcal{J}(\pi_t) - \mathcal{J}(\pi_{t+1}) \\
    &\leq - \langle \nabla_\theta \mathcal{J}(\pi_t), \theta_{t+1} - \theta_{t} \rangle + \frac{4}{(1 - \gamma)^3} \|\theta_{t+1} - \theta_{t}\|^2 \\
    &= - \eta_t \langle \nabla_\theta \mathcal{J}(\pi_t), M_t \nabla_\theta \mathcal{J}(\pi_t) \rangle + \frac{4 \eta_t^2}{(1 - \gamma)^3} \|M_t \nabla_\theta \mathcal{J}(\pi_t)\|^2 \\
    &= \sum_s \frac{d_t(s)}{d_{\pi_t,\gamma}(s)} \left(- \eta_t + \frac{4 \eta_t^2}{(1 - \gamma)^3} \frac{d_t(s)}{d_{\pi_t,\gamma}(s)}\right) \langle \nabla_{\theta_s} \mathcal{J}(\pi_t), \nabla_{\theta_s} \mathcal{J}(\pi_t) \rangle \\
    &\leq \sum_s \frac{d_t(s)}{d_{\pi_t,\gamma}(s)} \left(- \eta_t + \frac{4 \eta_t^2}{(1 - \gamma)^3} \Lambda_t\right) \langle \nabla_{\theta_s} \mathcal{J}(\pi_t), \nabla_{\theta_s} \mathcal{J}(\pi_t) \rangle \\
    &\leq - \sum_s \frac{d_t(s)}{d_{\pi_t,\gamma}(s)} \frac{(1-\gamma)^3 }{16} \frac{1}{\Lambda_t} \langle \nabla_{\theta_s} \mathcal{J}(\pi_t), \nabla_{\theta_s} \mathcal{J}(\pi_t) \rangle \hspace{1cm} (\eta_t = \frac{(1-\gamma)^3 }{8} \frac{1}{\Lambda_t}) \\
    &\leq - \sum_s \frac{(1-\gamma)^3 }{16} \frac{\lambda_t}{\Lambda_t} \langle \nabla_{\theta_s} \mathcal{J}(\pi_t), \nabla_{\theta_s} \mathcal{J}(\pi_t) \rangle \\
    &= -\frac{(1-\gamma)^3 }{16} \frac{\lambda_t}{\Lambda_t} \| \nabla_\theta \mathcal{J}(\pi_t)\|^2. \label{eq:upper_bound_grad}
\end{align}

We now apply Lemma~8 from~\cite{mei2020global}, which holds under \ass{2} and \ass{8}, to get
\begin{align}
    \delta_{t+1} - \delta_t &\leq -\frac{(1-\gamma)^3}{16 |\mathcal{S}|} \frac{1}{\kappa_t} \Big\lVert \frac{d_{\pi_\star,\gamma}}{d_{\pi_t,\gamma}} \Big\rVert_\infty^{-2} \cdot \big[ \min_s \pi_t(a_\star(s)|s) \big]^2 \cdot |\mathcal{J}(\pi_*) - \mathcal{J}(\pi_t)|^2 \\
    &\leq -\frac{(1-\gamma)^3}{16 |\mathcal{S}|} \frac{1}{\kappa_t} \Big\lVert \frac{d_{\pi_\star,\gamma}}{p_0} \Big\rVert_\infty^{-2} \cdot \big[ \inf_{t,s} \pi_t(a_\star(s)|s) \big]^2  \cdot \delta_t^2,
\end{align}
where $a_\star(s)$ is the unique optimal action in $s$ (per \ass{8}). We know from \thm{} \ref{thm:optimality-softmax} that for all $s$, $\pi_t(a_\star(s)|s) \to 1$. Thus $\inf_{t,s} \pi_t(a_\star(s)|s) > 0$ and we can define $1 / C \doteq \big[ \inf_{t,s} \pi_t(a_\star(s)|s) \big]^2 > 0$.

Let us write $K_t \doteq \frac{(1-\gamma)^3}{16 |\mathcal{S}| C } \frac{1}{\kappa_t} \Big\lVert \frac{d_{\pi_\star,\gamma}}{p_0} \Big\rVert_\infty^{-2}$. 
We have $\delta_{t+1} - \delta_t \leq - K_t \delta_t^2$. Exactly as in~\cite{mei2020global}, we now proceed by induction to show that $\delta_t \leq \frac{1}{\min_{i \leq t} K_i t}$.

For $t = 2$, we clearly have $\delta_1 \leq \frac{1}{1-\gamma} \leq \frac{1}{2 K_1}$.
Now, $\delta_{t+1} \leq \delta_t - \min_{i \leq t} K_i \delta_t^2$. On $[0, \frac{1}{2 \min_{i \leq t} K_i}]$, the function $x - \min_{i \leq t} K_i x^2$ is increasing, hence:
\begin{align}
    \delta_{t+1} \leq \frac{1}{\min_{i \leq t} K_i t} - \frac{\min_{i \leq t} K_i}{\min_{i \leq t} K_i^2 t^2} \leq \frac{1}{\min_{i \leq t} K_i (t+1)} \leq \frac{1}{\min_{i \leq t+1} K_i (t+1)},
\end{align}
which concludes the induction.

In the end, we get:
\begin{align}
    \mathcal{J}(\pi_*) - \mathcal{J}(\pi_t) &\leq \frac{16 |\mathcal{S}| C}{(1-\gamma)^3} \max_{i \leq t} \kappa_i \Big\lVert \frac{d_{\pi_\star,\gamma}}{p_0} \Big\rVert_\infty^{2} \frac{1}{t} \\
    &\leq \frac{16 |\mathcal{S}| C}{(1-\gamma)^4} \Big\lVert \frac{1}{p_0} \Big\rVert_\infty \frac{d_{\mytop}}{d_{\mybot}} \Big\lVert \frac{d_{\pi_\star,\gamma}}{p_0} \Big\rVert_\infty^{2} \frac{1}{t},
\end{align}
where in the last line we used the bound on $\kappa$ from Eq~\eqref{eq:kappa}.
Note that to match the definitions in the bound from~\cite{mei2020global}, we can write is as:
\begin{equation*}
    \mathcal{J}(\pi_*) - \mathcal{J}(\pi_t) \leq \frac{16 |\mathcal{S}| C}{(1-\gamma)^6 t}  \Big\lVert \frac{(1-\gamma) d_{\pi_\star,\gamma}}{p_0} \Big\rVert_\infty^{2} \, \Big\lVert \frac{1}{p_0} \Big\rVert_\infty \frac{ d_{\mytop}}{d_{\mybot}}.
\end{equation*}
Finally, we move from $\mathcal{J}(\pi_*)$ and $\mathcal{J}(\pi_t)$ to $v_\infty(s)$ and $v_t(s)$ which introduces an extra $\frac{1}{1 - \gamma} \Big\lVert \frac{1}{p_0} \Big\rVert_\infty$ and concludes the proof.
\end{proof}




We can use the above result to compute the time at which the rates from \thm{} \ref{thm:softmax_rates} become valid.
\begin{corollary}
$t_0$ from \thm{} \ref{thm:softmax_rates} can be chosen as follows:
\begin{align}
    t_0 = \frac{32 |\mathcal{S}| C}{(1 - \gamma)^7} \cdot \Big\lVert \frac{d_{\pi_\star,\gamma}}{p_0} \Big\rVert_\infty^2 \cdot \Big\lVert \frac{1}{p_0} \Big\rVert_\infty^2 \frac{d_{\mytop}}{d_{\mybot}} \frac{1}{\inf_s \delta(s)}.
\end{align}
\end{corollary}
\begin{proof}
We know that the rates from \thm{} \ref{thm:softmax_rates} kick in for $t_0$ such that for any $s,a$ and $t \geq t_0$:
    \begin{align}
        v_\infty(s) - \frac{\delta(s)}{2} &\leq v_t(s) \leq v_\infty(s), \nonumber \\
        q_\infty(s,a) - \frac{\delta(s)}{2} &\leq q_t(s,a) \leq q_\infty(s,a). \nonumber
    \end{align}
    
From \thm{}~\ref{thm:ass3_softmax}, it is straightforward to show that it is verified for:
\begin{equation*}
    t_0 \geq \frac{32 |\mathcal{S}| C}{(1 - \gamma)^7 t} \cdot \Big\lVert \frac{d_{\pi_\star,\gamma}}{p_0} \Big\rVert_\infty^2 \, \Big\lVert \frac{1}{p_0} \Big\rVert_\infty^2 \frac{d_{\mytop}}{d_{\mybot}} \frac{1}{\inf_s \delta(s)}. \qedhere
\end{equation*}
\end{proof}

\section{Domains}
\label{app:domains}
\subsection{Chain domain}
\label{app:chain}
    \begin{figure}[h]
        \begin{center}
            \scalebox{0.7}{
                \begin{tikzpicture}[->, >=stealth', scale=0.6 , semithick, node distance=2cm]
                    \tikzstyle{every state}=[fill=white,draw=black,thick,text=black,scale=1]
                    \node[state]    (x0)                {$s_0$};
                    \node[state]    (x1)[right of=x0]   {$s_1$};
                    \node[state]    (x2)[right of=x1]   {$s_2$};
                    \node[state]    (x3)[right of=x2]   {$s_3$};
                    \node[state]    (x4)[right of=x3]   {$s_4$};
                    \node[state]    (x5)[right of=x4]   {$s_5$};
                    \node[state]    (x6)[right of=x5]   {$s_6$};
                    \node[state]    (x7)[right of=x6]   {$s_7$};
                    \node[state]    (x8)[right of=x7]   {$s_8$};
                    \node[state,accepting]    (x9)[right of=x8]   {$s_9$};
                    \node[state,accepting]    (x-1)[above of=x4]   {$s_{-1}$};
                    \node[] (y)[above of=x5] {$r=\beta\gamma^8$};
                    \path
                    (x0) edge[above]    node{$a_1$}     (x-1)
                    (x1) edge[above]    node{$a_1$}     (x-1)
                    (x2) edge[above]    node{$a_1$}     (x-1)
                    (x3) edge[above]    node{$a_1$}     (x-1)
                    (x4) edge[above]    node{$a_1$}     (x-1)
                    (x5) edge[above]    node{$a_1$}     (x-1)
                    (x6) edge[above]    node{$a_1$}     (x-1)
                    (x7) edge[above]    node{$a_1$}     (x-1)
                    (x8) edge[above]    node{$a_1$}     (x-1)
                    (x0) edge[above]    node{$a_2$}     (x1)
                    (x1) edge[above]    node{$a_2$}     (x2)
                    (x2) edge[above]    node{$a_2$}     (x3)
                    (x3) edge[above]    node{$a_2$}     (x4)
                    (x4) edge[above]    node{$a_2$}     (x5)
                    (x5) edge[above]    node{$a_2$}     (x6)
                    (x6) edge[above]    node{$a_2$}     (x7)
                    (x7) edge[above]    node{$a_2$}     (x8)
                    (x8) edge[above]    node{$a_2$}     (x9)
                    (x8) edge[below]    node{$r=1$}   (x9);
                \end{tikzpicture}
            }
            \caption{Deterministic chain MDP. Initial state is $s_0$. Reward is 0 everywhere except when accessing final states $s_{-1}$ and $s_9$. Reward in $s_{-1}$ is set such that $q(s_0,a_1)=\beta q_\star(s_0,a_2)$, with $\beta\in[0,1)$.}
            \label{appfig:chain}
        \end{center}
    \end{figure}
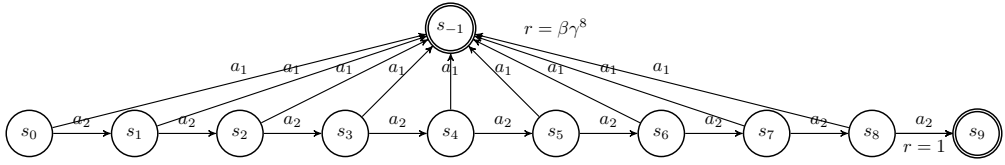
    
    The chain domain is designed to measure the ability of algorithms to overcome misleading rewards, \textit{i.e.} immediate rewards that would guide the updates towards a suboptimal solution. In every state $s_k$, the agent has the opportunity to play action $a_1$ and receive an immediate reward of $\beta\gamma^{|\mathcal{S}|-2}$, or to play $a_2$ and progress to next state $s_{k+1}$ without any immediate reward. A reward of 1 is eventually obtained when reaching state $s_{|\mathcal{S}|-1}$. \fig{} \ref{appfig:chain} represents a chain of size 10 (both terminal states $s_{-1}$ and $s_{\mathcal{S}-1}$ count as a single state). Formally, the deterministic MDP $\langle \mathcal{S}, \mathcal{A}, p, p_0, r, \gamma\rangle$ is constructed from two hyperparameters: the size $|\mathcal{S}|$ and the ratio $\beta$:
    \renewcommand{\arraystretch}{0.1}
    \begin{align}
        \mathcal{S} &\doteq \{s_i\}_{i\in\llbracket 0,|\mathcal{S}|-1\rrbracket} \\
        \mathcal{A} &\doteq \{a_1, a_2\} \\
        p(\cdot|s,a) &\doteq \left\{
        \begin{array}{l}
             \forall s, \,\, p(s_{|\mathcal{S}|-1}|s,a_1)= 1 \\
             \forall i<|\mathcal{S}|-1, \,\, p(s_{i+1}|s_i,a_2)= 1 \\
             \text{the episode terminates when } s_{|\mathcal{S}|-1} \text{ is reached} 
        \end{array} \right. \\
        p_0(s_0) &\doteq 1 \\
        r(s,a) &\doteq \left\{
        \begin{array}{l}
             \forall s, \,\, p(s,a_1)= \beta\gamma^{|\mathcal{S}|-2} \\
             \forall i<|\mathcal{S}|-2, \,\, r(s_{i},a_2)= 0 \\
             r(s_{|\mathcal{S}|-1},a_2)= 1
        \end{array} \right. \\
        \gamma &\doteq 0.99.
    \end{align}
    
    There are two remarkable policies:
    \begin{itemize}
        \item Low-hanging-fruit policy: $\left\{
        \begin{array}{l}
             \pi_{\mybot}(a_1|s) = 1 \\
            v_{\mybot}(s) = \beta\gamma^{|\mathcal{S}|-2} 
        \end{array} \right. $
        \item Optimal policy: $\left\{
        \begin{array}{l}
            \pi_\star(a_2|s) = 1 \\
            v_\star(s_i) = \gamma^{|\mathcal{S}|-2-i} 
        \end{array} \right. $
    \end{itemize}
    
    The policy gradient is attracted by the low-hanging-fruit policy because of its shorter horizon, even though it is sub-optimal.
       
    In order to account for different experimental settings where the size of the chain $|\mathcal{S}|$ or the value ratio $\beta$ differ, we report the experimental results in terms of normalized expected return, formally computed as:
    \begin{align}
        \overline{\mathcal{J}}_\pi = \frac{\mathcal{J}_\pi-\mathcal{J}_{\mybot}}{\mathcal{J}_\star - \mathcal{J}_{\mybot}}, \label{eq:normalized_chain}
    \end{align}
    where $\mathcal{J}_\star = v_\star(s_0)$ is the performance of the optimal policy and $\mathcal{J}_{\mybot} = v_{\mybot}(s_0)$ the performance of the low-hanging-fruit policy. Note in particular that $\overline{\mathcal{J}}_\star = 1$ and $\overline{\mathcal{J}}_{\mybot} = 0$.
    
\subsection{Random MDPs domain}
\label{app:garnets}
    The Random MDPs domain is designed to test the algorithms in situations where the exploration is not an issue. Indeed, by its design of highly stochastic transition functions, Random MDPs will have a non-null chance to visit every state whatever the behavioural policy is. This is therefore a domain where we would expect the policy gradient update to perform well, if not the best, and hope our modified updates still perform comparably. We reproduce the Random MDPs environment published in \sect{} B.1.3 of \cite{Laroche2019}, for which we recall the specifications hereafter.
    
	The Random MDPs domain uses four parameters for the MDP generation: the number of states, the number of actions in each state, the connectivity of the transition function stating how many states are reachable after performing a given action in a given state, and the discount factor $\gamma$. In our experiments, we fix the number of actions to 4, the connectivity to 2, and the discount factor $\gamma$ to 0.99.
    	
	The initial state is arbitrarily set to be $s_0$, we then search with dynamic programming the performance of the optimal policy for all potential terminal state $s_f \in \mathcal{S}/{s_0}$. We select the terminal state for which the optimal policy yields the smaller value function and set it as terminal: $r(s,a,s_f)=1$ and $P(s|s_f,a)=0$ for all $s\in\mathcal{S}$ and $a\in\mathcal{A}$. The reward function is set to 0 everywhere else. We found that the optimal value-function is on average $0.6$, amounting to an average horizon of 10, and has surprising low variance. Below, we write this environmental MDP $\langle \mathcal{S}, \mathcal{A}, p, p_0, r, \gamma \rangle$, its optimal state-action value function $q_\star$, its optimal performance $\mathcal{J}_\star$, and its uniform policy performance  $\mathcal{J}_u$, where $\widetilde{\pi}$ denotes the uniform random policy: $\pi_u(a|s) = \frac{1}{|\mathcal{A}|}$ for all $s\in\mathcal{S}$ and all $a\in\mathcal{A}$.
	
    In order to account for different random MDPs, we report the experimental results in terms of normalized expected return, formally computed as:
    \begin{align}
        \overline{\mathcal{J}}_\pi = \frac{\mathcal{J}_\pi-\mathcal{J}_u}{\mathcal{J}_\star - \mathcal{J}_{u}}. \label{eq:normalized_garnets}
    \end{align}
    Note in particular that $\overline{\mathcal{J}}_\star = 1$ and $\overline{\mathcal{J}}_{u} = 0$.

\section{Full report of finite MDPs planning experiments (under \ass{1})}
\label{app:expes-exact}
\subsection{Chain domain}
\label{sec:chain-exact}
In classic algorithms, the actor is conflicted between two objectives: explore and visit the unknown states, or exploit and follow the best policy under the current incomplete knowledge. The chain experiment depicted on \fig{} \ref{appfig:chain} is intended to test actor-critic algorithms' ability to explore and find the optimal policy under misleading rewards, and then to commit to this optimal policy. In every state, the agent has the opportunity to choose action $a_1$ and receive a fair reward, or to go further without any reward signal and get a chance to eventually yield a higher return.

Our analysis underlines the importance of exploration and states the requirements for guaranteeing sufficient exploration. Our experiment validates these findings. We consider:
\begin{itemize}
    \item the on-policy updates:
    \begin{itemize}
        \item policy gradient (PG) distribution: $\frac{d_{\pi_t,\gamma}}{\Vert d_{\pi_t,\gamma} \rVert_1}$, 
        \item undiscounted distribution: $\frac{d_{\pi_t,1}}{\Vert d_{\pi_t,1} \rVert_1}$, 
    \end{itemize}
    \item uniform distribution: $d_u=\frac{1}{|\mathcal{S}|}$, 
    \item and off-policy trade-offs between the uniform and policy gradient distributions: $o_t d_u + (1-o_t) \frac{d_{\pi_t,\gamma}}{\Vert d_{\pi_t,\gamma} \rVert_1}$
    \begin{itemize}
        \item with $o_t$ constant: 0.1 and 0.48,
        \item with $o_t \in \Theta(t^{-0.48})$: $o_t = 10t^{-0.48}$,
        \item with $o_t \in \Theta(t^{-1})$: $o_t = 10t^{-1}$,
    \end{itemize}
\end{itemize}
\myuline{For a fair comparison, all update densities $d_t$ are normalized: $\sum_s d_t(s) = 1$.}

We consider four experiments:
\begin{itemize}
    \item \textbf{Exp1:} We observe instantaneous performance of each density schedule across iterations.
    \item We look at the speed of discovering the optimal policy. We measure it by tracking the number of updates required to reach a normalised return of 0.48 and its dependency on:
    \begin{itemize}
        \item \textbf{Exp2:} the number of states in the chain,
        \item \textbf{Exp3:} the value rate $\beta$ between the low-hanging-fruit policy and the optimal one,
        \item \textbf{Exp4:} the learning rate.
    \end{itemize}
\end{itemize}

For the softmax parametrization, we also add a policy entropy regularizing term $-\lambda \log \pi(\cdot|s)$ and observe its impact in \textbf{Exp5} (in replacement to \textbf{Exp4} which delivers the same results as in the direct parametrization), as is often done in practice.

\begin{figure*}[t]
	\centering
	\subfloat[vs $t$, with $\beta=0.95$, $|\mathcal{S}|=10$, $\eta=1$]{
		\includegraphics[trim = 5pt 5pt 5pt 5pt, clip, width=0.48\columnwidth]{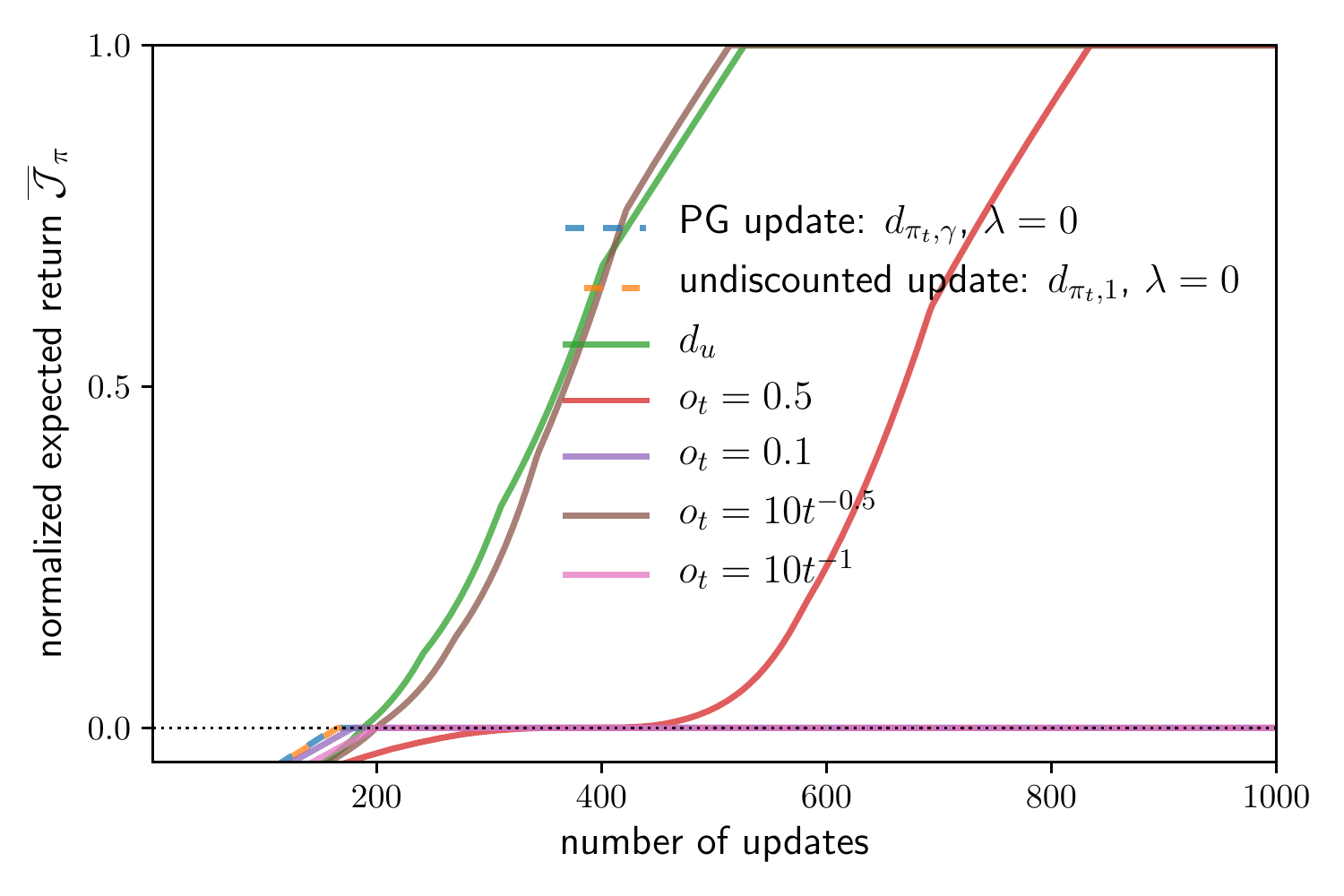}
		\label{fig:exact-direct-vs-t}
	}
	\subfloat[vs $|\mathcal{S}|$, with $\beta=0.95$, $\eta=1$]{
		\includegraphics[trim = 5pt 5pt 5pt 5pt, clip, width=0.48\columnwidth]{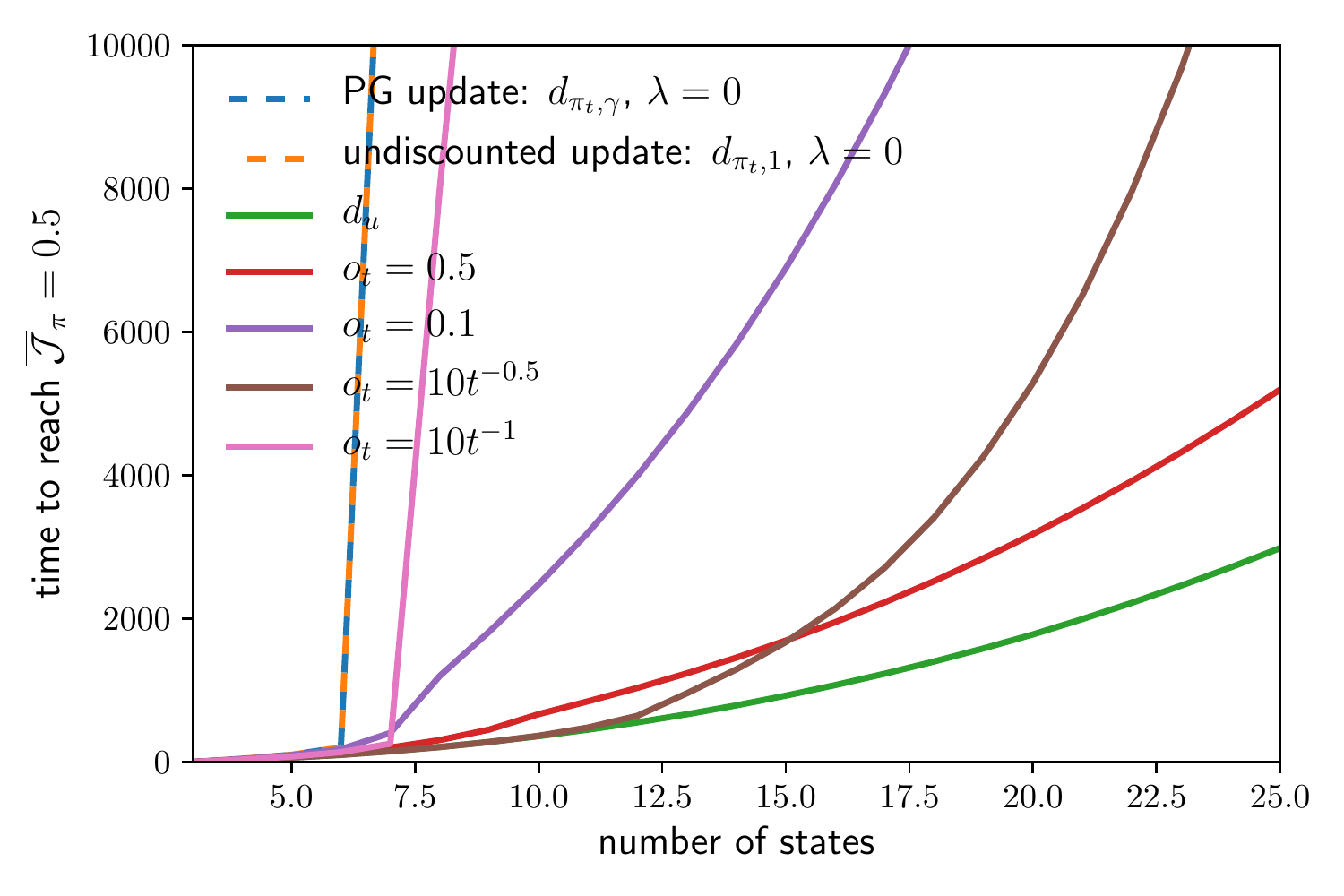}
		\label{fig:exact-direct-vs-S}
	}\\
	\subfloat[vs $\beta$,  with $|\mathcal{S}|=15$, $\eta=1$]{
		\includegraphics[trim = 5pt 5pt 5pt 5pt, clip, width=0.48\columnwidth]{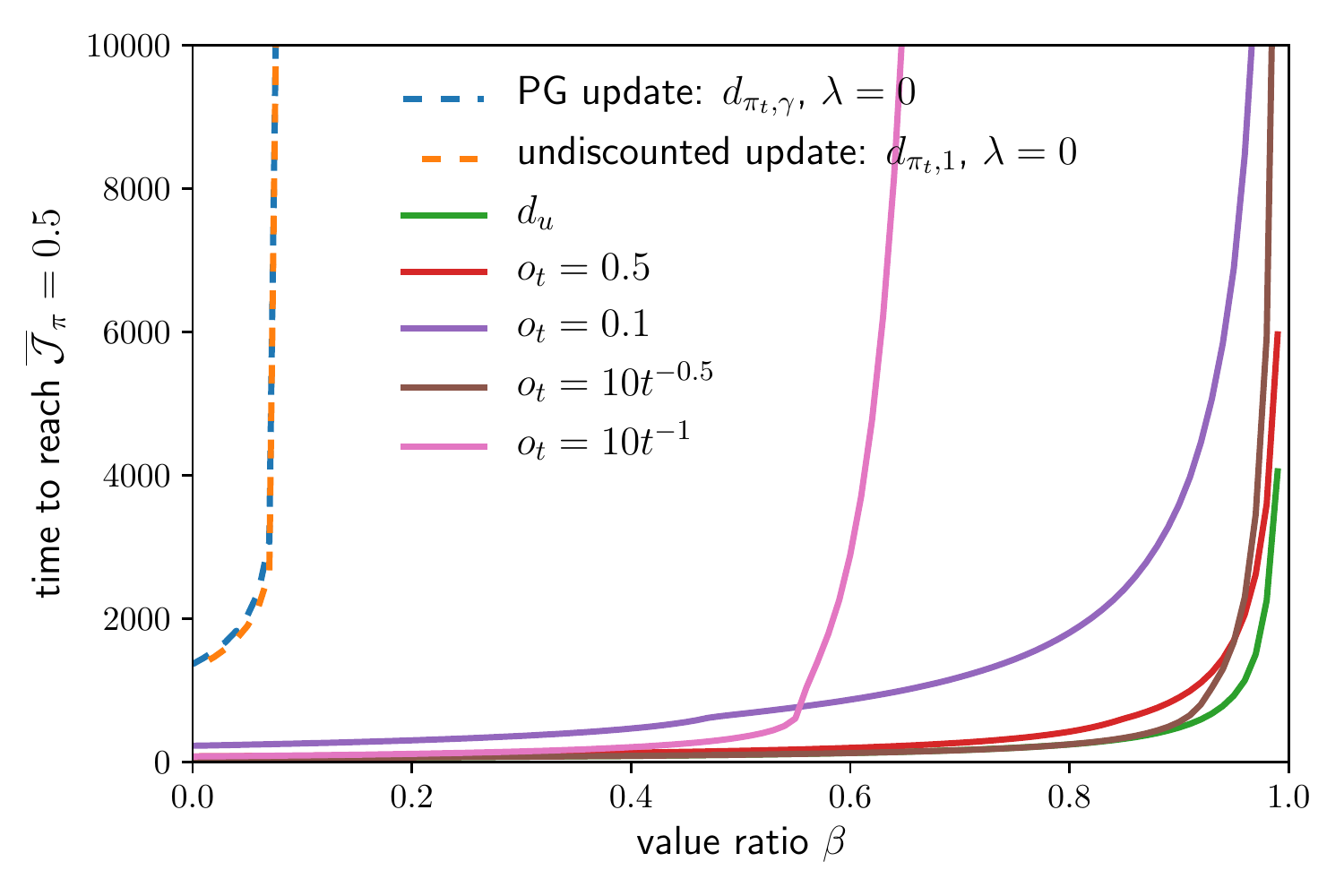}
		\label{fig:exact-direct-vs-beta}
	}
	\subfloat[vs $\eta$, with $\beta=0.48$, $|\mathcal{S}|=10$]{
		\includegraphics[trim = 5pt 5pt 5pt 5pt, clip, width=0.48\columnwidth]{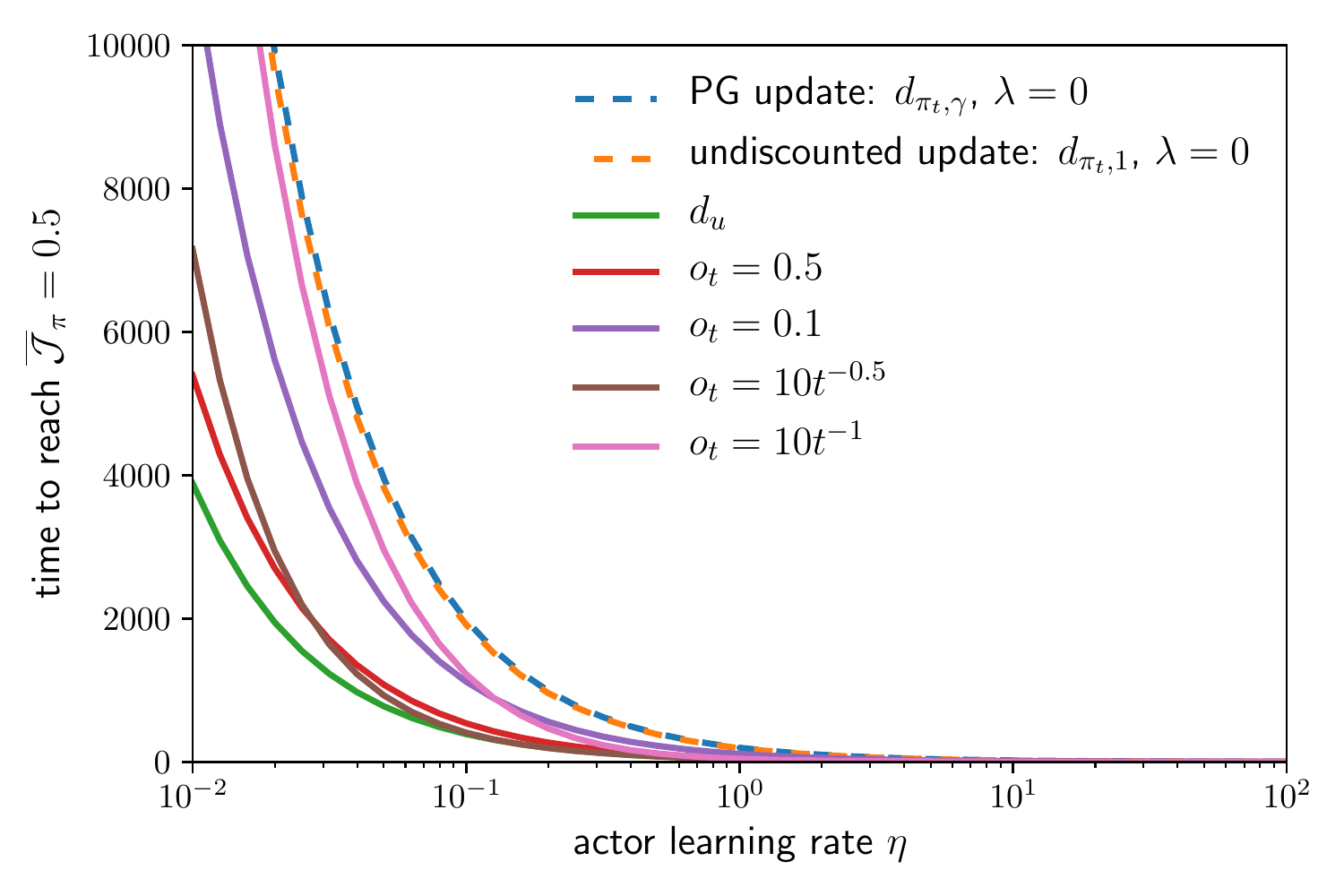}
		\label{fig:exact-direct-vs-eta}
	}
	\caption{Chain experiments with the direct parametrization under \ass{1}.}
	\label{fig:exact-direct}
	\vspace{-10pt}
\end{figure*}
\subsubsection{Direct parametrization}
\paragraph{Exp1:} Its results are displayed on \fig{} \ref{fig:exact-direct-vs-t}.
We observe that on-policy updates converge very fast to the low-hanging fruit policy. Fast convergence can be desirable in easy-planning tasks, but in harder tasks such as the chain it leads to suboptimal behavior. In contrast, relying on a uniform state distribution for the actor updates is efficient and eventually discovers the optimal policy as long as the off-policiness $o_t$ is strong enough. It is worth noting that even under the guarantees offered by having \con{4} satisfied (when $o_t\in\Omega(t^{-1})$, the policy can get stuck on the low-hanging-fruit policy for a long time (as for $o_t = 10t^{-1}$). However, once the optimal policy has been identified the convergence is fast and meets the optimal performance predicted by \thm{} \ref{thm:direct_rates}.

\paragraph{Exp2:} Its results are displayed on \fig{} \ref{fig:exact-direct-vs-S}.
We observe a failure mode of on-policy updates. When the chain is of size 7 or more, the policy converges to the suboptimal solution and will never get out of it. In contrast, \con{4} guarantees the policy will eventually be optimal. However, as illustrated by the $o_t=10t^{-1}$ curve, it might be very long (possibly exponentially long) if the condition is only satisfied tightly. The convergence gets faster as the weight granted to the uniform density increases.

\paragraph{Exp3:} Its results are displayed on \fig{} \ref{fig:exact-direct-vs-beta}.
We observe that the issues spotted in \textbf{Exp1} and \textbf{Exp2} is not due to the near-optimality of the low-hanging-fruit policy. Indeed, even with a reasonably short chain (15), and the low-hanging-fruit policy yielding an expected return 10 times smaller than optimal, the on-policy updates will still converge to it. In contrast, when the density requirement is satisfied, the updates reliably discover the optimal solution.

\paragraph{Exp4:} Its results are displayed on \fig{} \ref{fig:exact-direct-vs-eta}.
We observe that the actor learning rate does not have any impact on the ability of the update to lead to an optimal policy\footnote{We do not show the results, but on-policy updates never find the optimal policy when $\beta=0.7$ for any $\eta$.}. Sill, making it larger speeds up convergence, so in the idealized setting we considered in this section where \ass{1} is satisfied, the higher the actor learning rate, the better. And it has to be noted that setting it to infinity makes the update identical to a policy improvement step in the policy iteration algorithm.
    
\begin{figure*}[t]
	\centering
	\subfloat[$\beta=0.95$, $|\mathcal{S}|=10$, $\eta=10$, $\lambda=0.01$]{
		\includegraphics[trim = 5pt 5pt 5pt 5pt, clip, width=0.48\columnwidth]{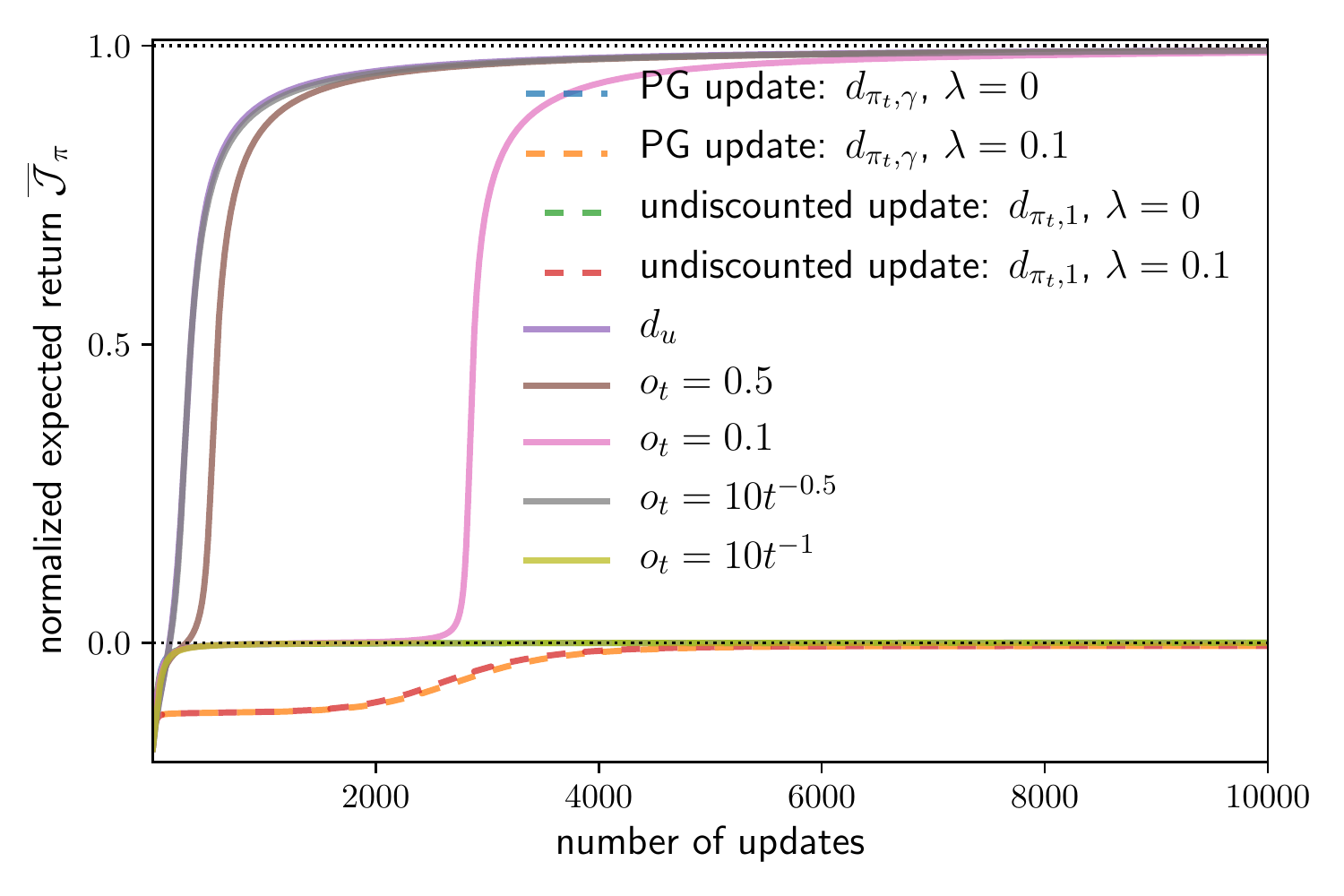}
		\label{fig:exact-softmax-vs-t}
	}
	\subfloat[vs $|\mathcal{S}|$, with $\beta=0.95$, $\eta=10$, $\lambda=0.01$]{
		\includegraphics[trim = 5pt 5pt 5pt 5pt, clip, width=0.48\columnwidth]{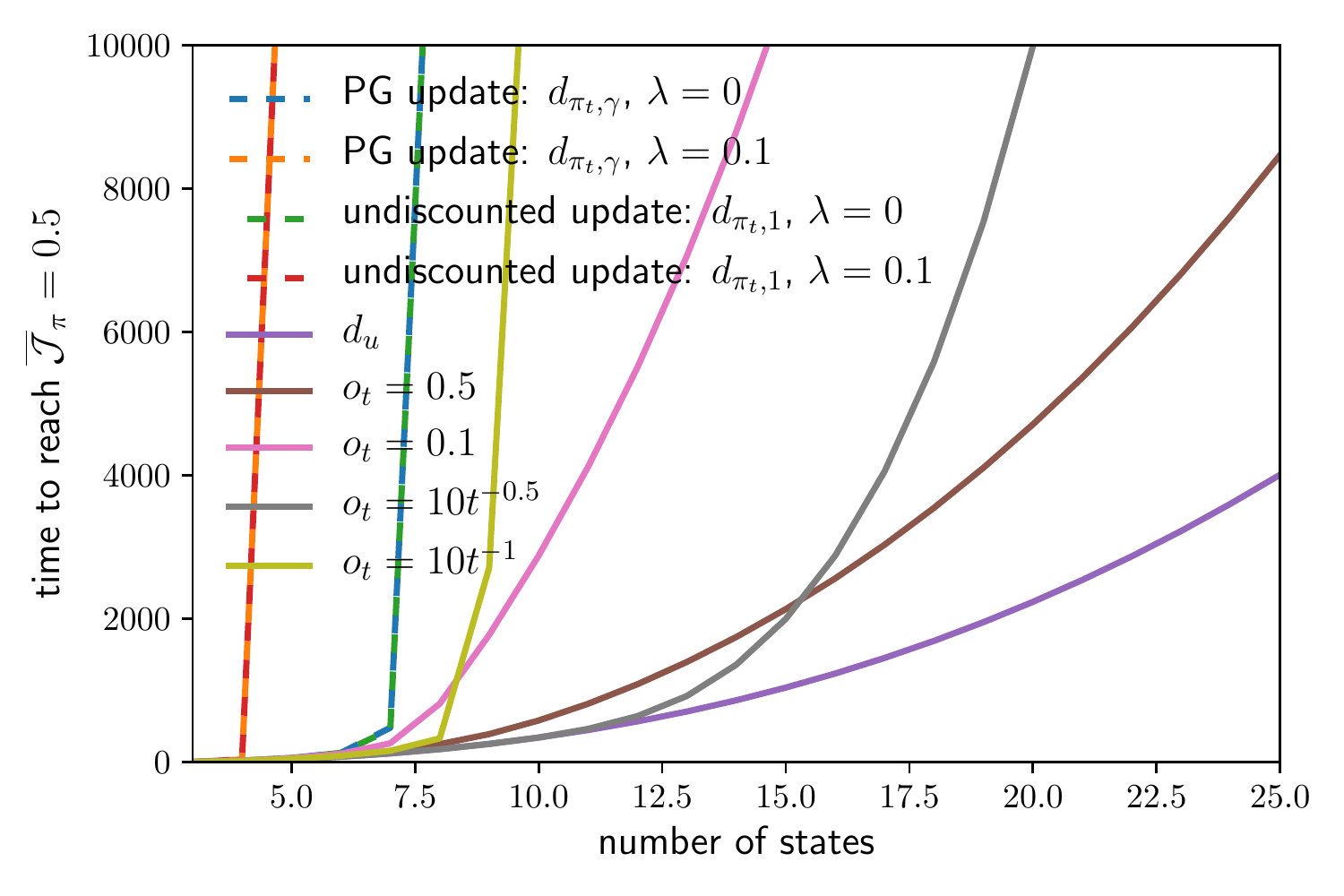}
		\label{fig:exact-softmax-vs-S}
	}\\
	\subfloat[vs $\beta$,  with $|\mathcal{S}|=15$, $\eta=10$, $\lambda=0.01$]{
		\includegraphics[trim = 5pt 5pt 5pt 5pt, clip, width=0.48\columnwidth]{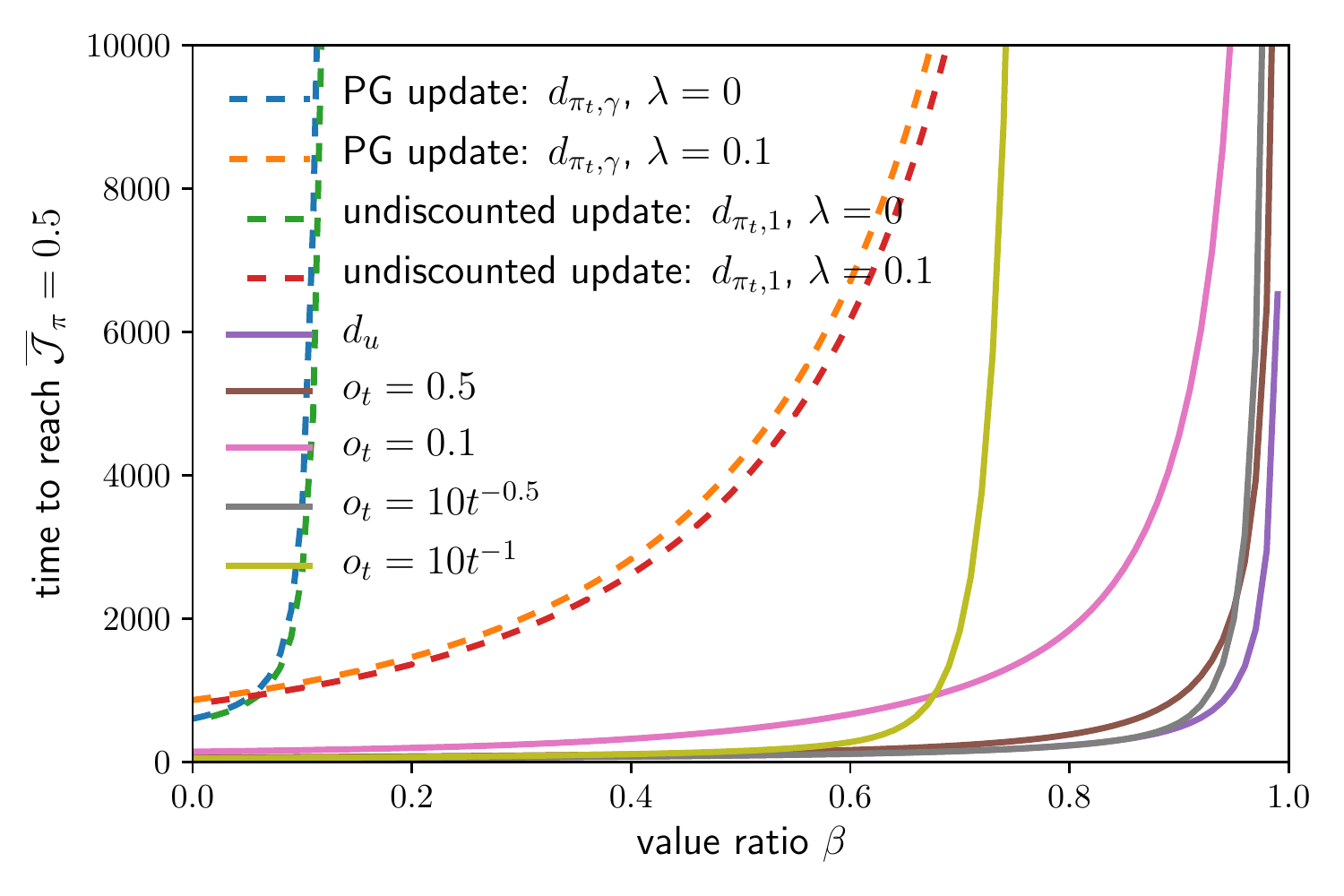}
		\label{fig:exact-softmax-vs-beta}
	}
	\subfloat[vs $\lambda$, with $\beta=0.48$, $|\mathcal{S}|=15$, $\eta=10$]{
		\includegraphics[trim = 5pt 5pt 5pt 5pt, clip, width=0.48\columnwidth]{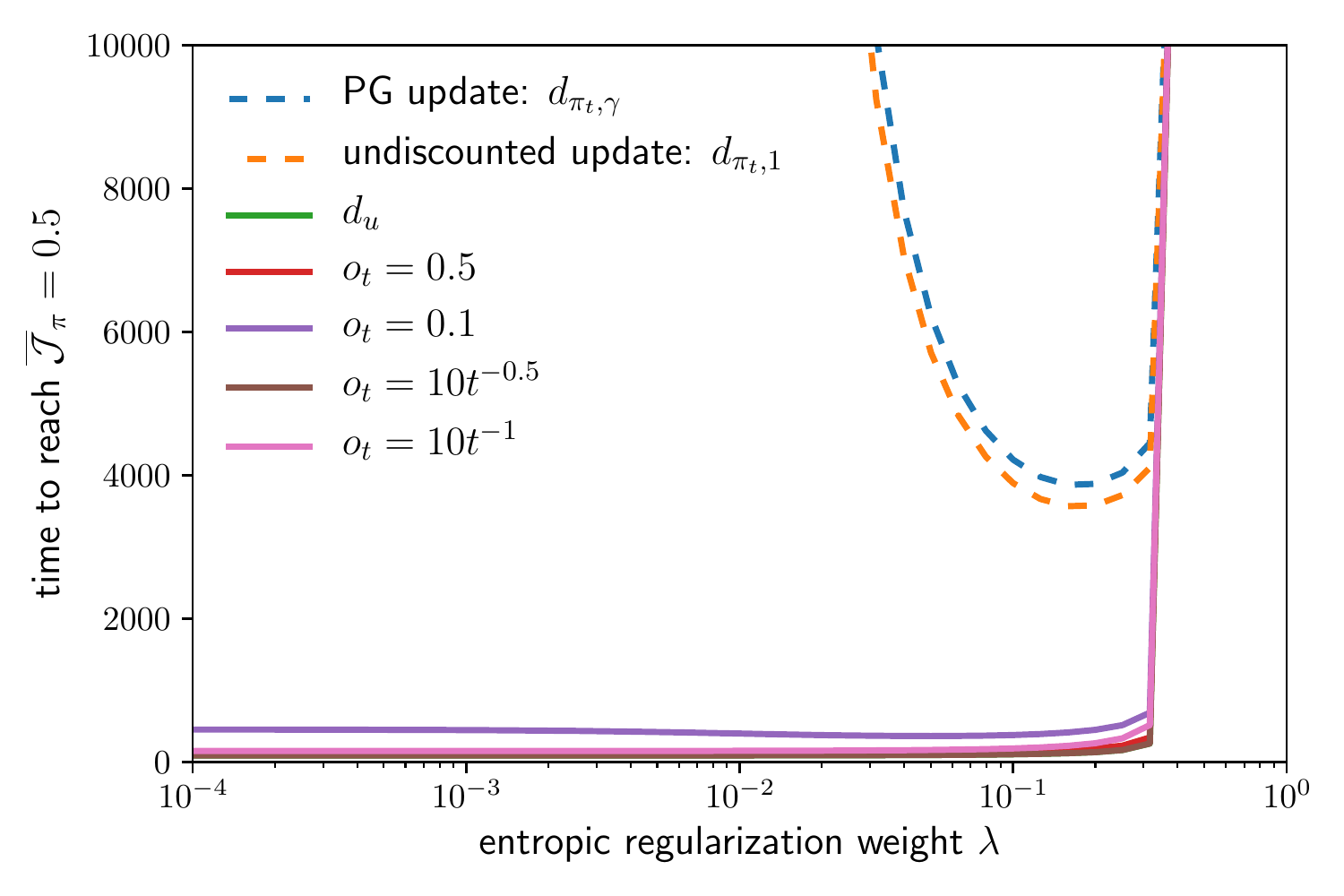}
		\label{fig:exact-softmax-vs-lambda}
	}
	\caption{Chain experiments with the softmax parametrization under \ass{1}.}
	\label{fig:exact-softmax}
	\vspace{-10pt}
\end{figure*}
\subsubsection{Softmax parametrization}
Note that $\eta$ has been multiplied by 10 in the softmax experiments as compared to the direct ones, in order to get a rate of convergence comparable to that of direct parametrization.

\paragraph{Exp1:} Its results are displayed on \fig{} \ref{fig:exact-softmax-vs-t}.
We observe similar results than for the direct parametrization, except that the algorithms never completely converge to the optimal (or suboptimal) policies. We also observe an interesting behaviour of the policy entropy regularized versions of the on-policy updates, at first converging to a policy worse than the low-hanging-fruit one. This stems from the bias induced by the policy entropy regularization: during a first stage it converged to an entropy regularized policy close to the low-hanging-fruit policy, then, after update 2000, it discovered the low-hanging-fruit policy and slowly converged to it.

\paragraph{Exp2:} Its results are displayed on \fig{} \ref{fig:exact-softmax-vs-S}.
We observe exactly the same results as in the direct parametrization.

\paragraph{Exp3:} Its results are displayed on \fig{} \ref{fig:exact-softmax-vs-beta}.
We get results very similar to those of the direct parametrization. The policy entropy regularized updates behaviour is however quite interesting. It seems quite efficient at helping the on-policy updates find the best policy, but the introduced bias prevents learning it when $\beta$ is too high. This last remark gets very clear by looking at \fig{} \ref{fig:exact-softmax-vs-S} where policy entropy regularized updates perform terribly because $\beta$ is high, even worse than the on-policy updates without the policy entropy regularization. It is also worth noting that this is the first experiment where there is a perceptible difference between on-policy updates, with the undiscounted update getting a little edge ($\gamma=0.99$ is rather high).

\paragraph{Exp5:} Its results are displayed on \fig{} \ref{fig:exact-softmax-vs-lambda}.
The first thing to notice is that, for all updates, when the policy entropy regularization gets large, $\lambda\approx 0.25$, the introduced bias is so high that the optimal regularized policy has a normalized performance lower than 0.48 (similar to what we observed on \fig{} \ref{fig:exact-softmax-vs-t}). We also notice, that, for $\lambda < 0.25$, its impact on updates that satisfy \con{4} is near to null. Finally, for the on-policy updates, there is a sweet spot where it allows them to converge in the vicinity of the optimal policy reasonably fast, but it is rather narrow in terms of $\lambda$ values, and slow in terms of time to find the optimal policy. And finally, let us recall that it will not converge to the optimal policy but to a stochastic policy that is optimal under the policy entropy regularized objective. We elaborate this last point further in Appendix \ref{app:policy_regularization}. 
  
\begin{figure*}[t]
	\centering
	\subfloat[Direct param, mean performance]{
		\includegraphics[trim = 5pt 5pt 5pt 5pt, clip, width=0.48\columnwidth]{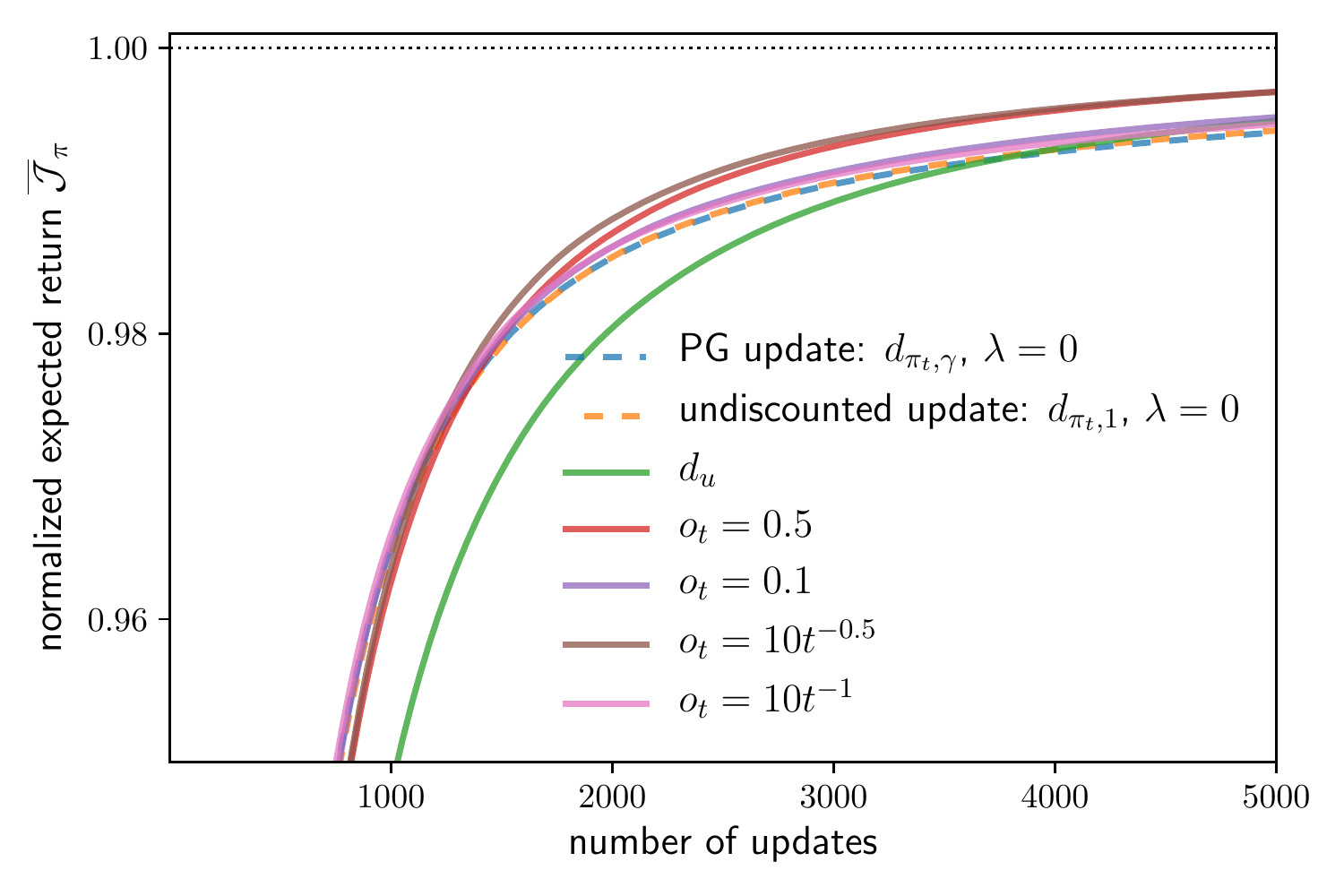}
		\label{fig:exact-garnets-direct-mean}
	}
	\subfloat[Direct param, decile performance]{
		\includegraphics[trim = 5pt 5pt 5pt 5pt, clip, width=0.48\columnwidth]{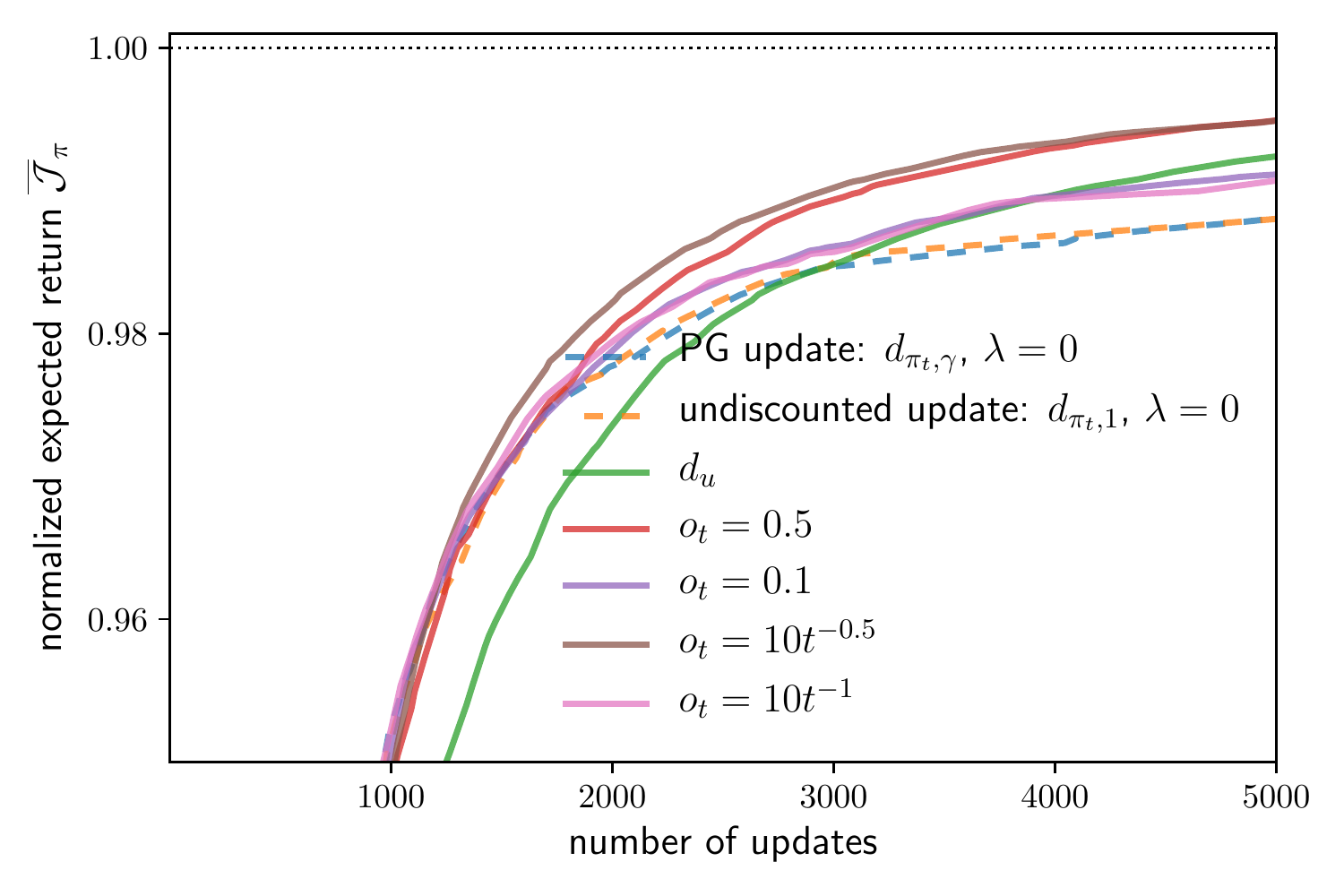}
		\label{fig:exact-garnets-direct-decile}
	}\\
	\subfloat[Softmax param, mean performance]{
		\includegraphics[trim = 5pt 5pt 5pt 5pt, clip, width=0.48\columnwidth]{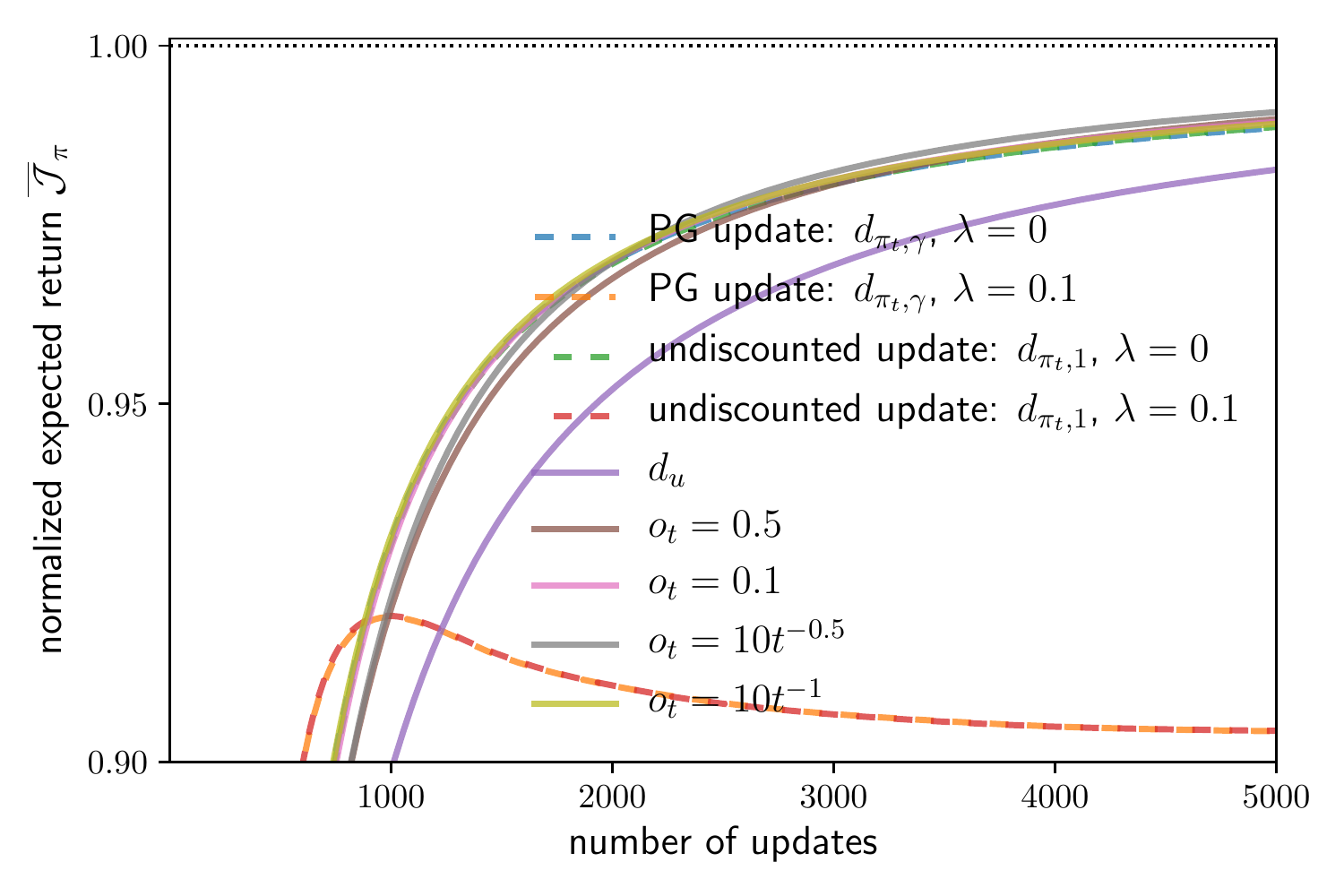}
		\label{fig:exact-garnets-softmax-mean}
	}
	\subfloat[Softmax param, decile performance]{
		\includegraphics[trim = 5pt 5pt 5pt 5pt, clip, width=0.48\columnwidth]{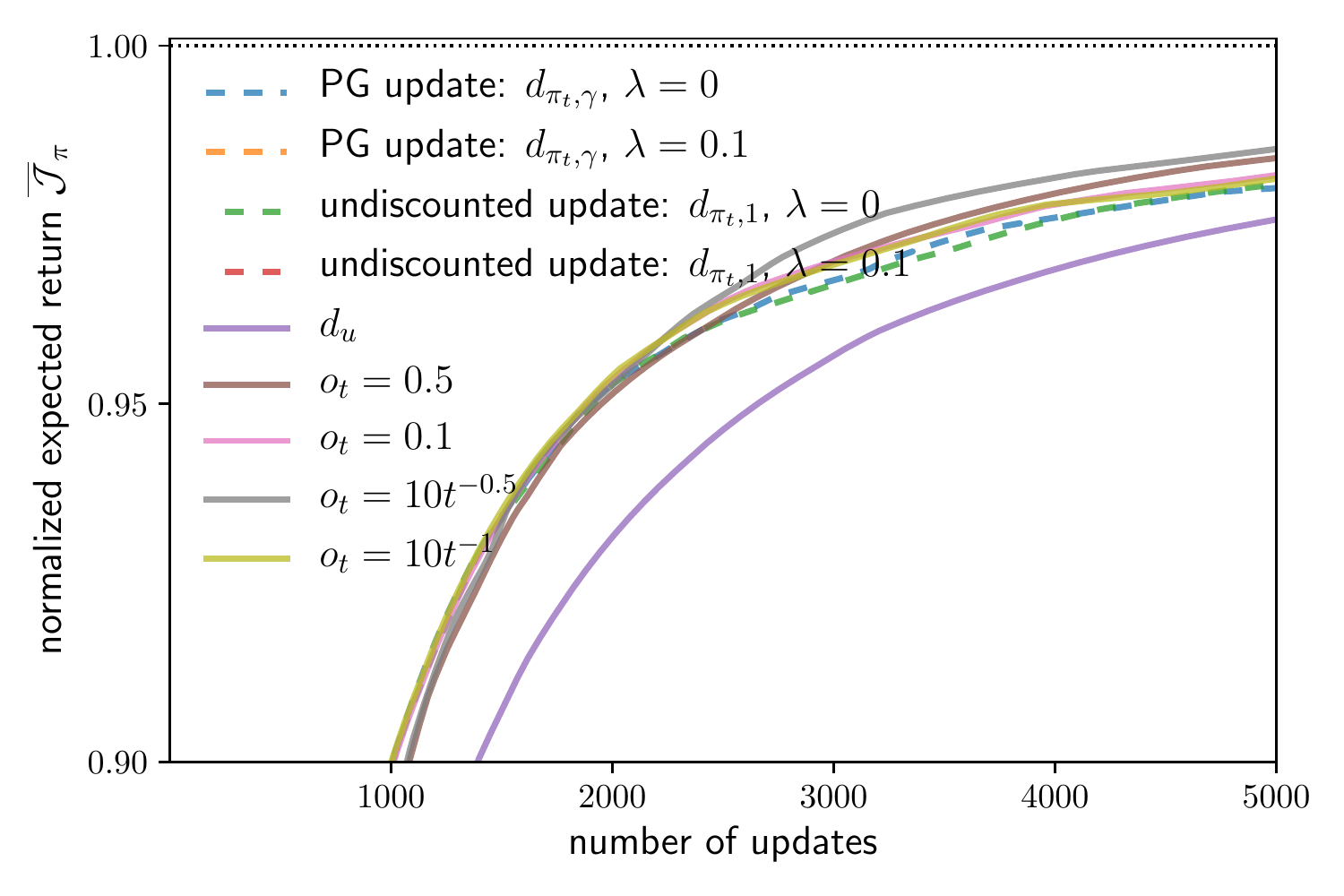}
		\label{fig:exact-garnets-softmax-decile}
	}
	\caption{Random MDPs experiments (50 states, 4 actions) under \ass{1}.}
		\label{fig:exact-garnets}
	\vspace{-10pt}
\end{figure*}
\subsection{Random MDPs domain}
The Random MDPs experiment described in \sect{} \ref{app:garnets} is intended to test the updates in random MDPs, that have not been designed for the on-policy updates to fail, and where they are known to be efficient. We will consider a single experiment consisting in repeating the following process for more than 100 runs:
\begin{enumerate}
    \item Generate a Random MDP.
    \item Run the policy optimization in this MDP with the same densities as those described in \sect{} \ref{sec:chain-exact}.
\end{enumerate}

We then report the mean (resp. decile) of the normalized expected return over the runs, as the mean (resp. decile) performance. We run the same experiment with both the direct and softmax parametrizations.

\subsubsection{Direct parametrization}
The mean performance is displayed on \fig{} \ref{fig:exact-garnets-direct-mean}. All algorithms perform almost equally good. Indeed, in absence of purposely designed deceptive rewards, the policy gradient and the undiscounted update are rather reliable at finding strong policies. Still, we start observing from update 2,000 on that the updates with low off-policy $o_t$ updates (\textit{i.e.} small weight on the uniform policy) start losing ground to the ones with high off-policiness $o_t$. The notable exception is $d_u$ that suffered some lateness at the start of curve, but is catching up and overtaking the low off-policy $o_t$ curves around the 4,000th update. This suggests that a proper scheduling of $o_t$ might be worth studying in the future: in early stages, follow the true gradient and increase exploration when the value improvement slows down.

The decile performance is displayed on \fig{} \ref{fig:exact-garnets-direct-decile}. We observe that, in addition to yielding a higher mean performance, high off-policy $o_t$ updates are also very stable: their decile performance is very close to their mean performance. In contrast, the decile performance of the on-policy updates seem to stagnate around 0.99 at the end of the curve. 

\subsubsection{Softmax parametrization}
The mean performance is displayed on \fig{} \ref{fig:exact-garnets-softmax-mean}. The results are similar to that with the direct parametrization except that:
\begin{itemize}
    \item the mean performance remains significantly further from the optimal (around 0.99 here vs 0.997 with direct parametrization). This is due to the fact that softmax parametrization takes longer to converge.
    \item the gap between high and low off-policy $o_t$ updates is much thinner. We believe this might also be due to the slowness of the softmax parametrization convergence. As a consequence, $d_u$ remains far from the other updates. 
    \item the policy entropy regularized versions of the on-policy updates perform much worse. Indeed, the bias introduced by the policy entropy regularization impairs the convergence to a policy that is near optimal with respect to the normalized expected return (which is our objective function). Nevertheless, it has to be noted that it produced the faster improvement in the early stages (until update 1,000). We do not have an explanation for that.
\end{itemize}

The decile performance is displayed on \fig{} \ref{fig:exact-garnets-softmax-decile}. We observe the same behavior as with the direct parametrization, with the exception of the narrower gaps already discussed above.

\subsection{General conclusion for this set of experiments}
We test the performance across time, actor learning rate $\eta$, MDP parameters $|\mathcal{S}|$ and $\beta$, off-policiness $o_t$, and policy entropy regularization $\lambda$, with both direct and softmax parametrizations, on both the chain and random MDPs. The full report is available in Appendix \ref{app:expes-exact}. 

On the chain domain, we empirically confirm that enforcing updates with $d_t$ including an off-policy component helps path discovery and policy planning, while on-policy updates fail at converging to the optimal policy, even with policy entropy regularization, in a reasonable amount of time. By playing with the chain domain hyperparameters $\beta$ and $|\mathcal{S}|$, we observe that on-policy updates are sensitive to both: even with a small $\beta=0.1$, a reasonable sized chain $|\mathcal{S}|=15$ cannot be solved. In contrast, updates that are partly off-policy allow to converge fast to optimality even with $\beta=0.95$ and $|\mathcal{S}|=25$.

On the random MDP domain, we empirically observe that discounted and undiscounted updates perform well but still oftentimes stagnate at 99\% of optimal performance, while $d_t$ including an off-policy component achieves even closer to optimality. Finally, we also note that a purely uniform $d_t$ slows down training in the random MDPs experiment. We conclude that it is best to include both on-policy and off-policy components in $d_t$. Experiments also allowed us to observe the biased convergence implied by the policy entropy regularization.

\section{Full report of finite MDPs reinforcement learning experiments}
\label{app:expes-sample}
\subsection{Algorithms design}
\label{app:algos}
This section provides the full implementation of the algorithms on the testbed.

\begin{algorithm}[ht]
\caption{On-policy algorithms for experiments in finite MDPs (by default $\nu=0$ and $\lambda=0$).}
\multiline{\textbf{Input:} UCB reward bonus weight $\nu$ and policy entropy regularization weight $\lambda$.}
\multiline{\textbf{Input:} Initialization of the value $q_0$, set actor learning rate $\eta$ and critics learning rate $\eta_c$.}
\begin{algorithmic}[1]
    \State Initialize parameters $\theta=0$ of the actor $\pi\doteq \pi_{\theta}$ and the critic $q \doteq q_0$.
    \For{$t=0$ to $\infty$}
        \State Sample a transition $\tau_t = \langle s_t,a_t\sim\pi_{b}(\cdot|s_t),s_{t+1}\sim p(\cdot|s_t,a_t),r_t\sim r(\cdot|s_t,a_t)\rangle$.
        \renewcommand\algorithmicdo{}
        \For{once, do a single on-policy update:}
            \State Update critic $q$ with a SARSA update on $\tau_t$:
                \begin{align}
                    q(s_t,a_t)\leftarrow q(s_t,a_t)+\eta_c\left(r+\gamma\sum_{a'\in\mathcal{A}}q(s',a')\pi(s',a')-q(s,a)\right)
                \end{align}
            \State Perform a stochastic update step in state $s$ on Dr Jekyll's actor $\theta$:
            \begin{align}
                \theta_{s_t,a_t} \leftarrow \theta_{s_t,a_t} + \eta d \left(r + \nu\sqrt{\frac{\log t}{n_{s_t,a_t}}} + \lambda \log \pi(a_t|s_t) + \sum_{a'\in\mathcal{A}}\pi(s'_y,a')q(s'_t,a')\right) \nabla_{\theta} \pi(a_t|s_t)
            \end{align}
        \EndFor
        \renewcommand\algorithmicdo{do}
    \EndFor
\end{algorithmic}
\label{alg:classic}
\end{algorithm}

\begin{algorithm}[ht]
\caption{Dr Jekyll \& Mr Hyde algorithm for experiments in finite MDPs.}
\multiline{\textbf{Input:} Scheduling of exploration $(\epsilon_t)=\frac{\epsilon_1}{t^{\alpha_\epsilon}}$, and of off-policiness $(o_t)=\frac{o_1}{t^{\alpha_o}}$.}
\multiline{\textbf{Input:} Initialization of the value $q_0$, set actor learning rate $\eta$ and critics learning rate $\eta_c$.}
\begin{algorithmic}[1]
    \State Initialize $\mathring{D}=\emptyset$, parameters $\mathring{\theta}=0$ of Dr Jekyll's actor $\mathring{\pi}\doteq \pi_{\mathring{\theta}}$, and critic $\mathring{q} \doteq q_0$.
    \State Initialize $\tilde{D}=\emptyset$ and Mr Hyde's value to $\tilde{q}=q_0$, and policy $\tilde{\pi}$ to uniform.
    \State Set the behavioural policy to Dr Jekyll: $\pi_{b} \leftarrow \mathring{\pi}$ and the working replay buffer $D_{b} \leftarrow \mathring{D}$.
    \For{$t=0$ to $\infty$}
        \State Sample a transition $\tau_t = \langle s_t,a_t\sim\pi_{b}(\cdot|s_t),s_{t+1}\sim p(\cdot|s_t,a_t),r_t\sim r(\cdot|s_t,a_t)\rangle$
        \State Add it to the working replay buffer $D_b\leftarrow D_b\cup\{\tau_t\}$.
        \InlineIfThen{$\tau$ was terminal,}{$(\pi_{b}, D_b) \leftarrow (\tilde{\pi}, \tilde{D})$ w.p. $\epsilon_t$, $(\pi_{b}, D_b) \leftarrow (\mathring{\pi}, \mathring{D})$ otherwise.}
        \renewcommand\algorithmicdo{}
        \For{once, do a single update:}
            \State $\tau \doteq \langle s,a,s',r \rangle \sim \tilde{D}$ w.p. $o_t$, $\tau \doteq \langle s,a,s',r \rangle \sim \mathring{D}$ otherwise.
            \State Update Mr Hyde's critic $\tilde{q}$ with a $q$-learning update on $\tau$, then $\tilde{\pi}$ is greedy on $\tilde{q}$:
                \begin{align}
                    \tilde{q}(s,a)\leftarrow \tilde{q}(s,a)+\eta_c\left(\frac{1}{\sqrt{n_{s,a}}}+\gamma\max_{a'\in\mathcal{A}}\tilde{q}(s',a')-\tilde{q}(s,a)\right)\quad\text{then}\quad\tilde{\pi}\leftarrow \argmax_{a\in\mathcal{A}} \tilde{q}
                \end{align}
            \State Update Dr Jekyll's critic $\mathring{q}$ with a SARSA update on $\tau$:
                \begin{align}
                    \mathring{q}(s,a)\leftarrow \mathring{q}(s,a)+\eta_c\left(r+\gamma\sum_{a'\in\mathcal{A}}\mathring{q}(s',a')\mathring{\pi}(s',a')-\mathring{q}(s,a)\right)
                \end{align}
            \State Perform an expected update step in state $s$ on Dr Jekyll's actor $\mathring{\theta}$:
            \begin{align}
                \forall b\in\mathcal{A},\quad\quad\mathring{\theta} \leftarrow \mathring{\theta} + \eta \mathring{q}(s,b) \nabla_{\mathring{\theta}} \mathring{\pi}(b|s)
            \end{align}
        \EndFor
        \renewcommand\algorithmicdo{do}
    \EndFor
\end{algorithmic}
\label{alg:finiteMDP-JH}
\end{algorithm}

We first formally describe the on-policy algorithms in Algorithm \ref{alg:classic}. PG update has $d$ equal to $\gamma$ at the power of the episode's timestep, while $d=1$ for undiscounted updates. Unless specified otherwise UCB exploration bonus parameter $\nu$ and policy entropy regularization parameter $\lambda$ are set to 0.

We then formally describe the implementation of \jh{} in Algorithm \ref{alg:finiteMDP-JH}.

Not shown in the algorithms: in order to compare equally each update (for instance, discounted updates would observe less update amplitude than undiscounted ones), we normalize the empirical update steps $d$ by the average of their values in history. This normalizing step is not useful for the algorithms themselves, its only use is to control potential side effects of the update choice.

In order to prepare for its deep RL implementation, we extensively analyse the \jh{} algorithm, but only on the softmax parametrization. Indeed, there is no true equivalent to direct parametrization with neural networks. We analysed the performance of the algorithms against time (number of updates), the off-policiness $o_t$, the actor learning rate $\eta$, and the parameters of the chain domain: its size and the value ratio $\beta$. We will look at the same hyperparameters in the Reinforcement Learning setting, but also at the newly introduced hyperparameters: the critic learning rate $\eta_c$, the $q$-function initialization $q_0$, and the exploration schedule $\epsilon_t$.

But first let us specify the default hyperparameter settings:
\begin{itemize}
    \item Environment hyperparameters:
        \begin{itemize}
            \item Chain domain: $|\mathcal{S}|=10$, $|\mathcal{A}|=2$, $\gamma=0.99$, $\beta=0.8$.
            \item Random MDPs domain: $|\mathcal{S}|=100$, $|\mathcal{A}|=4$, $\gamma=0.99$.
        \end{itemize}
    \item Algorithms hyperparameters:
        \begin{itemize}
            \item All algorithms: $\eta=1$, $\eta_c=0.1$, $q_0=0$, $\lambda=0$, $\nu=0$.
            \item \jh{} specific: $\epsilon_t$ and $o_{t}$ are always shown.
            \item In all the experiments, Mr Hyde is trained with $Q$-learning on a $\gamma$ discounted objective from UCB rewards: $\tilde{r}(s,a)\doteq\frac{1}{\sqrt{n_{s,a}}}$.
        \end{itemize}
\end{itemize}

For \jh{}, we generally report the \textit{global} performance, \textit{i.e.} the expected performance of the mixture of Dr Jekyll and Mr Hyde, but also sometimes, when mentioned, we also look at the performance of the Dr Jekyll policy, because it is the best policy known to the algorithm at a given time.

\subsection{Vs. time (number of trajectories):}
\begin{figure*}[t]
	\centering
	\subfloat[Mean global performance]{
		\includegraphics[trim = 5pt 5pt 5pt 5pt, clip, width=0.48\columnwidth]{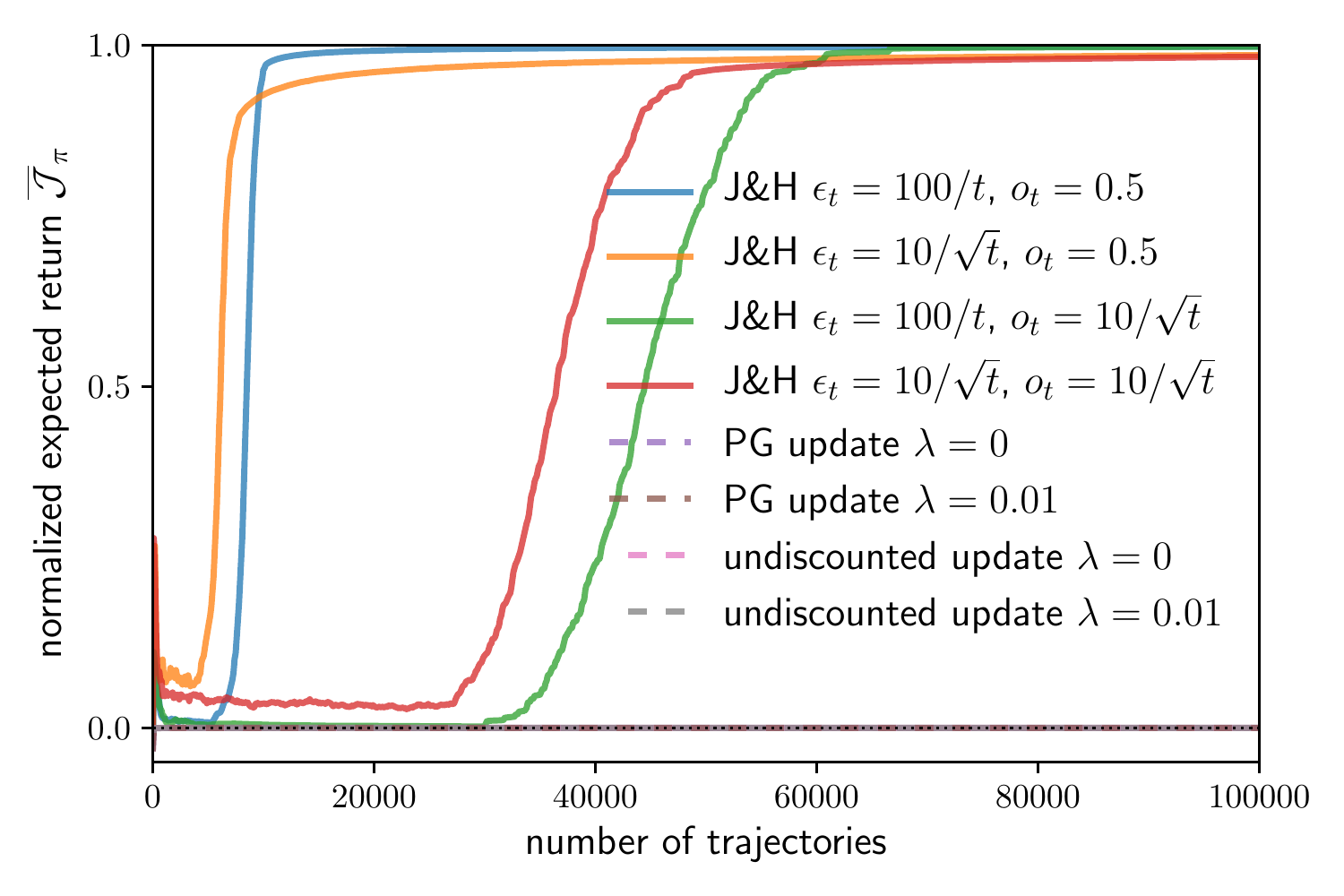}
		\label{fig:sample-chain-time-mean}
	}
	\subfloat[Mean Jekyll performance]{
		\includegraphics[trim = 5pt 5pt 5pt 5pt, clip, width=0.48\columnwidth]{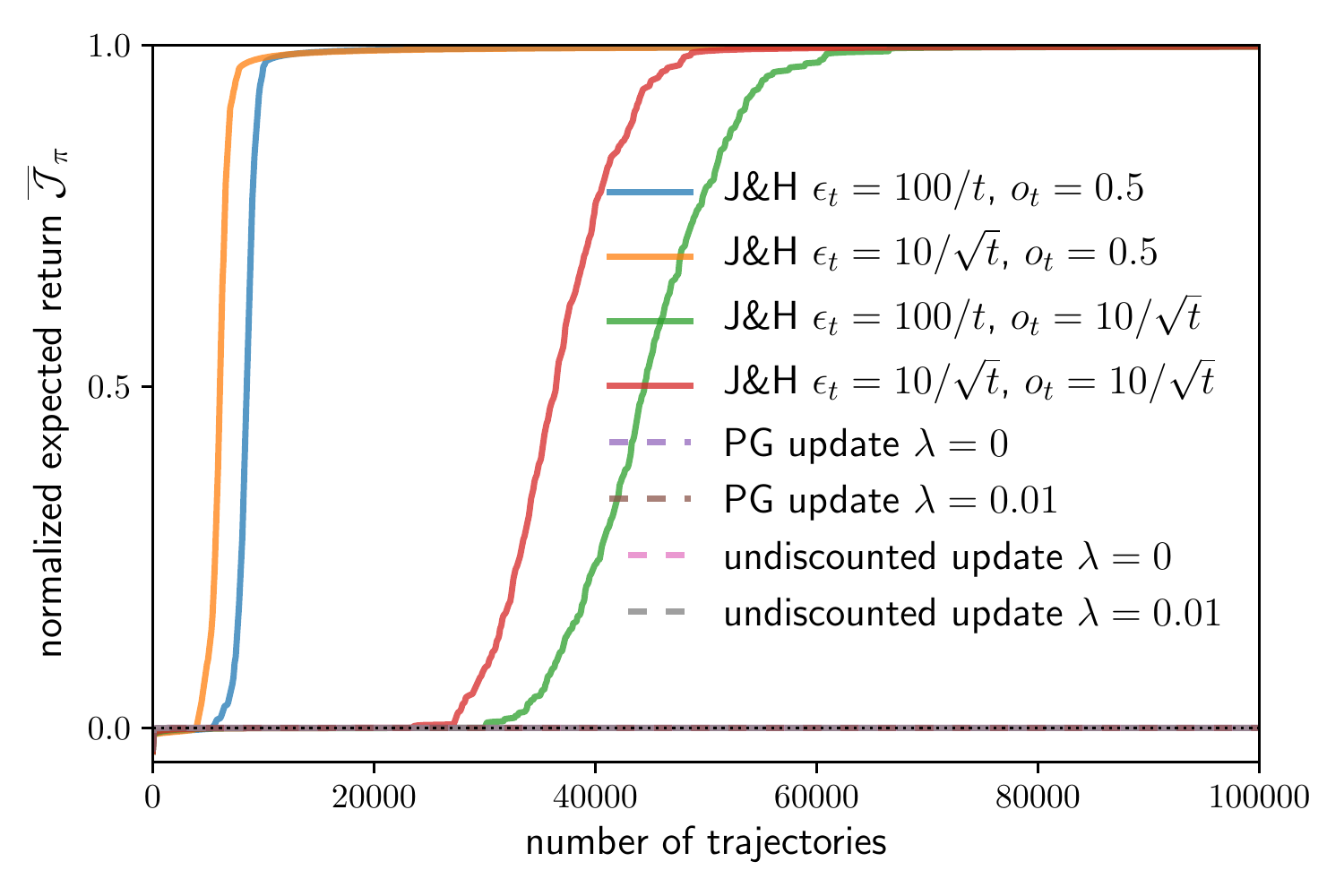}
		\label{fig:sample-chain-time-mean-jek}
	}
	\caption{Chain MDPs experiments vs. number of trajectories (200+ runs).}
		\label{fig:sample-chain-time}
	\vspace{-10pt}
\end{figure*}
In the chain experiment, we expect \jh{} to perform much better than on-policy updates (policy gradient and undiscounted), since they are known to easily converge to suboptimal solutions. This is indeed what we observe on \fig{} \ref{fig:sample-chain-time-mean}. All versions of \jh{} eventually converge, in all runs, to the optimal policy, while on-policy updates, with or without entropic regularization, never converged to the optimal solution. Looking more closely at the different \jh{} variants, we notice that the off-policiness is instrumental to make the convergence faster, more than the exploration ratio $\epsilon_t$, because this latter cannot be set too large in order not to compromise the global performance. Indeed, we already observe with $\epsilon_t = \frac{10}{\sqrt{t}}$ that exploration prevents from converging closely to optimal in global performance. We may confirm that Dr Jekyll converges well to optimal on \fig{} \ref{fig:sample-chain-time-mean-jek}.

\begin{figure*}[t]
	\centering
	\subfloat[Mean global performance]{
		\includegraphics[trim = 5pt 5pt 5pt 5pt, clip, width=0.48\columnwidth]{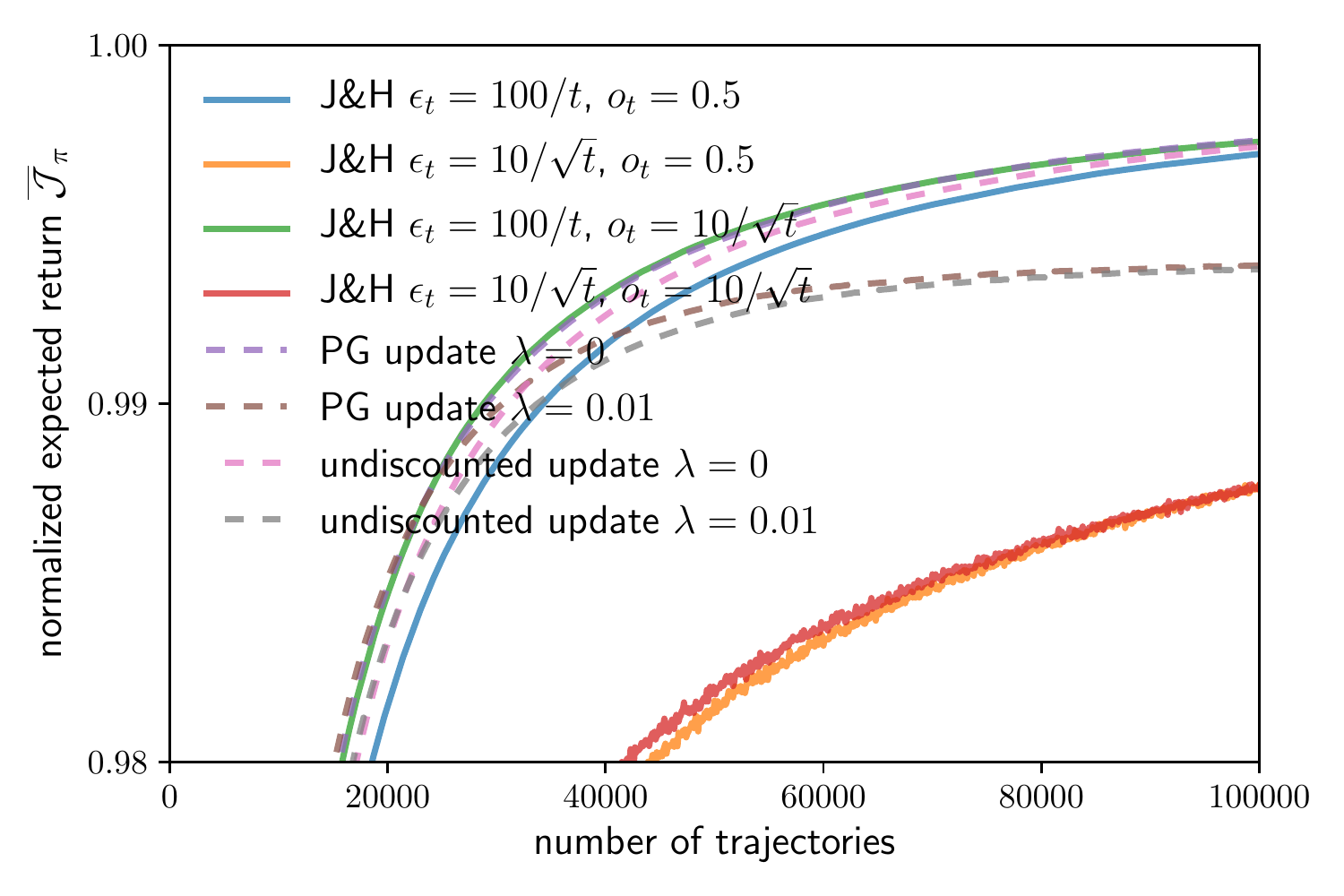}
		\label{fig:sample-garnets-time-mean}
	}
	\subfloat[Mean Jekyll performance]{
		\includegraphics[trim = 5pt 5pt 5pt 5pt, clip, width=0.48\columnwidth]{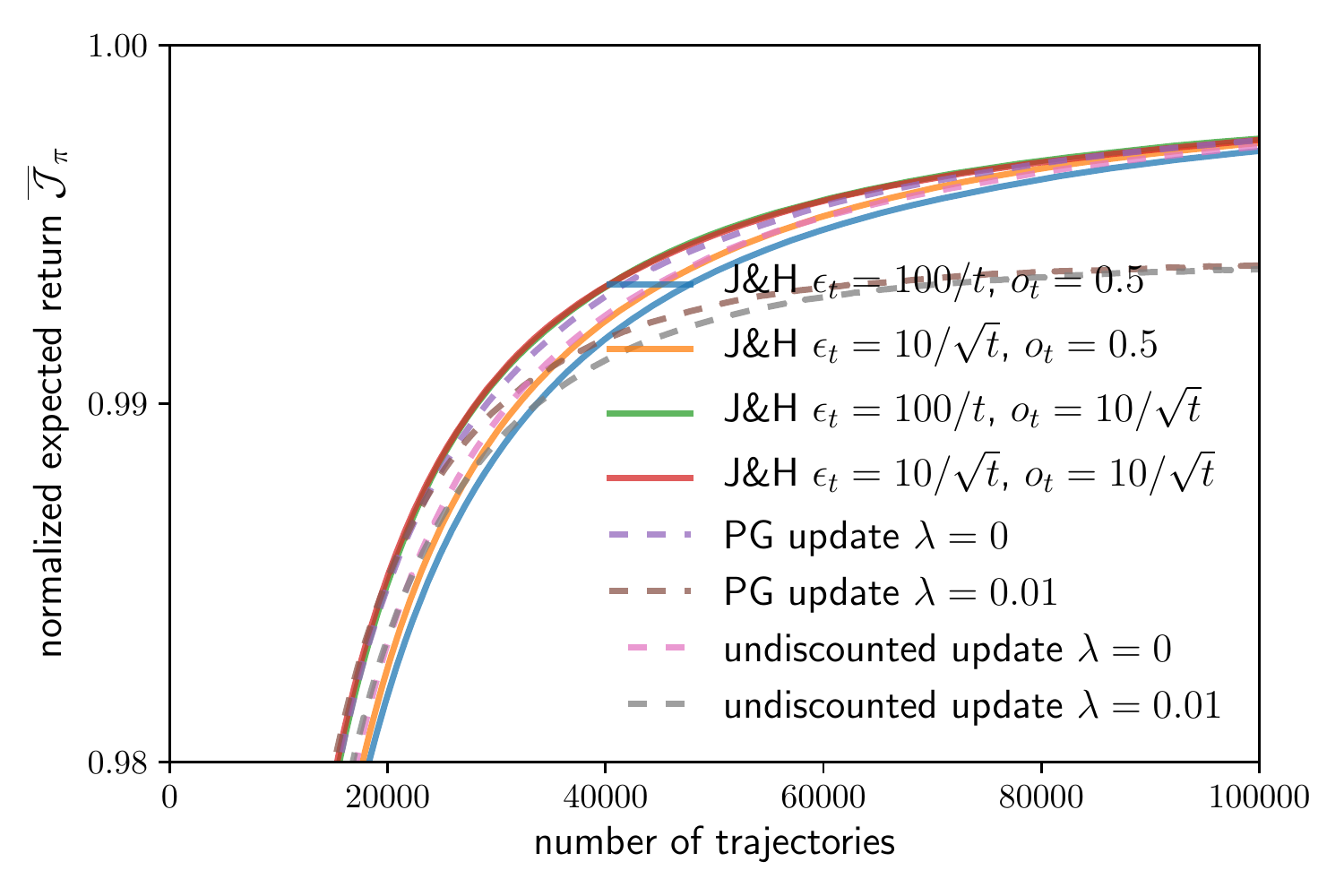}
		\label{fig:sample-garnets-time-mean-jek}
	}\\
	\subfloat[Mean global performance (long)]{
		\includegraphics[trim = 5pt 5pt 5pt 5pt, clip, width=0.48\columnwidth]{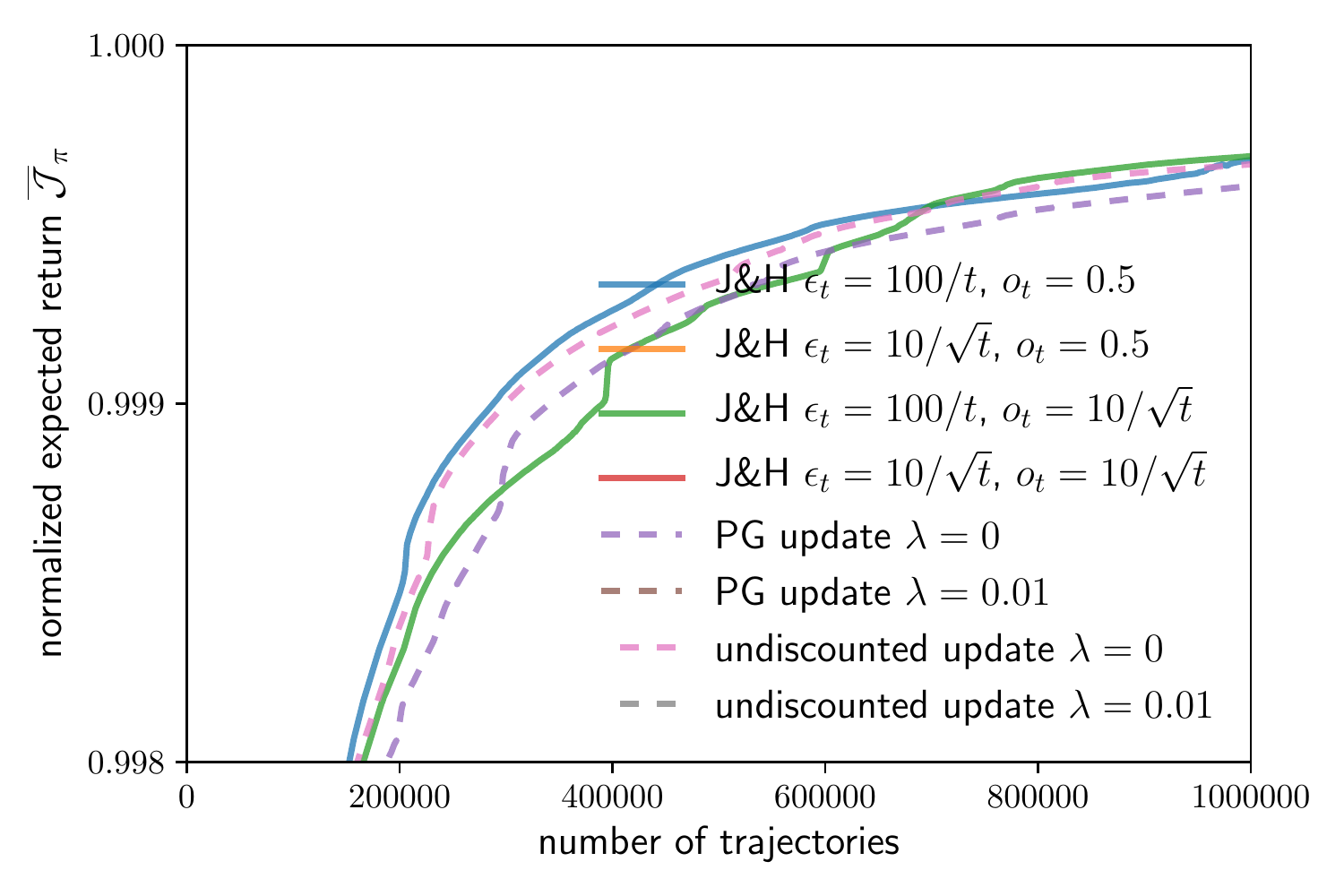}
		\label{fig:sample-garnets-long_time-mean}
	}
	\subfloat[Decile global performance (long)]{
		\includegraphics[trim = 5pt 5pt 5pt 5pt, clip, width=0.48\columnwidth]{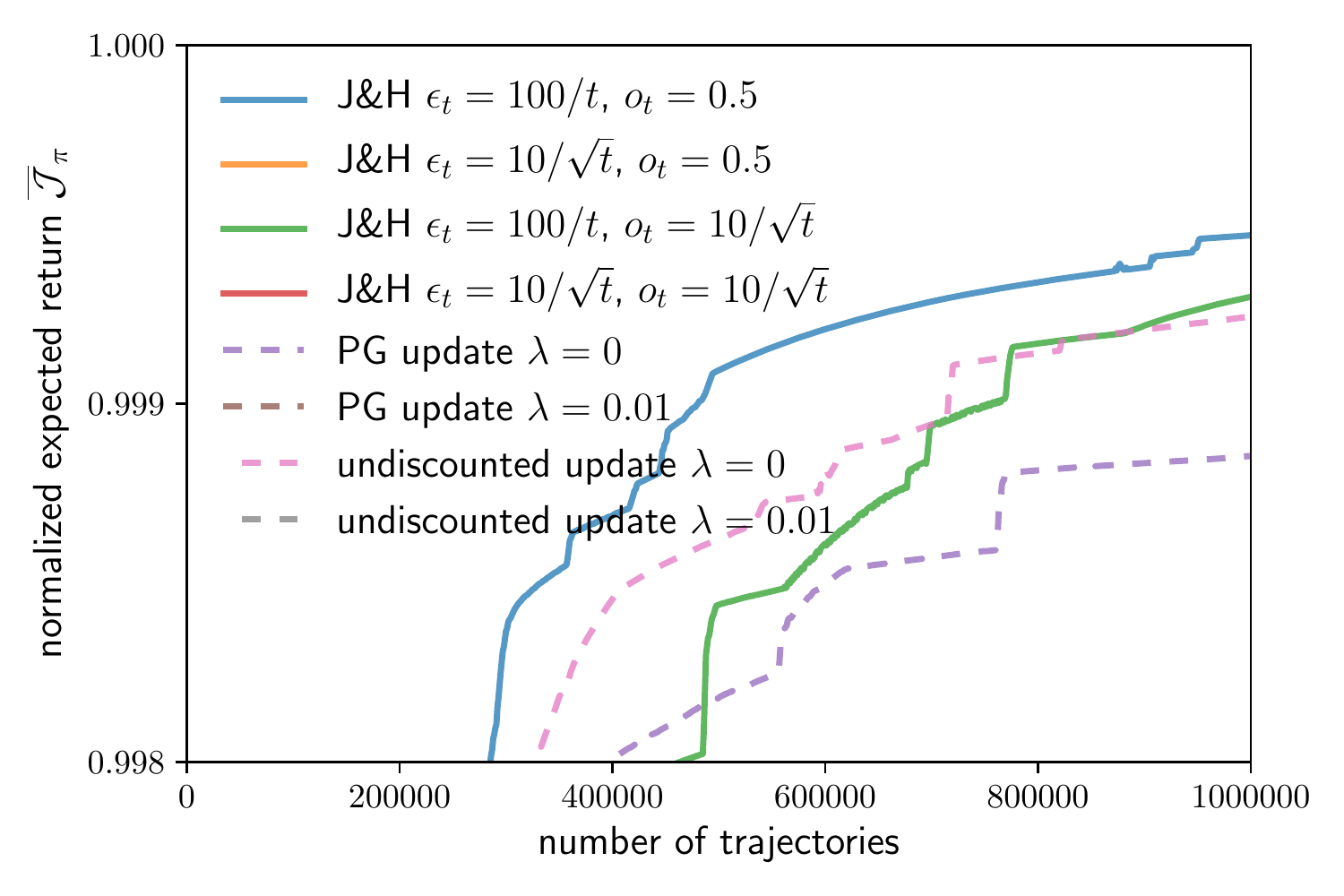}
		\label{fig:sample-garnets-long_time-decile}
	}
	\caption{Random MDPs experiments vs. number of trajectories.}
		\label{fig:sample-garnets-time}
	\vspace{-10pt}
\end{figure*}
In the random MDPs domain, we do not expect to see any improvement, the on-policy updates (policy gradient and undiscounted) are not \textit{tricked} to converge to a suboptimal solution. Moreover, the high stochasticity of the domain, makes it so that every state is visited a lot of times even when following a deterministic policy. \fig{} \ref{fig:sample-garnets-time-mean} confirms that \jh{} can keep up with the speed of the true gradient. Its versions with $\epsilon_t=\frac{10}{\sqrt{t}}$ suffer a low global performance because of its exploration that remains high, as \fig{} \ref{fig:sample-garnets-time-mean-jek} confirms. We notice that entropic regularization does not help training faster and also biases the training objective which results into a plateaued performance. Finally, we may observe that the choice of $o_t$ has a slight impact: using half Dr Jekyll, half Mr Hyde samples to train offline slows a bit the convergence in the first 100k trajectories, however, looking at a longer horizon on \fig{} \ref{fig:sample-garnets-long_time-mean} shows that it gets better later\footnote{Note that despite using more than 100 runs to compute these curves, there remains some instability and the visible results are probably not significant enough to elect a winner, however they are enough to see that they perform comparably.}. It is also interesting to notice that the gap to optimal performance has decreased by a factor 10 with 10 times more data, which corroborates our theoretical findings that the regret should be in $\mathcal{O}(t^{-1})$. Finally, the decile performance across time shown on \fig{} \ref{fig:sample-garnets-long_time-decile} suggests that the use of strong off-policiness mitigates the risk and improves the worst case performance.

\subsection{Vs. off-policiness $o_t$:}
\begin{figure*}[t]
	\centering
	\subfloat[Chain experiment]{
		\includegraphics[trim = 5pt 5pt 5pt 5pt, clip, width=0.48\columnwidth]{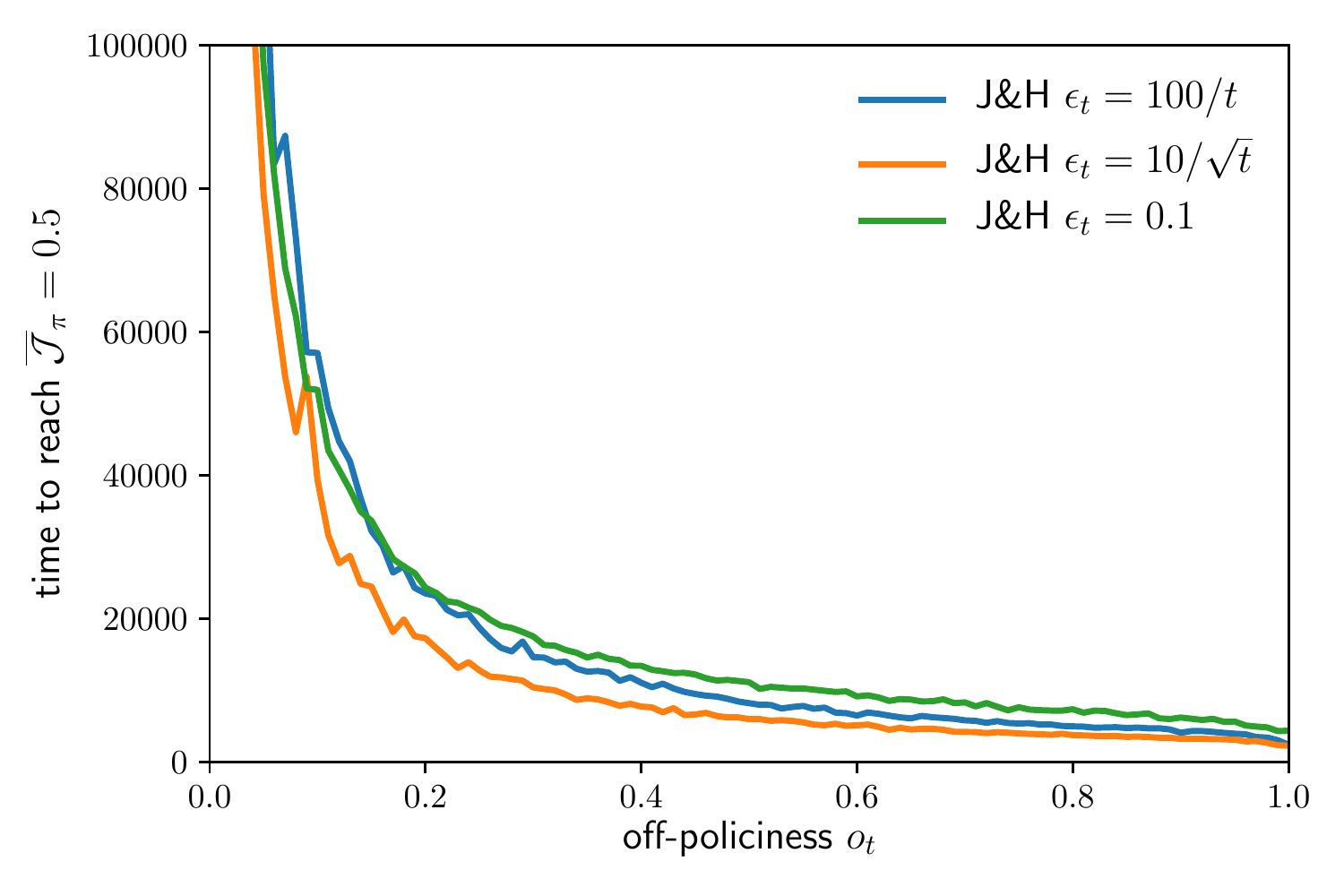}
		\label{fig:sample-chain-offpol}
	}
	\subfloat[Random MDPs experiment]{
		\includegraphics[trim = 5pt 5pt 5pt 5pt, clip, width=0.48\columnwidth]{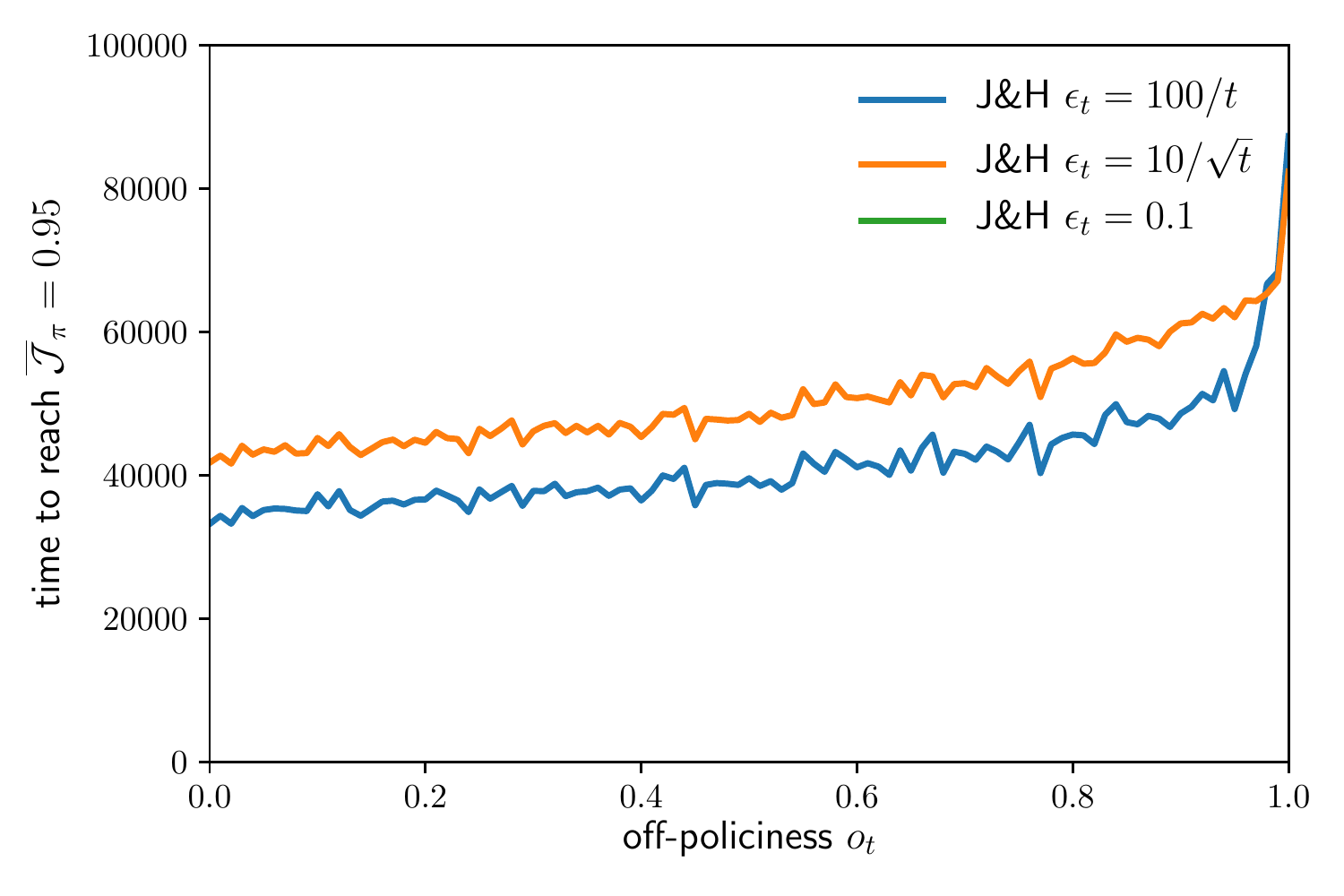}
		\label{fig:sample-garnets-offpol}
	}
	\caption{Experiments vs. off-policiness $o_t$: average time to reach a global performance objective.}
		\label{fig:sample-offpol}
	\vspace{-10pt}
\end{figure*}
Off-policiness being a hyperparameter of \jh{} alone, we report the results for \jh{} only. On \fig{} \ref{fig:sample-chain-offpol}, we observe that the higher the off-policiness, the best better it is in the chain experiment. However, it is quite the opposite in the random MDPs experiments results shown on \fig{} \ref{fig:sample-garnets-offpol}, even though it is less critical. The scheduling of $o_t$ would be an interesting focus for further research. At this point of our knowledge we conjecture that it is better to train on-policy when the expected value improvement are high, and train more off-policy when it is small.

\subsection{Vs. learning rates $\eta$ and $\eta_c$:}
\begin{figure*}[t]
	\centering
	\subfloat[Chain experiment]{
		\includegraphics[trim = 5pt 5pt 5pt 5pt, clip, width=0.48\columnwidth]{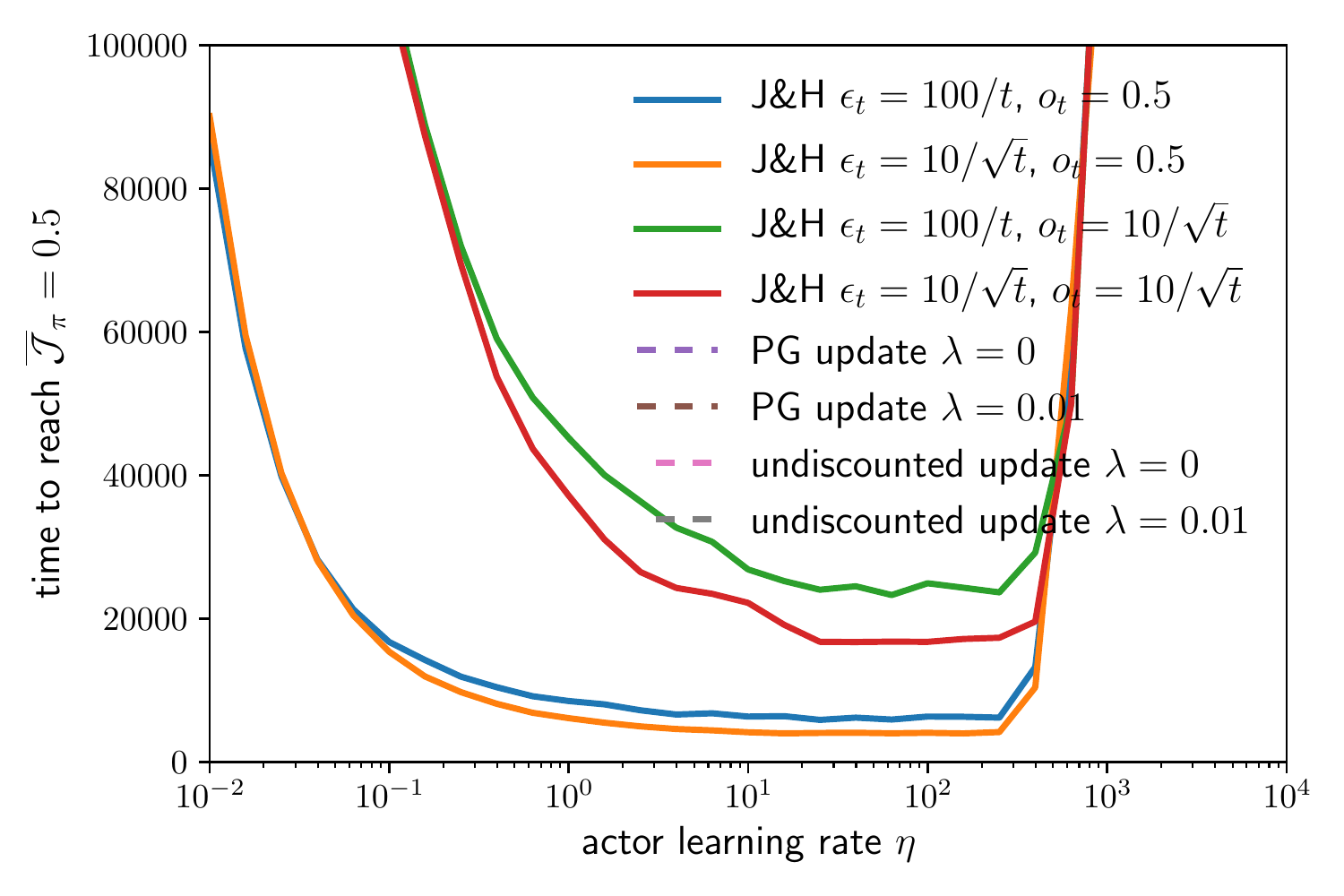}
		\label{fig:sample-chain-learningrate}
	}
	\subfloat[Random MDPs experiment]{
		\includegraphics[trim = 5pt 5pt 5pt 5pt, clip, width=0.48\columnwidth]{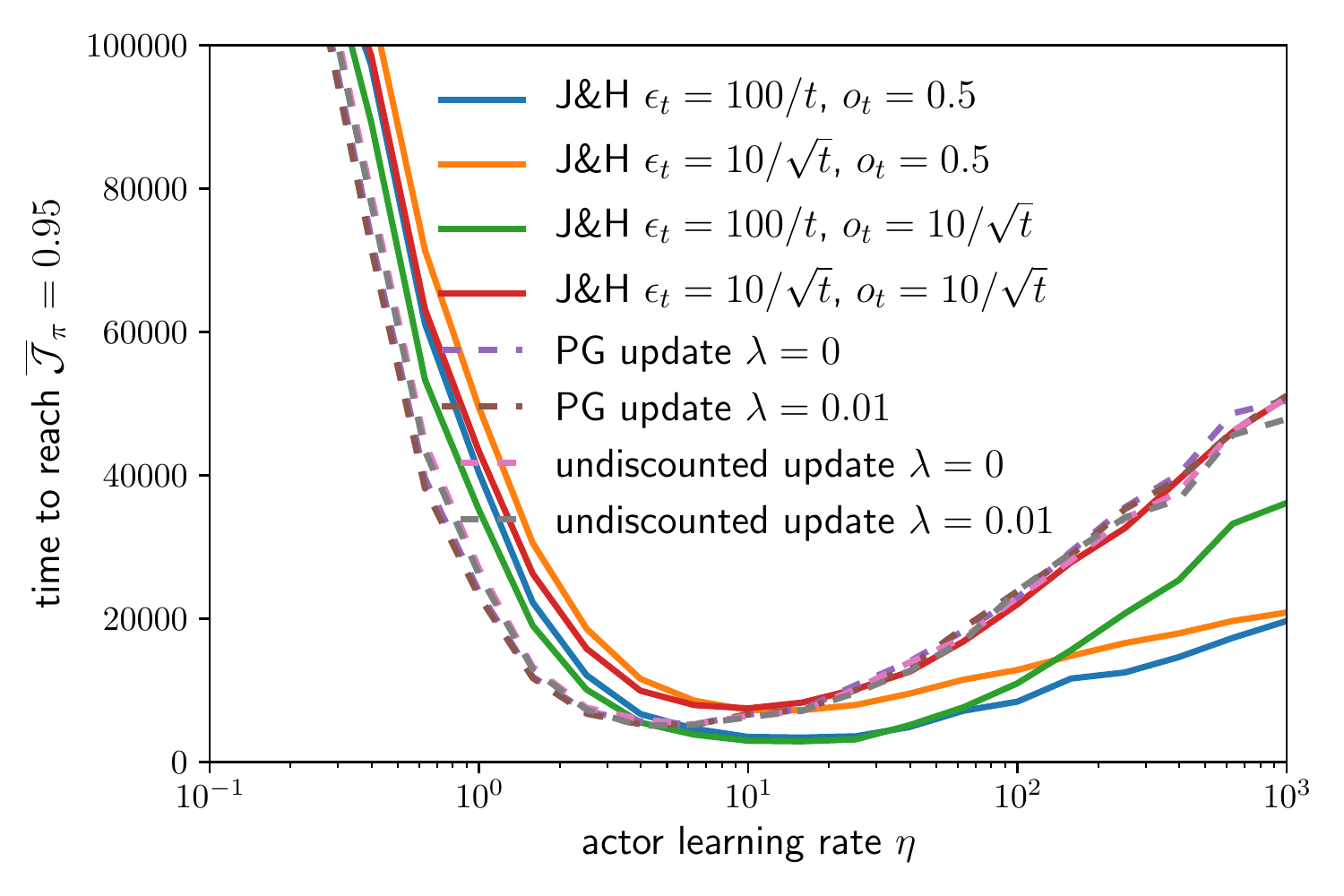}
		\label{fig:sample-garnets-learningrate}
	}
	\caption{Experiments vs. actor learning rate $\eta$: average time to reach a global performance objective.}
		\label{fig:sample-learningrate}
	\vspace{-10pt}
\end{figure*}
 The study of $\eta$ on the chain domain (results on \fig{} \ref{fig:sample-chain-learningrate} reveals that the performance of \jh{} improves until some value where it breaks. It looks like a numerical stability issue but after checking it does not seem like anything of the sort, rather that the softmax parameters go faster away than they can go back. Indeed, as our theoretical analysis shows, there is a factor in $\pi(1-\pi)$ the updates that slows down the recovery from a run that converged too far into a suboptimal policy. Actually, this phenomenon has been observed before on policy gradients, \textit{e.g.} on Fig. 1.d of \cite{mei2020global}. Note that the value of this hyperparameter does not allow to help the on-policy algorithms to find the optimal solution. We made the similar experiment on easier settings (lower $\beta$ and size), and this did not improve (not shown). Note also that the default value for the actor learning rate has been set to 1 which is not its optimal value, but it was a better value to make the critic/actor learning rates consistent with each other.
 
 The experiment on the Random MDPs shown on \fig{} \ref{fig:sample-garnets-learningrate} is one of the most interesting empirical results for many reasons\footnote{We notice that the default value we used for this hyperparameter is not favorable for \jh{}.}. Firstly, we observe that all updates admit an optimal value for the actor learning rate in the range [2,30]: rather high, but not too high. This shows that there is a value for not updating too fast the policy. Secondly, off-policiness is more tolerant for higher $\eta$ values: its optimal value is obtained for $\eta\approx 20$, while on-policy updates are best with $\eta\approx 5$. Thirdly, when $\eta$ is optimally set for each update, \jh{} significantly outperforms on-policy updates ($\approx$ 2,500 vs. 3,500). Finally, it is satisfying to observe that $\eta\approx 20$ is a also optimal in the chain experiment, suggesting that it may not be too much domain dependent. 
 
\begin{figure*}[t]
	\centering
	\subfloat[Chain experiment]{
		\includegraphics[trim = 5pt 5pt 5pt 5pt, clip, width=0.48\columnwidth]{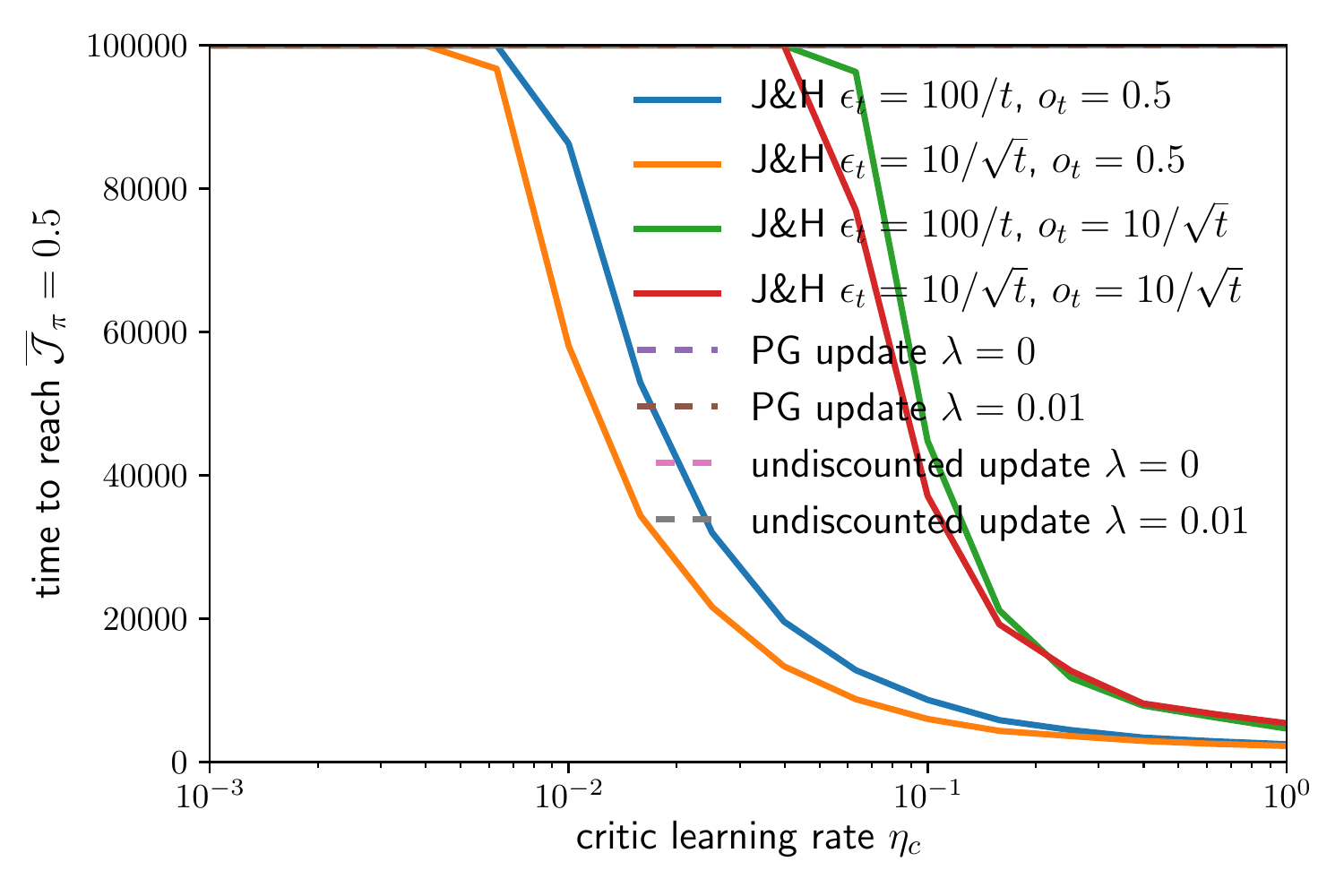}
		\label{fig:sample-chain-criticlr}
	}
	\subfloat[Random MDPs experiment]{
		\includegraphics[trim = 5pt 5pt 5pt 5pt, clip, width=0.48\columnwidth]{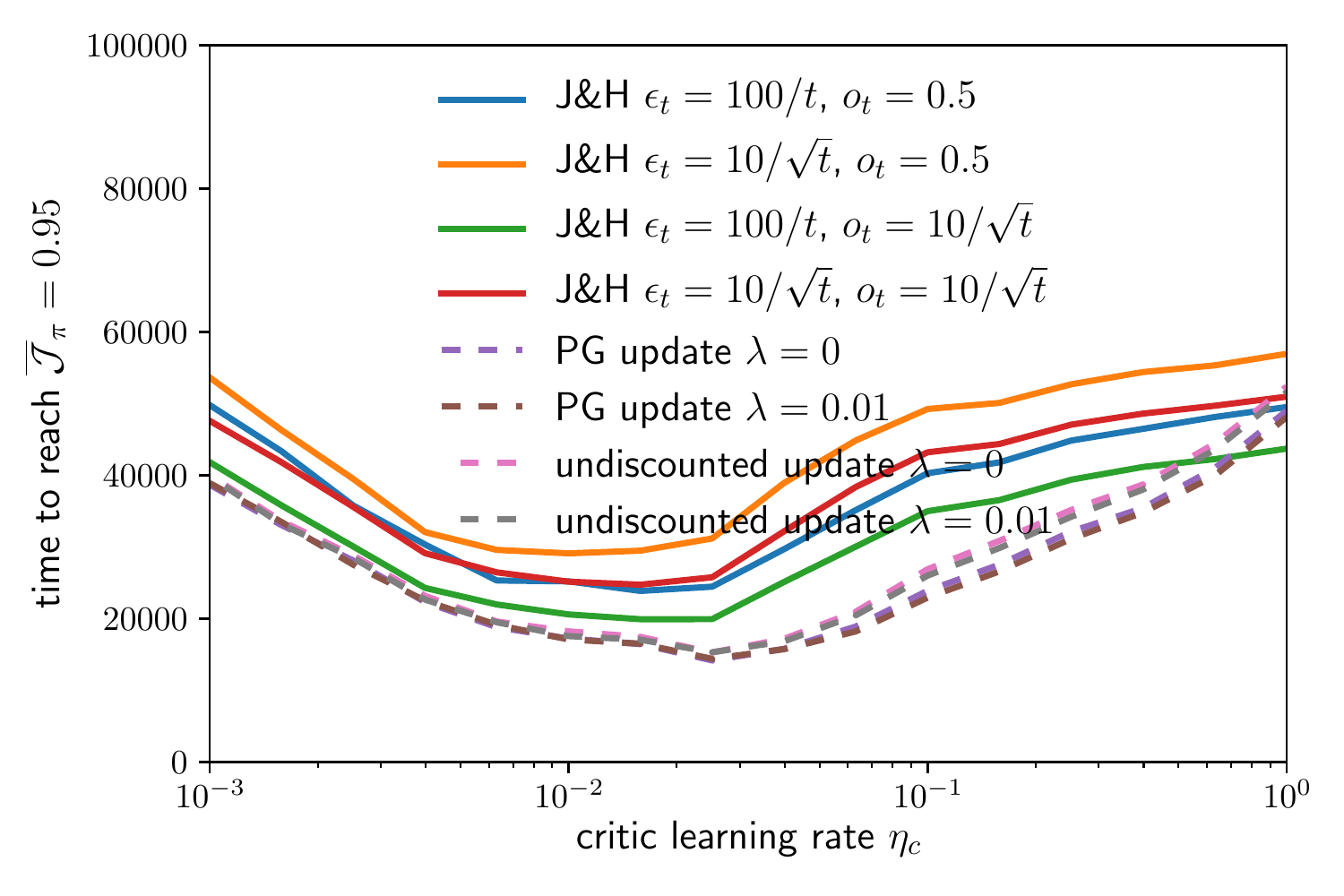}
		\label{fig:sample-garnets-criticlr}
	}
	\caption{Experiments vs. critic learning rate $\eta_c$: average time to reach a global performance objective.}
		\label{fig:sample-criticlr}
	\vspace{-10pt}
\end{figure*}
 Since the chain domain is deterministic, we would expect a high critic learning rate $\eta_c$ to be better and this is what we observe, but even a maximal critic learning rate of 1 is not enough for the on-policy updates with or without policy entropy regularization to solve the chain is the default settings (size of 10 and $\beta=0.8$). We also not that \jh{} is much more tolerant to a small critic learning rate when the off-policiness is kept high. The random MDP experiment (\fig{} \ref{fig:sample-garnets-criticlr}) illustrates the risk of setting a too high critic learning rate. A small learning rate implies a longer but steadier time to converge, while a high learning rate might compromise the convergence at all.
 
\begin{figure*}[t]
	\centering
	\subfloat[Mean performance]{
		\includegraphics[trim = 5pt 5pt 5pt 5pt, clip, width=0.48\columnwidth]{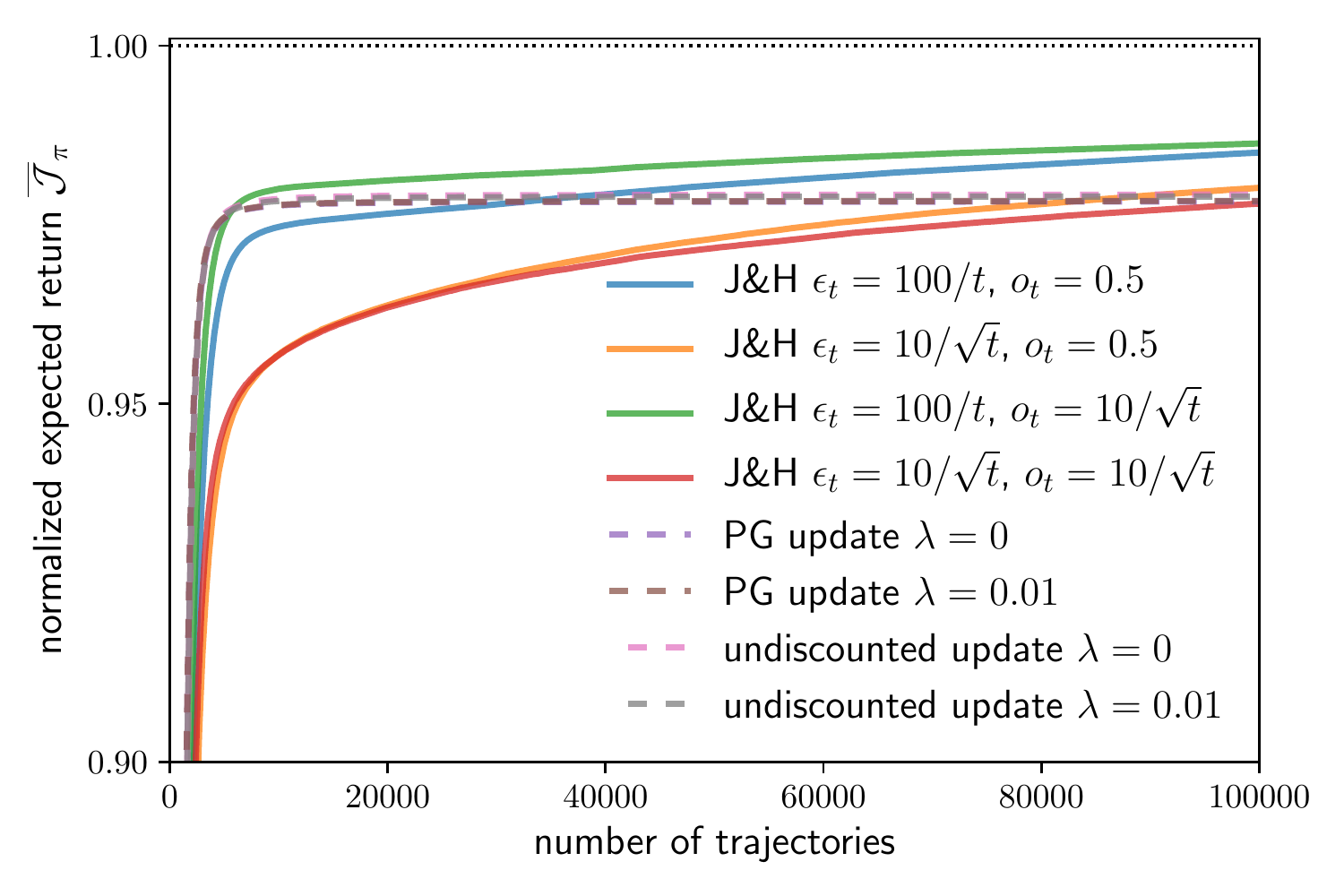}
		\label{fig:opt_params_m}
	}
	\subfloat[Decile performance]{
		\includegraphics[trim = 5pt 5pt 5pt 5pt, clip, width=0.48\columnwidth]{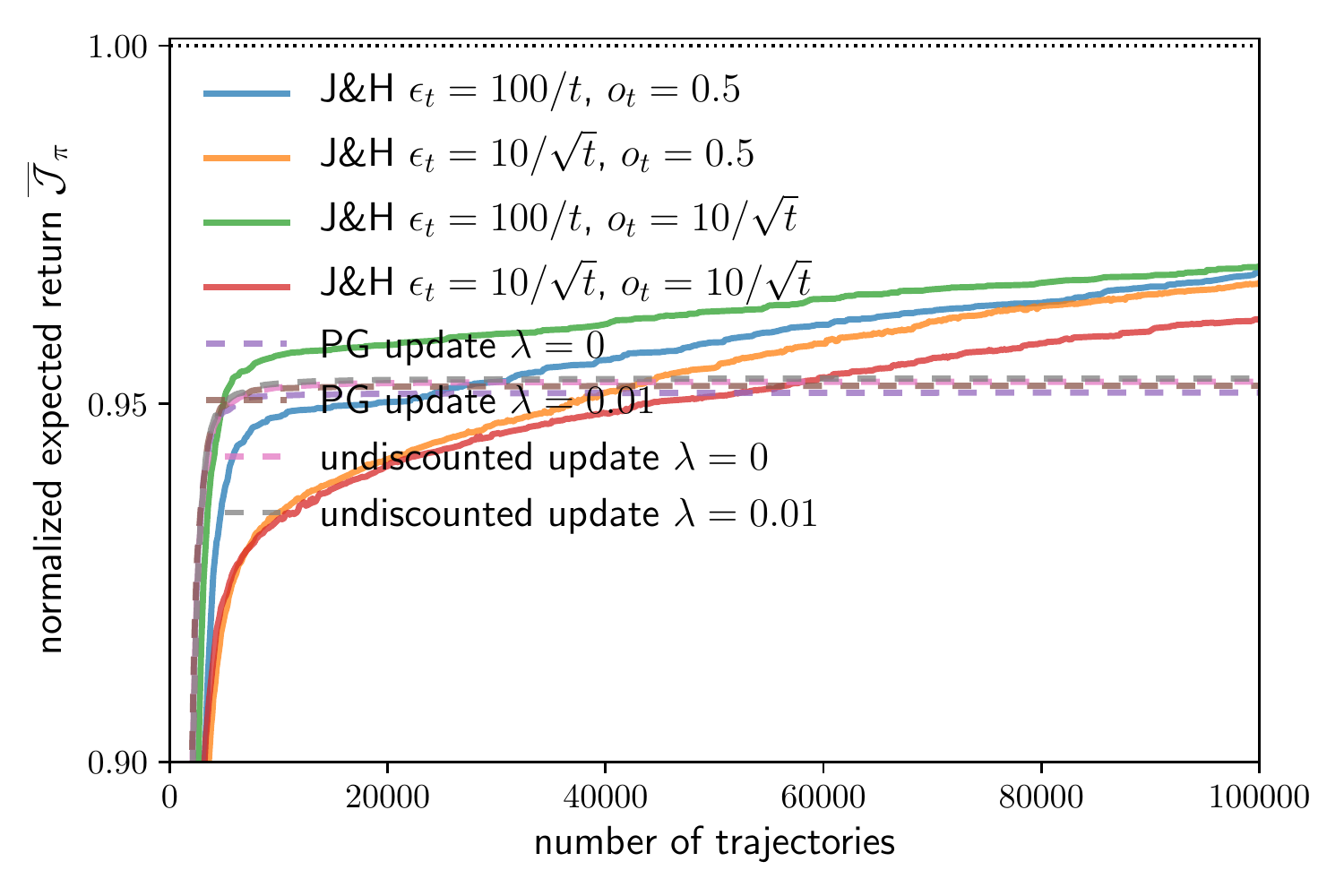}
		\label{fig:opt_params_d}
	}
	\caption{Random MDPs experiments with $\eta=10$ and $\eta_c=0.025$.}
		\label{fig:opt_params}
	\vspace{-10pt}
\end{figure*}
 One must keep in mind that these insights are for reaching a normalized performance of 0.95 in the random MDPs experiments. We tried a run with $\eta=10$ and $\eta_c=0.025$ for all algorithms, since they seem to all perform well with this setting in our hyperparameter search we have just presented. \fig{} \ref{fig:opt_params} displays the results. It clearly appears that reaching $\overline{\mathcal{J}}_\pi=0.99$ (let alone 0.999) would take much longer than with our default setting. We judged better to stick with our default settings for this reason.

\subsection{Vs. parameters of the chain domain:}
On \fig{} \ref{fig:sample-chain-states}, we vary the size of the chain, and observe that \jh{} algorithm scales well to it as long as the off-policiness is maintained long enough for it to unlearn to suboptimal policy and move on to the optimal policy. It is worth noting that the on-policy updates fail at finding even with its easiest hyperparameters, indeed the $q$-function being initialized to 0 and the critic learning rate being of 0.1 makes it that the algorithms already converged to the suboptimal policy when the later states values are properly estimated. In fact, the policy entropy regularized versions of the on-policy updates reach $\overline{\mathcal{J}}_\pi=0.48$ around the 150,000\textsuperscript{th} episode.

On \fig{} \ref{fig:sample-chain-states}, we vary the performance ratio $\beta$ and observe that the on-policy updates manage to make it work even for a chain of size 10 when $\beta$ is sufficiently small, and that the policy entropy regularization helps a bit to push further the limit, but not by much. In general, the delay occurred by the time necessary for the critic to be trained is fatal to these algorithms, hence the importance of prioritizing the replay of exploratory experiences. \fig{} \ref{fig:sample-chain-states} also validates the reliability of the \jh{} algorithms, in particular with a strong off-policiness.
\begin{figure*}[t]
	\centering
	\subfloat[vs. chain size]{
		\includegraphics[trim = 5pt 5pt 5pt 5pt, clip, width=0.48\columnwidth]{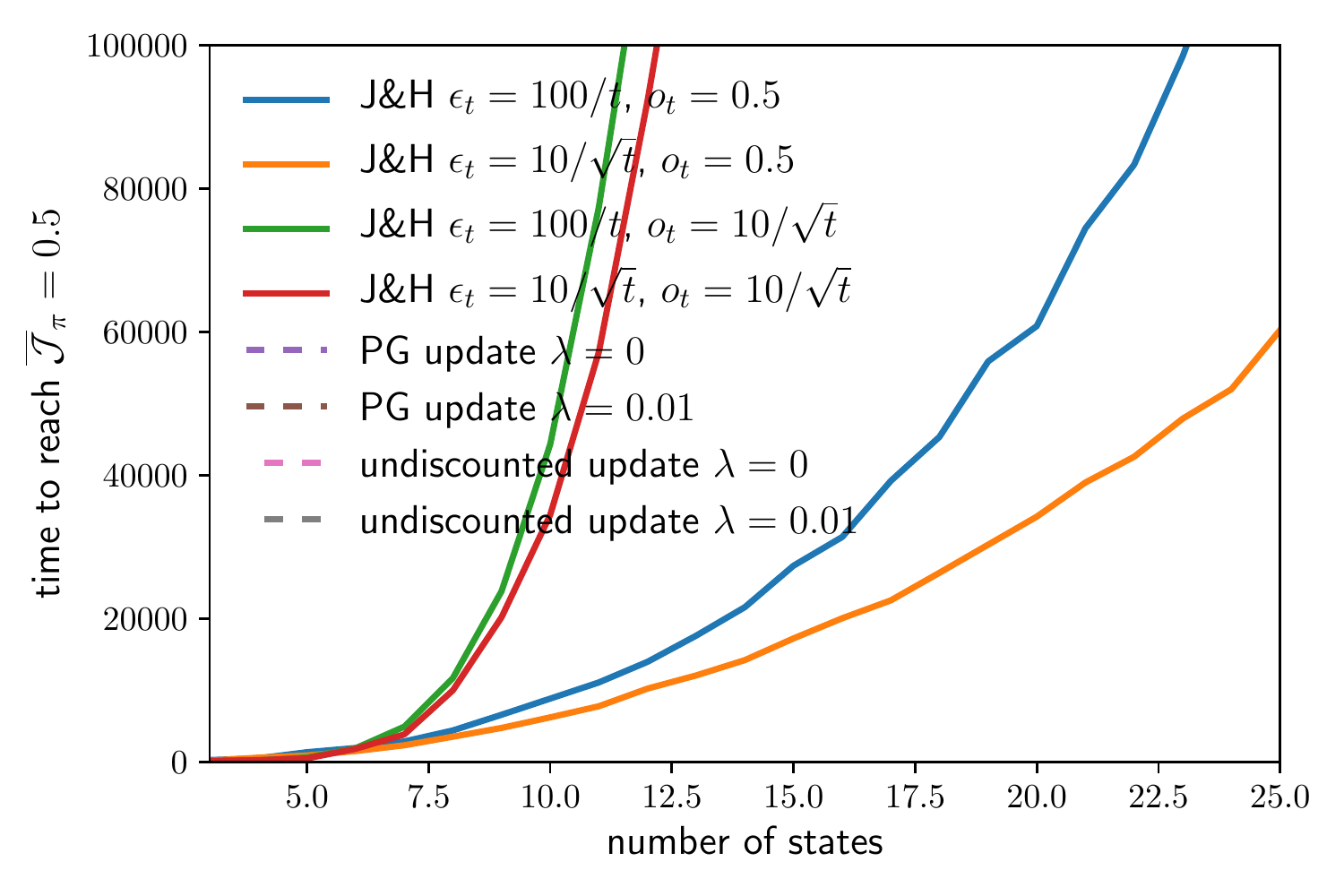}
		\label{fig:sample-chain-states}
	}
	\subfloat[vs. performance ratio $\beta$]{
		\includegraphics[trim = 5pt 5pt 5pt 5pt, clip, width=0.48\columnwidth]{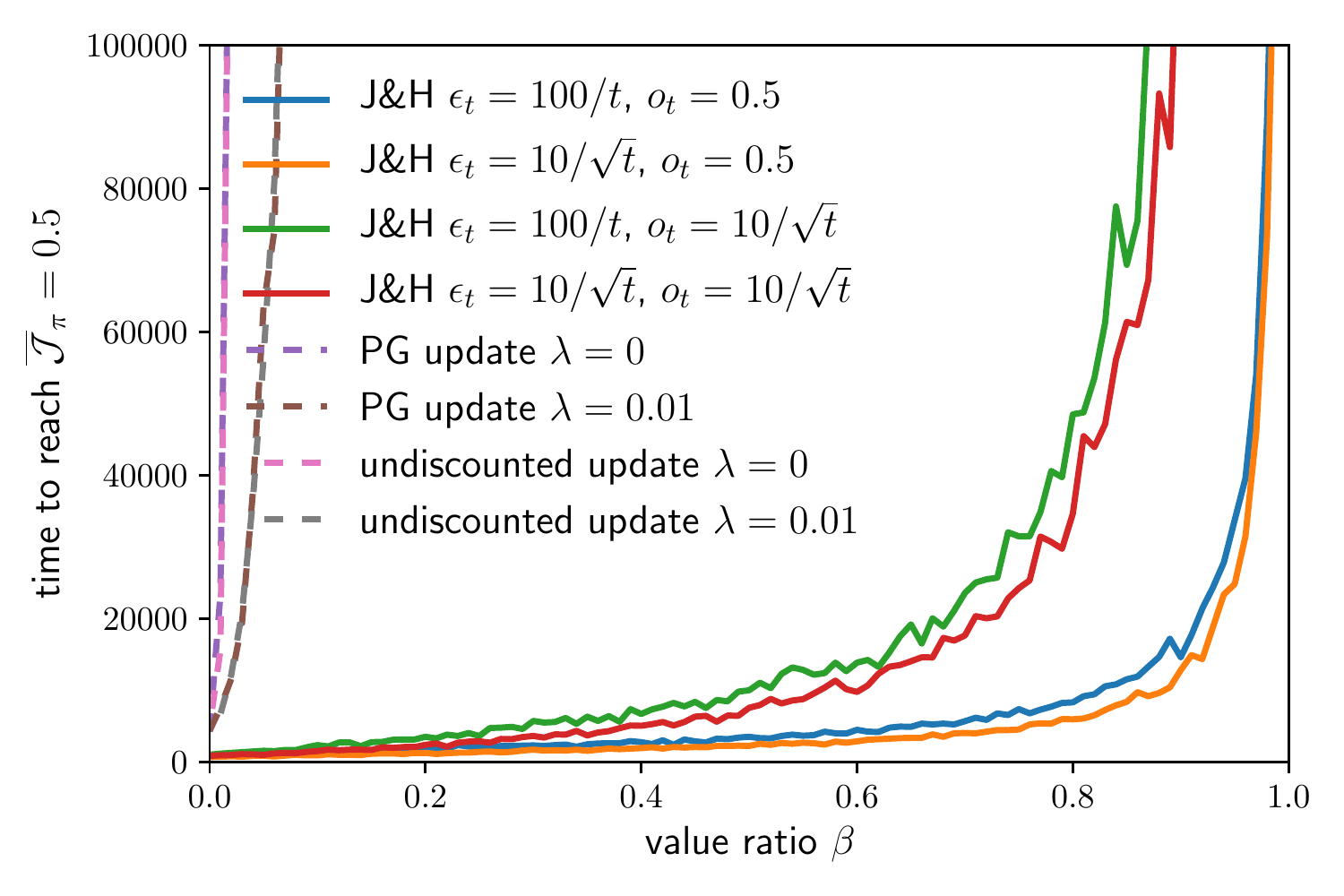}
		\label{fig:sample-chain-ratio}
	}
	\caption{Experiments vs. chain hyperparameters: average time to reach a global performance objective.}
		\label{fig:sample-chain-hyper}
	\vspace{-10pt}
\end{figure*}

\subsection{Vs. $q$-function initialization $q_0$:}
Initializing the $q$-function to high values has been a historic exploration scheme for $q$-learning, and we observe on \fig{} \ref{fig:sample-chain-criticlr} that is also works in the actor-critic world. As soon as $q_0$ gets larger than the value of the low-hanging-fruit policy, all updates push to explore until discovering the end of the chain. In contrast, \jh{} is not as much sensitive to the critic initialization, in particular when the off-policiness is high enough to fight against a hostile one.
\begin{figure*}[t]
	\centering
	\subfloat[Chain experiment]{
		\includegraphics[trim = 5pt 5pt 5pt 5pt, clip, width=0.48\columnwidth]{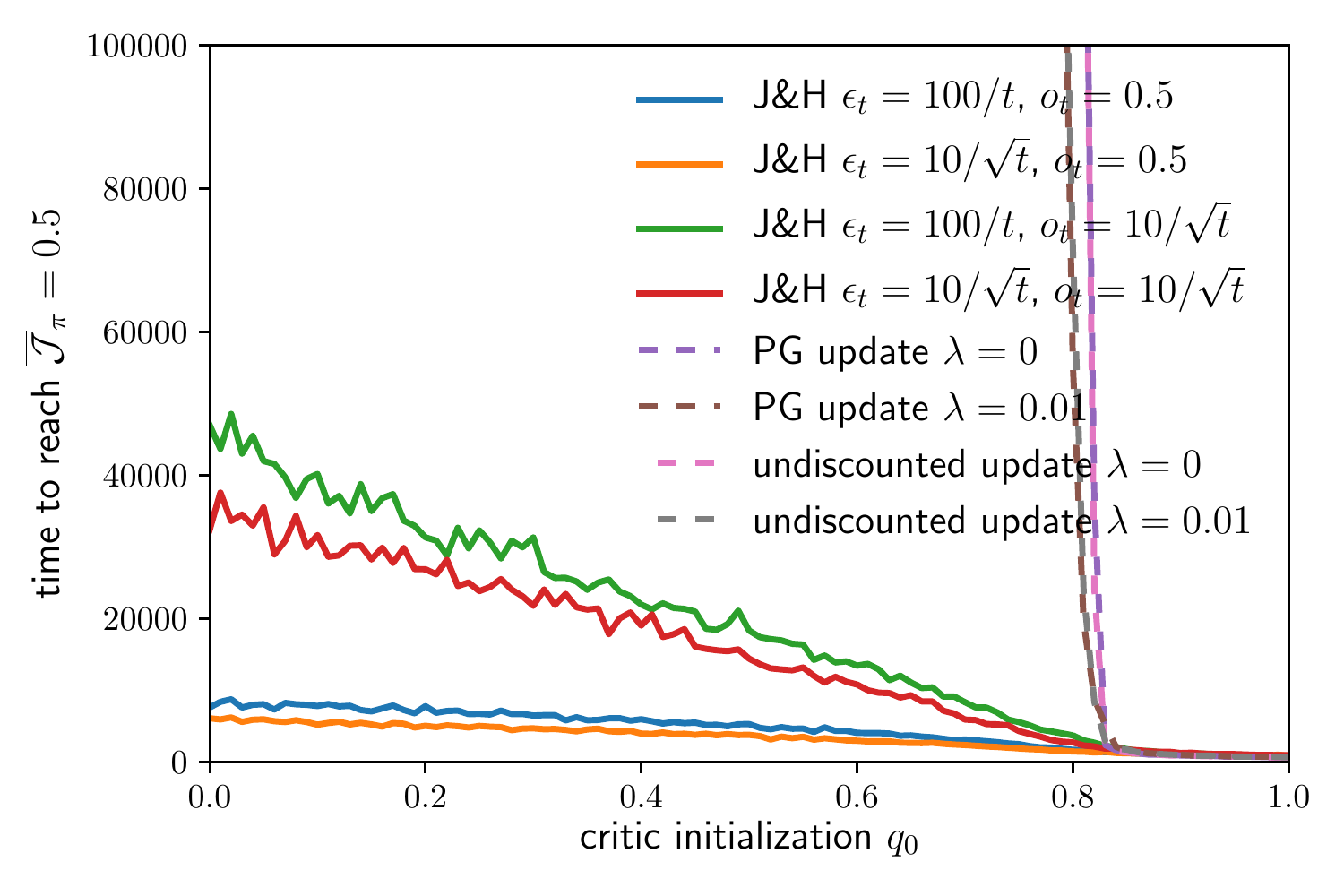}
		\label{fig:sample-chain-qinit}
	}
	\subfloat[Random MDPs experiment]{
		\includegraphics[trim = 5pt 5pt 5pt 5pt, clip, width=0.48\columnwidth]{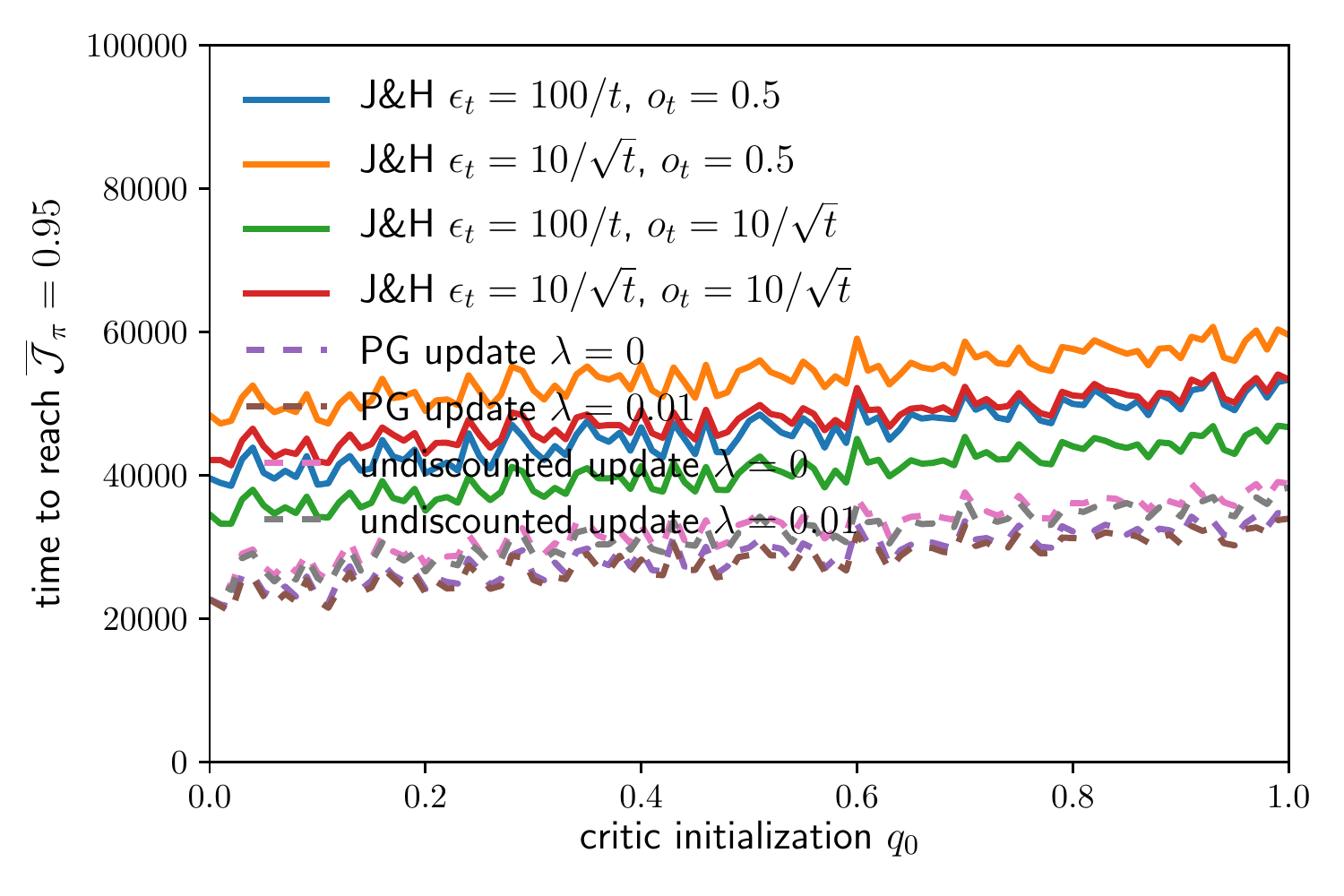}
		\label{fig:sample-garnets-qinit}
	}
	\caption{Experiments vs. $q$-function initialization $q_0$: average time to reach a global performance objective.}
		\label{fig:sample-qinit}
	\vspace{-10pt}
\end{figure*}

\subsection{Vs. exploration schedule $\epsilon_t$:}
Exploration schedule being a hyperparameter of \jh{} alone, we report the results for \jh{} only. In order to remove the dependence of the performance on its exploratory behaviour, contrarily to all the other experiments where we report the performance of \jh{}'s dyad, we report here the time for Jekyll's policy to reach the normalized value.

On \fig{} \ref{fig:sample-chain-epsilon}, we observe that having a strong off-policiness alleviates the lack of exploration to a certain extent. We however want to reiterate that the chain domain is deterministic, and as such, performing off-policy updates is equivalent to exploring. In a stochastic environment such as the random MDPs, exploration cannot be replaced with off-policy updates, because statistical significance is required to find the optimal policy. With $\epsilon_t=0$, unsurprisingly, the algorithm fails at finding the optimal reward and therefore finding the goal.
From $\epsilon_t=0.05$ to $\epsilon_t=0.6$, the speed improves from 15000 to 6500. From $\epsilon_t=0.6$ to $\epsilon_t=0.95$, the speed remains more of less constant in the range [6400, 6500]. With $\epsilon_t=1$, more surprisingly, the performance degrades quite significantly to 7300.

The random MDPs experiment (\fig{} \ref{fig:sample-garnets-epsilon}) displays similar behaviour at a different scale. From $\epsilon_t=0$ to $\epsilon_t=0.95$, the speed linearly improves from 35000 to 29000. With $\epsilon_t=1$, again surprisingly, the performance degrades catastrophically to more than 100000.

We analyse the combined results as follows: introducing Hyde's exploratory behavior generates a large performance improvement due to a need for exploration. This need is correlated with the planning difficulty of the task, hence the much stronger effect in the chain experiment. Performance then slowly improves as a function of $\epsilon_t$: this is where the statistical significance matters. Finally, when $\epsilon_t$ equals 1, the on-policy updates do not exist anymore, slowing down significantly the convergence speed.

\begin{figure*}[t]
	\centering
	\subfloat[Chain experiment]{
		\includegraphics[trim = 5pt 5pt 5pt 5pt, clip, width=0.48\columnwidth]{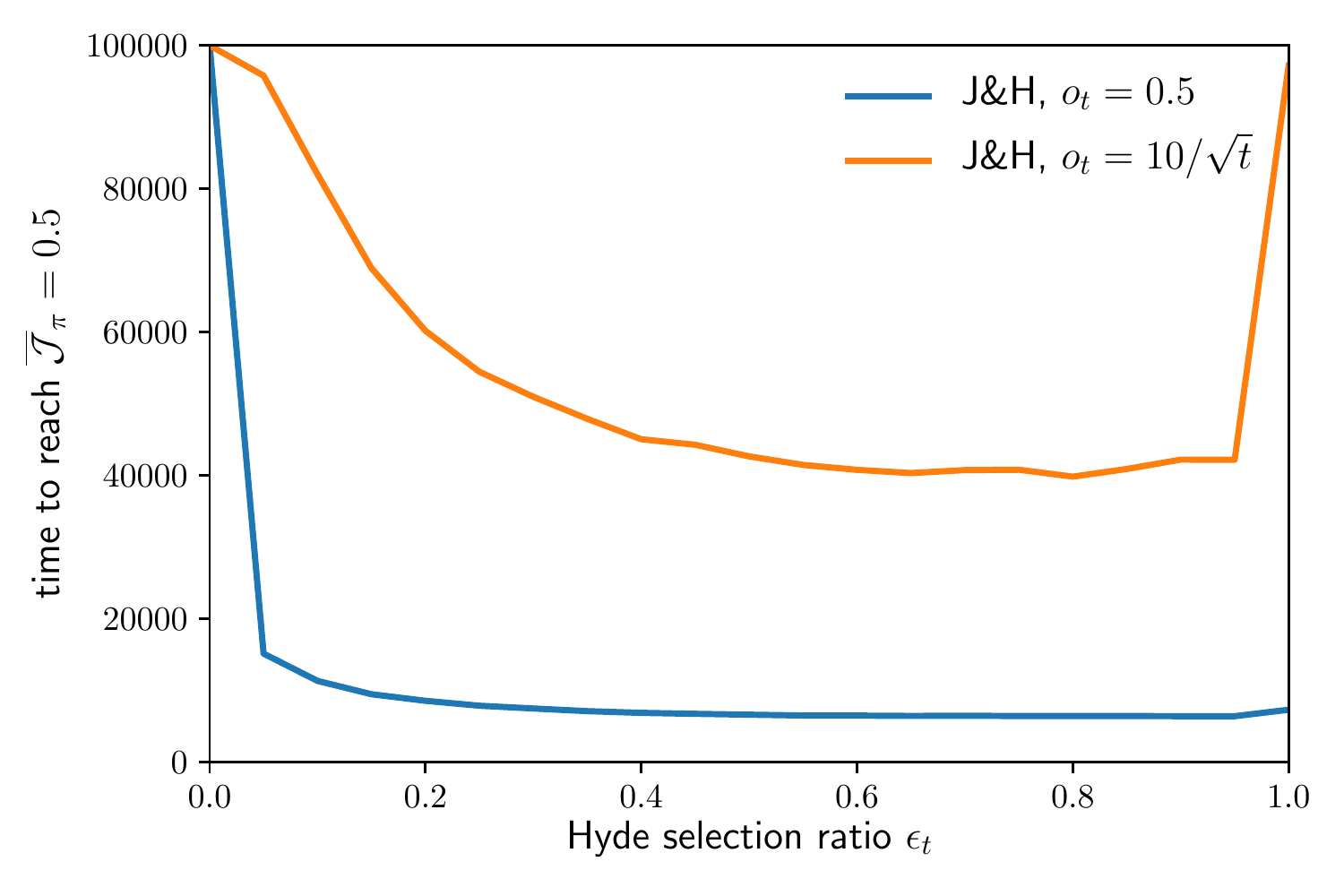}
		\label{fig:sample-chain-epsilon}
	}
	\subfloat[Random MDPs experiment]{
		\includegraphics[trim = 5pt 5pt 5pt 5pt, clip, width=0.48\columnwidth]{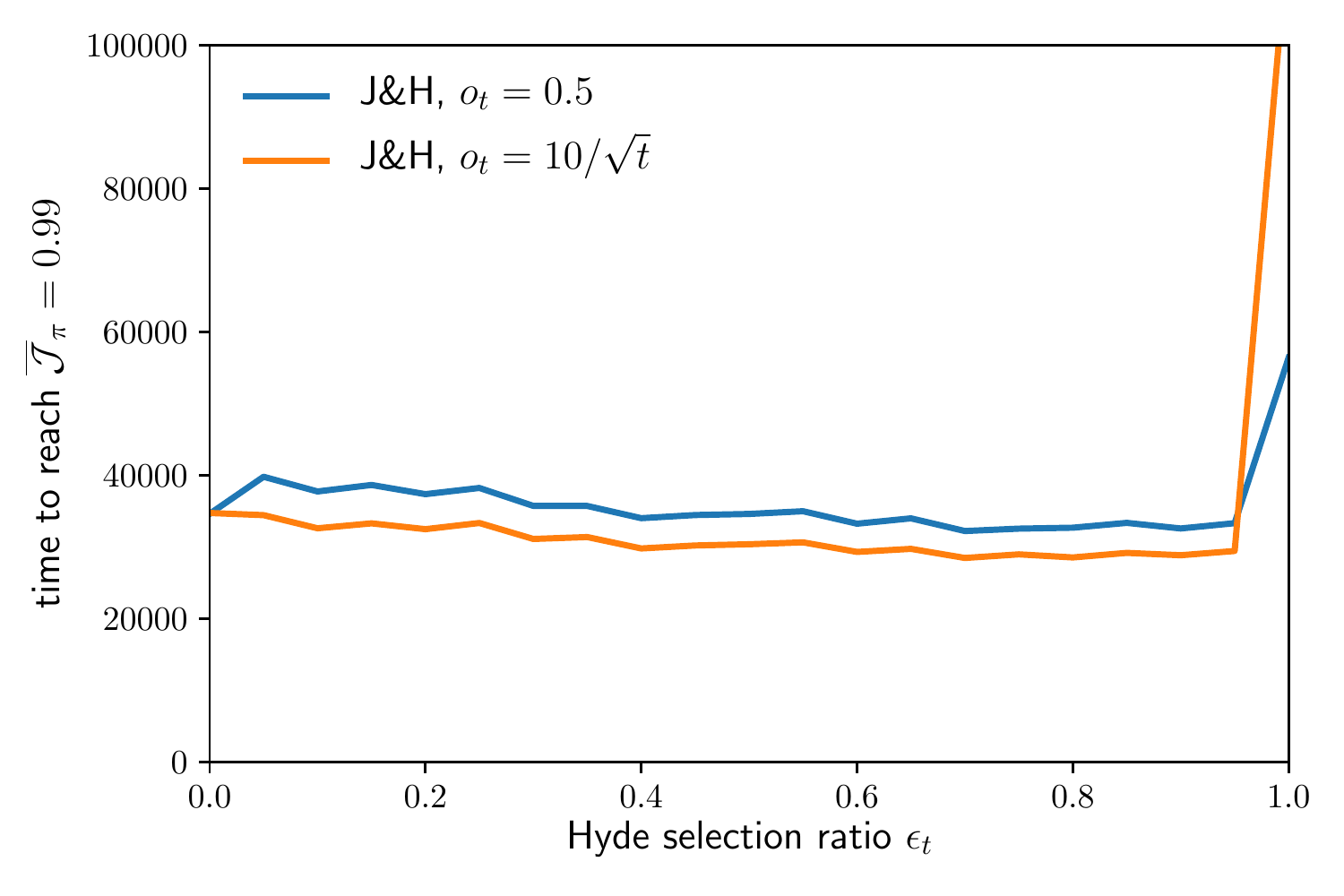}
		\label{fig:sample-garnets-epsilon}
	}
	\caption{Experiments vs. exploration schedule $\epsilon_t$: average time to reach a global performance objective.}
		\label{fig:sample-epsilon}
	\vspace{-10pt}
\end{figure*}

\subsection{General conclusion for this set of experiments}
We test the performance across time, actor $\eta$ and critic $\eta_c$ learning rates, MDP parameters $|\mathcal{S}|$ and $\beta$, off-policiness $o_t$ and exploration $\epsilon_t$, and critic initialization $q_0$ with the softmax parametrization on both the chain and random MDPs. The full report is available in Appendix \ref{app:expes-sample}. 

In the chain experiments, we observe that \jh{} shows the same ability to converge to the optimal policy, while the on-policy updates, with or without policy entropy regularization, always fail, except when $q_0>\mathcal{J}_{\mybot}$ because it induces an exploratory behaviour that solves its limits. In contrast, \jh{} is efficient whatever $q_0$. More specifically about \jh{}, while the exploration $\epsilon_t$ is necessary at a small amount, the amplitude of off-policiness $o_t$ is what matters the most. Since the chain domain is deterministic, the critic learning rate $\eta_c$ is optimal set at 1. Regarding the actor learning rate, it is best set high too, but until some limit where it breaks the convergence.

Regarding the random MDPs experiments, the results are much closer, the on-policy updates being helped by the high stochasticity of the environment, providing indeed a sort of exploration. We still observe that all algorithms provide similar performance, \jh{} performing slightly worse when $o_t$ or $\epsilon_t$ increase. Similarly, since exploration is not necessary in this domain, the performance slightly decrease with $q_0$ for all algorithms. Regarding the critic learning rate, it is best chosen around $\eta_c=0.025$. Maybe the most interesting result concern the dependency with respect to the actor learning rate. It seems that the off-policiness granted by \jh{} updates allows for higher $\eta$, which allows faster convergence when this hyperparameter is optimized independently for each algorithm. While expected, another interesting observation is that off-policiness seems to improve the worst case scenarios by allowing to recover more easily from converging to a suboptimal policy.

\section{Full report of deep reinforcement learning experiments}
\label{app:expes-deep}

\subsection{The Four Rooms environment}

\begin{figure*}[h]
	\centering
	\includegraphics[trim = 5pt 5pt 5pt 5pt, clip, width=0.48\columnwidth]{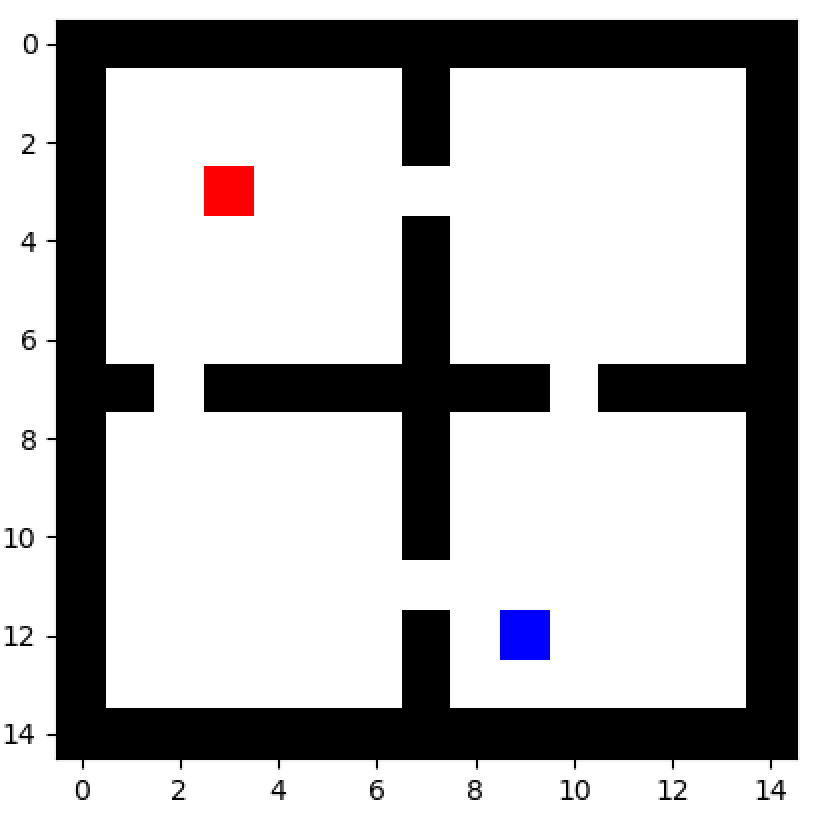}
	\caption{The Four Rooms environment, the agent (in blue) needs to navigate to the goal (in red).}
	\label{fig:four_rooms_env}
\end{figure*}

Figure~\ref{fig:four_rooms_env} shows an example of the environment. The agent (in blue) starts at a random place in one of the rooms. In Level 1, the initial state can be in any room, in Level 2, the initial state cannot be in the goal room. The agent then needs to navigate to the goal, using four actions, North, East, West, South. At each step in the environment, the agent receives a negative reward of $-0.1$. When reaching the goal, the agent receives a reward of $90$ and the episode terminates. If the agent hasn't reached the goal in 90 steps, the episode terminates. In the plots, we report the undiscounted sum of rewards, a quantity that belongs to $[-9, 90]$. The state consists of a random 20-dimensional embedding for each distinct position in the environment. The discount factor is set to $0.9$. After 10 episodes of training, we evaluate each agent during 10 episodes and report the average performance. Each agent was evaluated on 10 seeds, we report the average and one standard deviation over those 10 seeds. We trained the agents on an internal GPU cluster of P40 cards. Each training run took approximately 30 minutes.

\subsection{Implementation, architectures and hyperparameters}

We based our code on the repo~\cite{pchrist}, that contains PyTorch implementations of various deep reinforcement learning algorithms. We used the default architectures from the repo, the Adam optimizer and performed small hyperparameters sweeps on learning rates for each algorithm (described below).

\paragraph{DDQN} We used the default implementation provided by the repo~\cite{pchrist} for the Four Rooms environment. The q-network is a two-hidden layer MLP with 30 and 10 neurons in each. The replay buffer size is 40k and we used soft decay of the epsilon greedy exploration as well as soft updates of the target network. In terms of learning rate, we swept over $\{0.06, 0.03, 0.01, 0.006, 0.003, 0.001\}$ and found that $0.01$ performed best. We also found that decreasing the epsilon greedy decay compared to the default value in the repo gave better performance on Level 2 (it did not affect Level 1 much). We used the following probability for picking a random action: $100 / e$, where $e$ is the number of episodes trained on so far (compared to $10$ for the default).

\paragraph{Soft-Actor Critic} The actor and critics are two-hidden layers MLPs with 64 neurons in each. Two critics were trained, and the mininum between the two was used in the backup of the critics, as well as to train the actor. The entropy regularization hyperparameter was automatically tuned. The implementation is the standard one provided in~\cite{pchrist}. We also swept over the same learning rate set and found that similarly $0.01$ performed best. 

\paragraph{RND} For the RND target network, we used a one-hidden layer MLP with 20 neurons, outputing a 20-dimensional feature vector. The predictor network contains (as is customary in RND implementations) an additional hidden layer with 20 neurons as well. We did not search over architectures. The learning rate is set to $0.003$, which was found after a sweep on the set $\{0.1, 0.03, 0.01, 0.003, 0.001\}$.

\paragraph{Dr Jekyll \& Mr Hyde} Dr Jekyll's architecture is the same as the Soft-Actor Critic one described above. Its learning rate was set to 0.001, which was found to perform best out of the set $\{0.06, 0.03, 0.01, 0.006, 0.003, 0.001, 0.0006\}$. Mr Hyde's architecture and learning rate are the same as the DDQN architecture described above. The decay of the probability of picking Mr. Hyde (parameter $\epsilon_1$ in algorithm~\ref{alg:deep-JH}) is set to $4k$ for Level 1 and $40k$ for Level 2. Note that this corresponds to the number of steps in the environment, not the number of episodes played. Those values correspond approximately to the decay of the exploration in DDQN.

\paragraph{DDQN+RND and SAC+RND} In the case of DDQN+RND, we simply augmented the reward of the environment using the exploration bonus provided by the RND network. For SAC+RND, we trained an additional critic on the exploration bonus and added it to SAC's actor update. In both cases, we weighted the exploration bonus by $1$ afer sweeping over the set $\{0.1, 0.3, 0.6, 1.0, 1.3, 1.6\}$ and finding that it performed best in both cases.

\subsection{The detailed algorithm}

Algorithm~\ref{alg:deep-JH} describes the implementation used to produce Figure~\ref{fig:four-rooms}.

\begin{algorithm}[ht]
\caption{Dr Jekyll \& Mr Hyde algorithm for experiments with a deep architecture.}
\multiline{\textbf{Input:} Scheduling of exploration $(\epsilon_t)=\frac{\epsilon_1}{t}$.}
\begin{algorithmic}[1]
    \State Initialize Dr Jekyll's actor $\mathring{\pi}\doteq \pi_{\mathring{\theta}}$ and critic $\mathring{q}=q_{\mathring{\theta}}$.
    \State Initialize Mr Hyde's q-network $\tilde{q}=q_{\tilde{\theta}}$ (its policy $\tilde{\pi}$ is greedy with respect to $\tilde{q}$).
    \State Initialize the RND predictor $p_{RND}$ and target $t_{RND}$ networks.
    \State Initialize a buffer $F=\emptyset$ of capacity $K=128$ and a buffer $\tilde{D}=\emptyset$ of capacity $40k$.
    \State Set the behavioural policy to Dr Jekyll: $\pi_{b} \leftarrow \mathring{\pi}$ and the working replay buffer $D_{b} \leftarrow \mathring{D}$.
    \For{$t=0$ to $\infty$}
        \State Sample a transition $\tau_t = \langle s_t,a_t\sim\pi_{b}(\cdot|s_t),s_{t+1}\sim p(\cdot|s_t,a_t),r_t\sim r(\cdot|s_t,a_t)\rangle$.
        \InlineIfThen{Mr Hyde is in control ($\pi_{b} = \tilde{\pi}$),}{add $\tau_t$ to Mr Hyde's replay buffer $\tilde{D}$}.
        \State Add $\tau_t$ to buffer $F$.
        \InlineIfThen{$\tau_t$ was terminal,}{$\pi_{b} \leftarrow \tilde{\pi}$ wp $\epsilon_t$, $\pi_{b} \leftarrow \mathring{\pi}$ otherwise.}
        \renewcommand\algorithmicdo{}
        \State Initialize minibatch with FIFO buffer: $B\leftarrow F$
        \For{$i=1\to K$}
            \State Add $\tau_i \doteq \langle s,a,s',r \rangle \sim \tilde{D}$ to minibatch $B$
        \EndFor
        \State For each transition $\tau_i \in B$, compute RND reward $\tilde{r}_{\tau_i} = \|p_{RND}(s') - t_{RND}(s')\|^2$. 
        \State Update Mr Hyde's q-network $\tilde{q}$ with a DDQN update on $B$ where $\tilde{r}$ is used instead of $r$
            \begin{align}
                \tilde{q}(s,a)&\leftarrow \tilde{q}(s,a)+\eta_c\left(\tilde{r}(s,a)+\gamma\max_{a'\in\mathcal{A}}\tilde{q}(s',a')-\tilde{q}(s,a)\right).
            \end{align}
        \State Update Dr Jekyll's critic $\mathring{q}$ with a SARSA update on $B$:
            \begin{align}
                \mathring{q}(s,a)\leftarrow \mathring{q}(s,a)+\eta_c\left(r+\gamma\sum_{a'\in\mathcal{A}}\mathring{q}(s',a')\mathring{\pi}(s',a')-\mathring{q}(s,a)\right).
            \end{align}
        \State Perform an expected update step in each state $s$ of $B$ on Dr Jekyll's actor $\mathring{\theta}$:
        \begin{align}
            \forall b\in\mathcal{A},\quad\quad\mathring{\theta} \leftarrow \mathring{\theta} + \eta \mathring{q}(s,b) \nabla_{\mathring{\theta}} \mathring{\pi}(b|s)
        \end{align}
        \renewcommand\algorithmicdo{do}
        \State Update $p_{RND}$ by gradient descent on $\sum_{s' \in B} \|p_{RND}(s') - t_{RND}(s')\|^2$.
    \EndFor
\end{algorithmic}
\label{alg:deep-JH}
\end{algorithm}

\section{Code description}
\label{app:code}
The code for all our experiments can be found at \url{https://github.com/microsoft/Dr-Jekyll-and-Mr-Hyde-The-Strange-Case-of-Off-Policy-Policy-Updates}. 

\subsection{Finite MDPs code}
It relies on two main files: 
\begin{itemize}
    \item \texttt{main\_config.py} to run experiments in a fixed setting and report average/quantile performance across time,
    \item and \texttt{main\_hyperparam.py} to run experiments where a hyperparameter varies, and report the time to reach a fixed normalized performance target $\tilde{\mathcal{J}}$ (see \eq{} \eqref{eq:normalized_chain} and \eqref{eq:normalized_garnets}).
\end{itemize}

They both use the same configuration file determining the setting the experiment (the hyperparameter search overwrites the setting of this hyperparameter during the run): which setting (planning or RL), which environment (chain or random MDPs), which algorithms, with which hyperparameters, for how many steps, and for how many runs. Such configuration files may be found in the \texttt{expes} folder, where the settings of all reported experiments may be retrieved as well as their figures (and some more). Nevertheless, we deleted the experiment result files because they are voluminous (10s of Go), but we have them and can provide them on demand, or better: their random seeds.

The two environments are implemented in python files:
\begin{itemize}
    \item \texttt{chain.py}: the chain environment described in \app{} \ref{app:chain},
    \item and \texttt{garnets.py}: the random MDPs environment described in \app{} \ref{app:garnets} \cite{Laroche2019}.
\end{itemize}

Two utility python files are used:
\begin{itemize}
    \item \texttt{utils.py} contains all utility files,
    \item except \texttt{proj\_simplex.py} implementing the projection on the simplex, that has been isolated in order to account for their authors \cite{Blondel2014}.
\end{itemize}

Finally, all the algorithmic innovation that we claim belongs to \texttt{gradient\_ascent.py}, including the implementation of \jh{} (\alg{} \ref{alg:finiteMDP-JH}).

\subsection{Deep \jh{} implementation}

The implementation is based on the repository~\cite{pchrist}, released under the MIT License and which contains various PyTorch implementations of standard deep reinforcement learning algorithms. The entry point to train the agents is the \texttt{results/Four\_Rooms.py} file. To train a given one, simply run:
\begin{center}
    \texttt{python results/Four\_Rooms.py --agent \{agent\} --level \{level\}}
\end{center}
where agent is chosen in [JH\_Discrete, SAC\_Discrete, SAC\_DiscreteRND, DDQN, DDQNRND] and level is 1 or 2.

\end{document}